\documentclass{article}

\usepackage{jair}

\usepackage{svg}
\usepackage{csvsimple}

\title{Meta-optimization for transfer learning\\ from mixtures of pre-trained kernel regressors}

\author{\name Xuwei Yang \email yangx212@mcmaster.ca \\
      \addr Department of Mathematics\\
      McMaster University
      \AND
      \name Anastasis Kratsios \email kratsioa@mcmaster.ca \\
      \addr Department of Mathematics\\
      McMaster University and Vector Institute
      \AND
      \name Florian Krach \email florian.krach@math.ethz.ch\\
      \addr Department of Mathematics\\
      ETH Z\"{u}rich
      \AND Matheus Grasselli \email grassel@mcmaster.ca\\
      \addr Department of Mathematics\\
      McMaster University
       \AND Aurelien Lucchi \email aurelien.lucchi@unibas.ch\\
    \addr Department of Mathematics and Informatics\\
      University of Basel}

\newcommand\numberthis{\addtocounter{equation}{1}\tag{\theequation}}

\usepackage{amsthm}
\newtheorem{proposition}{Proposition}
\newtheorem{lemma}{Lemma}
\newtheorem{theorem}{Theorem}
\newtheorem{remark}{Remark}
\newtheorem{assumption}{Assumption} 
\newtheorem{corollary}{Corollary}
\newtheorem{experiment}{Experiment}

\usepackage{algpseudocode}
\usepackage[linesnumbered,lined,boxed,commentsnumbered,ruled,longend]{algorithm2e}
    \RestyleAlgo{ruled}

\SetCommentSty{mycommfont}

\usepackage{lscape}

\usepackage{wrapfig} 
\usepackage{adjustbox} 
    \usepackage[column=0]{cellspace} 
    \usepackage{multirow}
\usepackage{enumitem}
\usepackage{float}
\usepackage[colorinlistoftodos]{todonotes}
\usepackage{amsmath,amsfonts,bm,amssymb,mathtools}

\usepackage[T1]{fontenc}    
\usepackage{booktabs}       
\usepackage{amsfonts}       
\usepackage{nicefrac}       
\usepackage{microtype}      
\usepackage{bbm}            

\usepackage{caption}
\usepackage{subcaption}

\usepackage{hyperref} 
\makeatletter
\def\@footnotecolor{red}
\define@key{Hyp}{footnotecolor}{%
 \HyColor@HyperrefColor{#1}\@footnotecolor%
}
\def\@footnotemark{%
    \leavevmode
    \ifhmode\edef\@x@sf{\the\spacefactor}\nobreak\fi
    \stepcounter{Hfootnote}%
    \global\let\Hy@saved@currentHref\@currentHref
    \hyper@makecurrent{Hfootnote}%
    \global\let\Hy@footnote@currentHref\@currentHref
    \global\let\@currentHref\Hy@saved@currentHref
    \hyper@linkstart{footnote}{\Hy@footnote@currentHref}%
    \@makefnmark
    \hyper@linkend
    \ifhmode\spacefactor\@x@sf\fi
    \relax
  }%
\makeatother
    \hypersetup{
    	colorlinks = true,
    	linkcolor = blue,
    	anchorcolor = blue,
    	citecolor = blue,
    	filecolor = blue,
    	urlcolor = blue,
    	footnotecolor=black
    }

\usepackage{xcolor} 

\usepackage{xparse}

\usepackage[colorinlistoftodos]{todonotes}
\usepackage{soul}

\def\1{\bm{1}}

\newcommand{\E}{\mathbb{E}}

\newcommand{\R}{\mathbb{R}}

\DeclareMathOperator*{\argmin}{arg\,min}

    \NewDocumentCommand{\rr}{o}{\mathbb{R}^{\IfValueT{#1}{#1}}}

\definecolor{darkcerulean}{rgb}{0.03, 0.27, 0.49}
\definecolor{darkmidnightblue}{rgb}{0.0, 0.2, 0.4}
\definecolor{darkcyan}{rgb}{0.0, 0.55, 0.55}
\definecolor{LandscapeBlue}{RGB}{10,38,255}
\definecolor{darkgreen}{rgb}{0.0, 0.2, 0.13}
\definecolor{deepjunglegreen}{rgb}{0.0, 0.29, 0.29}
\definecolor{applegreen}{rgb}{0.55, 0.71, 0.0}
\definecolor{LandscapeGreen}{RGB}{18,125,9}
\definecolor{darkcandyapplered}{rgb}{0.64, 0.0, 0.0}
\definecolor{darkred}{rgb}{0.55, 0.0, 0.0}
\definecolor{darkscarlet}{rgb}{0.34, 0.01, 0.1}
\definecolor{jasper}{rgb}{0.84, 0.23, 0.24}
\definecolor{darkjazzberryjam}{rgb}{0.45, 0.04, 0.37}
\definecolor{LandscapePurple}{RGB}{135,0,170}
\definecolor{MidnightBlue}{RGB}{25,25,112}
\definecolor{MidnightBlueComplementingGreen}{RGB}{25,112,25}
\definecolor{MidnightBlueComplementingPurple}{RGB}{112,25,112}
\definecolor{MidnightBlueComplementingRed}{RGB}{112,25,69}

\newcommand{\eqdef}{\ensuremath{\stackrel{\mbox{\upshape\tiny def.}}{=}}}

\makeatletter
\renewcommand*{\p@equation}{}
\makeatother

\DeclarePairedDelimiter\abs{\lvert}{\rvert}%
\DeclarePairedDelimiter\norm{\lVert}{\rVert}%

\makeatletter
\let\oldabs\abs
\def\abs{\@ifstar{\oldabs}{\oldabs*}}
\let\oldnorm\norm
\def\norm{\@ifstar{\oldnorm}{\oldnorm*}}
\makeatother

\usepackage{booktabs}
\usepackage{lscape}

\usepackage[utf8]{inputenc} 
\usepackage[T1]{fontenc}    
\usepackage{url}            
\usepackage{booktabs}       
\usepackage{amsfonts}       
\usepackage{nicefrac}       
\usepackage{microtype}      

\usepackage{natbib}       
\usepackage[capitalise]{cleveref}

\usepackage{comment}
\usepackage{marginnote}
\definecolor{MidnightBlue}{RGB}{25,25,112}
\definecolor{MidnightBlueComplementingGreen}{RGB}{25,112,25}
\definecolor{MidnightBlueComplementingPurple}{RGB}{112,25,112}
\definecolor{MidnightBlueComplementingRed}{RGB}{112,25,69}
\definecolor{deepjunglegreen}{rgb}{0.0, 0.29, 0.29}
\definecolor{applegreen}{rgb}{0.55, 0.71, 0.0}
\definecolor{WowColor}{rgb}{.75,0,.75}
\definecolor{MildlyAlarming}{rgb}{0.85,0.25,0.1}
\definecolor{SubtleColor}{rgb}{0,0,.50}
\definecolor{SubtleColor2}{rgb}{0.6,0.21,.50}
\newcounter{margincounter}

\NewDocumentCommand{\Anastasis}{mo}{
    \IfValueF{#2}{
                        {{\scriptsize
                            \textcolor{violet}{ 
                            \textbf{A:}
                            \textit{{#1}}
                            }
                        }}
        }
    \IfValueT{#2}{
                        \marginnote{{\scriptsize
                            \textcolor{violet}{ 
                            \textbf{A:}
                            \textit{{#1}}
                            }
                        }}
        }
                    }

\NewDocumentCommand{\Xuwei}{mo}{
    \IfValueF{#2}{
                        {{\scriptsize
                            \textcolor{magenta}{ 
                            \textbf{X:}
                            \textit{{#1}}
                            }
                        }}
        }
    \IfValueT{#2}{
                        \marginnote{{\scriptsize
                            \textcolor{magenta}{ 
                            \textbf{S:}
                            \textit{{#1}}
                            }
                        }}
        }
                    }

\NewDocumentCommand{\Matheus}{mo}{
    \IfValueF{#2}{
                        {{\scriptsize
                            \textcolor{blue}{ 
                            \textbf{S:}
                            \textit{{#1}}
                            }
                        }}
        }
    \IfValueT{#2}{
                        \marginnote{{\scriptsize
                            \textcolor{blue}{ 
                            \textbf{S:}
                            \textit{{#1}}
                            }
                        }}
        }
                    }

\NewDocumentCommand{\Aurelien}{mo}{
    \IfValueF{#2}{
                        {{\scriptsize
                            \textcolor{cyan}{ 
                            \textbf{AL:}
                            \textit{{#1}}
                            }
                        }}
        }
    \IfValueT{#2}{
                        \marginnote{{\scriptsize
                            \textcolor{cyan}{ 
                            \textbf{AL:}
                            \textit{{#1}}
                            }
                        }}
        }
                    }

\NewDocumentCommand{\Rodrigo}{mo}{
    \IfValueF{#2}{
                        {{\scriptsize
                            \textcolor{orange}{ 
                            \textbf{R:}
                            \textit{{#1}}
                            }
                        }}
        }
    \IfValueT{#2}{
                        \marginnote{{\scriptsize
                            \textcolor{orange}{ 
                            \textbf{R:}
                            \textit{{#1}}
                            }
                        }}
        }
                    }

\usepackage{cleveref}
\newcounter{termcounter}
\renewcommand{\thetermcounter}{\Roman{termcounter}}
\crefname{term}{term}{terms}
\creflabelformat{term}{#2\textup{(#1)}#3}

\makeatletter
\def\term{\@ifnextchar[\term@optarg\term@noarg}
\def\term@optarg[#1]#2{%
  \textup{#1}%
  \def\@currentlabel{#1}%
  \def\cref@currentlabel{[][2147483647][]#1}%
  \cref@label[term]{#2}}
\def\term@noarg#1{%
  \refstepcounter{termcounter}%
  \textup{(\thetermcounter)}%
  \cref@label[term]{#1}}
\makeatother

\begin{document}



\maketitle

\begin{abstract}
In contemporary machine learning, models are trained by aggregating and fine-tuning a set of pre-trained models, trained on distinct datasets, to a given novel dataset.  This paper addresses the \textit{meta-optmization} problem of the task of designing optimal transfer learning algorithms, which minimize a pathwise regret functional. Using techniques from optimal control, we construct the unique regret-optimal optimization algorithms for fine-tuning mixtures of pre-trained kernel regressors on $\mathbb{R}^d$ sharing the same feature map. Our regret functional balances the objects of predictive power on the new dataset against transfer learning from other datasets under an algorithmic stability penalty.  We show that an adversary which perturbs $q$ training pairs by at most $\varepsilon>0$, across all training sets, cannot reduce the regret-optimal algorithm's regret by more than $\mathcal{O}(\varepsilon q \bar{N}^{1/2})$, where $\bar{N}$ is the aggregate number of training pairs.  Further, we derive estimates of the computational complexity of our regret-optimal optimization algorithm.  Our theoretical findings are ablated on American option pricing problems using random feature models. 
\end{abstract}

\section{Introduction}
\label{s:Intro}



In transfer learning one is usually confronted with a focal task defined by a dataset $\mathcal{D}_1$ {drawn from a distribution $\mu_1$ on $\mathbb{R}^{d+1}$} and additional (more or less) related datasets $\mathcal{D}_2,\dots,\mathcal{D}_N$ {, where each $\mathcal{D}_i$ is drawn from (possibly different) distributions $\mu_i$ on $\mathbb{R}^{d+1}$.}  
The goal is to maximize the performance of a model $f_{\theta}:\mathbb{R}^d\to \mathbb{R}$ on $\mathcal{D}_1$, by optimizing the parameters $\theta \in \mathbb{R}^p$ with the possibility to leverage the knowledge (about the focal task) that is enclosed in the additional datasets $\mathcal{D}_2,\dots,\mathcal{D}_N$.
{
The performance of a model is quantifies using a user-specified $L_{\ell}$-Lipschitz loss function $\ell:\mathbb{R}\times \mathbb{R}\to \mathbb{R}$; with the intent being to minimize the out-of-sample loss on $\mu_1$ of any learner $f_{\theta}:\mathbb{R}^D\to \mathbb{R}$, as quantified by its \textit{true risk} $\mathcal{R}(f_{\theta})$. In practice, we minimize a \textit{weighted} version of the empirical risk $\hat{\mathcal{R}}_{w}^{\mathcal{D}}(f_{\theta})$ proportionally accounting for each dataset, using a \textit{mixture coefficient} $w\in \Delta_N\eqdef \{u\in[0,1]^N:\,\sum_{n=1}^N\,u_n=1\}$.  These risks are defined by
\[
        \mathcal{R}(f_{\theta})
    \eqdef 
        \mathbb{E}_{(X,Y)\sim \mu_1}\big[
            \ell(f_{\theta}(X),Y)
        \big] 
        \] 
and
\[        \hat{\mathcal{R}}_{w}^{\mathcal{D}}(f_{\theta})
    \eqdef 
        \sum_{i=1}^N\,
            \frac{w_i}{|\mathcal{D}_i|}
            \,
            \sum_{j=1}^{|\mathcal{D}_i|}
                \ell(f_{\theta}(X_j^i),Y_j^i)
.
\]
Our first main result (Theorem~\ref{thrm:transfer_bound}) is a \textit{statistical guarantee}, which shows that the following choice of weights $w^{\star}=(w^{\star})_{i=1}^N$ minimizes the worst-case generalization gap between the true risk and the weighted empirical risk of any $L$-Lipschitz learner (for a given $L\ge 0$).
\begin{equation}
\label{eq:Optimal_Weights}
    w^{\star}_i
    =
    \frac{
        e^{
        -
        \gamma\,\mathcal{W}_1(\hat{\mu}_1,\hat{\mu}_i)
        -
        \frac{\gamma}{|\mathcal{D}_i|^{1/(d+D)}}
        }\,
        I_{\mathcal{W}_1(\hat{\mu}_1,\hat{\mu}_i)\le \eta}
    }{
        \sum_{u=1}^N\,
        e^{-\gamma\mathcal{W}_1(\hat{\mu}_1,\hat{\mu}_u)
        -\frac{\gamma}{|D_u|^{1/(d+D)}}
        }\,
            I_{\mathcal{W}_1(\hat{\mu}_1,\hat{\mu}_u)\le \eta},
    }
\end{equation}
where $\mathcal{W}_1$ is the $1$-Wasserstein distance and $
\hat{\mu}_i\eqdef \frac{1}{|\mathcal{D}_i|}\,\sum_{(x,y)\in \mathcal{D}_i}\,\delta_{(x,y)}$is the empirical distribution associated to the $i^{th}$ dataset, and $\eta\ge 0$ \textit{sparse thresholding} hyperparameter controlling how similar a dataset must be to the focal task's dataset $\mathcal{D}_1$ to be considered.
Moreover, the weighted empirical risk is a consistent estimator of the true risk, converging at the non-parametric rate of $\mathcal{O}\big(N^{-1/(d+1)}\big)$.
}


{
Given the optimal mixture weights, defining the optimal weighted empirical risk for the general transfer learning problem, one may ask: 
\textit{How does one select optimal parameters $\theta$ given $w^{\star}$ as efficiently as possible?}
We will answer this \textit{meta-optimization} question in the specialized case where $f_{\theta}$ are kernel ridge regressors, the loss function $\ell$, and all the data generating distributions are compactly supported.  
}
Our \textit{meta-optimization} problem is thus to select, amongst a class of deterministic iterative optimization algorithms, those algorithms which minimize a fixed performance metric along their path.  We tailor our analysis to the following transfer learning problem.  

We identify deterministic algorithms defining and jointly optimizing parameters  $ \Theta(t) \eqdef (\theta_1^\top(t),\dots,\theta_N^\top(t))^\top$ associated with each task $\mathcal{D}_i$, focusing on the focal task $\mathcal{D}_1$. We define the parameter trajectories $\Theta^{0\cdots T}\eqdef \left\{\Theta(t)\right\}_{t=0}^T$ in $(\mathbb{R}^{Np})^T$, and $\theta_i^{\star}$  to be the greedy parameter choices optimizing the empirical risk \textit{calculated only over the $i^{th}$ dataset}.  All algorithms will be initialized at $\Theta(0)\eqdef \Theta^\star \eqdef (\theta_1^{\star},\dots,\theta_N^{\star})\in \mathbb{R}^{Np}$ and are terminated after $T\in \mathbb{N}_+$ iterations.
Building on~\cite{casgrain2019latent,casgrain2021optimizing}, given a $w\in \Delta_N$, we use the following \textit{regret}-based criterion to quantify the performance of an algorithm 
\allowdisplaybreaks 
\begin{align}
\label{eq:penalized_regret_functional__REGRETFORM}
    \mathcal{R}(\Theta^{0\cdots T})
\eqdef &
        \underbrace{
                \sum_{i=1}^N 
                w_i
                    \sum_{j=1}^{|\mathcal{D}_i|} 
                        (
                        f_{\theta^w(T)}(x_j^i)
                        - y^i_j  
                        )^2
        }_{\text{Performance: Weighted Empirical Risk}}
    -
                  \,\,  l^{\star}  
 \notag   \\
&     +     \sum_{t=0}^{T-1}  
        \,
             \big(
                    \underbrace{
                        \lambda \left\| \Theta(t+1) - \Theta^\star \right\|_2^2    
                    }_{\text{Transfer Learning Level}}
                + 
                    \underbrace{
                        \beta  \left\| \Delta \Theta(t) \right\|_2^2 
                    }_{\text{Stability}}
                \big) , 
\end{align}  
where the hyperparameter $\lambda\ge 0$ controls how far the algorithm's iterates can deviate from the pre-trained model parameters $\Theta^{\star}$, the hyperparameter $\beta\ge 0$ controls the stability of the training algorithm by penalizing rapidly changing increments $\Delta \Theta(t)\eqdef \Theta(t)-\Theta(t-1)$, where $t\ge 1$, and where $l^{\star}\eqdef \min_{\Theta \in \mathbb{R}^{Np}}\, \sum_{i=1}^N 
w_i
    \sum_{j=1}^{|\mathcal{D}_i|} 
        (
        f_{\theta^w}(x_j^i)
        - y^i_j  
        )^2
$ captures the achievable error across all tasks.  
In this paper, we select the optimal weights $w$ with Algorithm~\ref{alg:RegretOptimizationHeuristic__WarmStart}, using a procedure similar to the PAC-Bayesian approach of~\cite{alquier2021user}, often used in meta-learning~\citep{PACOH_rothfuss_2021,pavasovic2022mars,rothfuss2023scalable}.

{Our main motivation for this setting stems from the now typical setting, where one has access to a pre-trained foundation model, whose final linear layer is to be fine-tuned on several different datasets.  This leads to pre-trained deep learning models $f_{\theta_1}^{\star},\dots,f_{\theta_N}^{\star}$ whose hidden layers all coincide}.  In this case, $f_{\theta}$ factors as a kernel regressor where its \textit{frozen} hidden layers define a feature map, and the final linear layer is the only-trainable parameter.  For analytical tractability, we require that each neural network has the same frozen hidden layers.  
Thus, all our models are finite-rank kernel ridge regressors of the form
\begin{equation}
\label{eq:FRKR}
f_{\theta}(x)=\phi(x)^\top\theta
.
\end{equation}
In reservoir computing \citep{lukovsevivcius2009reservoir,grigoryeva2018echo} and randomized neural networks \citep{gonon2022approximation}, this assumption is standard since random reservoirs/feature maps are used. 
Moreover, in settings where each neural network has a different set of hidden layers, one can achieve the representation~\eqref{eq:FRKR} simply by concatenating each of the frozen hidden layers into a joint feature map.  Thus, the structural condition~\eqref{eq:FRKR} is made without any loss of generality. 

\subsection{Contributions}
\label{s:Introduction_Contrib}
{In Theorem~\ref{thrm:transfer_bound} we obtain a guarantee that the weighting scheme in~\eqref{eq:Optimal_Weights} (and implemented by Algorithm~\ref{alg:RegretOptimizationHeuristic__WarmStart}) minimizes an upper-bound on the \textit{worst-case} gap between the empirical risk for the focal dataset $\mathcal{D}_1$ and the (optimal) \textit{weighted} empirical risk over all datasets.  We emphasize that the first result is fully-general and holds for arbitrary deep learning models optimizing a general Lipschitz loss, just as much as it does for specialized kernel ridge regressors optimizing the MSE loss.}  

We present a \textit{regret optimal} algorithm, by which we mean an optimizer of~\eqref{eq:penalized_regret_functional__REGRETFORM}.  Aligned with the recent trends in optimization \cite{casgrain2019latent,JMLR:v18:17-653}, we leverage techniques from optimal control theory to characterize the regret-optimal algorithm.  We show that this is possible due to the observation that minimizers of~\eqref{eq:penalized_regret_functional__REGRETFORM} coincide with the minimizers of the energy functional 
\allowdisplaybreaks\begin{align} 
\label{eq:penalized_regret_functional}
    {\cal L}(\Theta^{0\cdots T})
\eqdef 
&        \sum_{t=0}^{T-1}  
        \, 
    \underbrace{
             \left[  
                    \lambda \left\| \Theta(t+1) - \Theta^\star \right\|_2^2    
                + 
                    \beta \left\| \Delta \Theta(t) \right\|_2^2 
            \right] 
    }_{\text{
        Running-Cost
    }} 
 + 
    \underbrace{
        l (\Theta(T), w;\,\mathcal{D} ) 
    }_{\text{
    Terminal Cost
    }} ,  
\end{align} 
where we abbreviate $
l(\Theta, w;\,\mathcal{D} )  
\eqdef 
\sum_{i=1}^N 
w_i
    \sum_{j=1}^{|\mathcal{D}_i|} 
        (
        f_{\theta^w}(x_j^i)
        - y^i_j  
        )^2$.
From this control-theoretic perspective, every algorithm can be regarded as driven by a ``control'', and the functional ${\cal L}$ has the form of a ``cost function'' of an optimal control problem,  decomposable into the sum of a ``running/operational cost'' and a ``terminal/objective cost''.  This allows us to leverage tools from optimal control to derive a closed-form expression of the unique regret-optimal algorithm driven by the unique ``optimal control'' (Algorithm~\ref{alg:RegretOptimization}).  

We prove that this procedure is adversarially robust, meaning that the value of~\eqref{eq:penalized_regret_functional} varies by $\mathcal{O}(\varepsilon \sqrt{\bar{N}}q)$ if a malicious adversary can perturb at most $1/q$ percent of all the data in each of the dataset $\mathcal{D}_i$ by at-most $\varepsilon\ge 0$, where $\bar{N}=\sum_{i=1}^N |\mathcal{D}_i|$.  
Furthermore, we show that our regret-optimal algorithm has a computational complexity of $\mathcal{O}(N^2p^3 + T\,(Np)^{2.373})$.  
Moreover, we characterize the optimal choice of the weights $w$ 
which the central planner must use to maximize the selected model's performance on the primary/focal dataset $\mathcal{D}_1$.  

\subsection*{Outline of Paper}   
In Section~\ref{s:Main}, we introduce the regret-optimal algorithm and demonstrate the optimality, computational complexity, and adversarial robustness of the algorithm. We  further introduce an accelerated algorithm of lower computational complexity that can achieve near-regret optimality. 
In Section~\ref{s:Experiments_n_Validation}, we conduct a convergence analysis of the regret-optimal algorithm and compare it to the standard gradient descent approach for American option pricing in quantitative finance.   

A game theoretic interpretation of $\mathcal{R}$ is given in Appendix~\ref{s:GameTheoreticInterpretation}. Appendix~\ref{s:Proofs_TechnicalLemmata} is devoted to proving the optimality, adversarial robustness, and computational complexity of the regret-optimal algorithm. The near-regret optimality and computational complexity of the accelerated algorithm are also proven. Appendix~\ref{a:ExperimentDetails} contains additional details on all experimental hyperparameters and longer tables for our ablation study.  Appendix~\ref{a:Recap:AOP_OS} includes background on American option pricing in quantitative finance and on finite-rank kernel ridge regressors.   

\vspace{-.5em}
\subsection{Related Work}
\label{s:Introduction__sss:RelatedWork}

Examples of $f_\theta$ of the form \eqref{eq:FRKR} include extreme learning machine \citep{gonon2020risk,gonon2022approximation,ghorbani2022linearized,MeiMontanari_CommPaAMath_2022__DoubleDescent} or any finite-rank kernel ridge regressor (fKRR) \citep{Amini_2021_EJS_SpectrallyTruncatedKernel,amini2022target}.  Random feature models, and finite-rank ridge regressors, are typical in contemporary quantitative finance \citep{gonon2021random,herrera2023optimal} or reservoir computing \citep{cuchiero2021discrete}, where it is favourable to rapidly deploy a deep learning model due to time or computing-power limitations/constraints~\citep{compagnoni2023effectiveness}.  Finite-rank kernel ridge regressors are also typical in the setting of standard transfer \citep{howard2018universal} and multi-task learning \citep{ren2018cross,jia2019leveraging} pipelines where one often has a pre-trained deep neural network model and, for any novel task, the user freezes the hidden layers in the deep neural network model and only fine-tunes/trains the deep neural network's final linear layer.  
This is useful since the hidden layers of most deep neural network architectures are known to process general, cruder features, and the final layer of any such model decodes these features in a task-specific manner~\citep{zeiler2014visualizing,yosinski2014transferable}.

The optimization of a model trained from multiple data sources lies at the heart of multi-task learning.  In  \cite{SenerKoltun_MultiTaskLEarningMultiObjectiveOptimization__NeurIPS2018}, the authors cast the task of optimizing~\eqref{eq:penalized_regret_functional__REGRETFORM} with $T=0=\lambda=\beta$ as a multi-objective optimization problem.  They propose to solve it via a multi-gradient descent algorithm which, under mild conditions, converges to a Pareto stationary point; i.e.\ which is essentially a critical point of each of the $i$-th player's penalized SSE.  
Though iterative schemes for multi-objective optimization can efficiently reach critical points \citep{fliege2019complexity}, this problem's critical point is not necessarily a minimum since the problem~\eqref{eq:penalized_regret_functional__REGRETFORM} is not convex when $w$ and $\Theta$ are jointly optimizable.   
The regularization hyperparameter $\lambda$ in 
\eqref{eq:penalized_regret_functional__REGRETFORM} controls the level of leveraging the information from the source tasks and has been studied in~\citep{pmlr-v28-kuzborskij13, 
IJCAI16-WangOliviaSchneiderPoczos, 
tian2023learningsimilarlinearrepresentations, 
lin2024smoothnessadaptivehypothesistransfer}.

This proposed algorithm reflects the ideas of the federated stochastic gradient descent (FedSGD) algorithm of \cite{SINGHZHU_FedSGD_2018} and of its online counterpart FedOGD which is an online federated version of online gradient descent\footnote{See \cite[Chapter 3]{hazan2016introduction} for details on online gradient descent} and more efficient online versions thereof such as \cite{kwon2022tighter}.  From the federated learning perspective, the problem of training individual fKRR on each dataset and then centralizing them by optimizing the mixing parameter $w$ in~\eqref{eq:penalized_regret_functional__REGRETFORM} and its origins date back to local gradient descent introduced in the optimization and control literature in \cite{Mangasarian_1995_localGD}.  Though several authors (e.g. \cite{TighterTheoryLocalSGD}) have proposed various progressive improvements of federated gradient-descent-type algorithms for training a learner from several data sources, the problem of studying \textit{the} most efficient iterative federated learning algorithm for regret minimization has not yet been tackled.  

We note that (stochastic) optimal control tools have revealed new insights into machine learning.  Examples include an optimal step-size control \citep{li2017stochastic} and batch-size control \citep{zhao2022batch}, developing online subgradient methods that adapt to the data's observed geometry \citep{DuchiHazanSinger_2011_JMLR_adaptiveSGOnlineOptim}, developing new maximum principle-based training algorithms for deep neural networks \citep{JMLR:v18:17-653}, or unifying (stochastic) optimization frameworks \citep{Casgrain_LatentVarStochOptim_2019}.

Our experiments focus on quantitative finance problems.   We thus add to the emerging literature, e.g.~\cite{kraus2017decision,jeong2019improving,cao2023transfer}, recognizing the potential of transfer learning in mathematical finance, especially in situations where data is scarce.

\vspace{-.5em}
\subsection*{Notation}
We introduce some additional notation.
Each dataset has a finite number of samples which we write as $\mathcal{D}_i = \{ (x_j^i, y_j^i) \, | \, 1 \leq j \leq |\mathcal{D}_i| \}$. We assume that some feature map $\phi$ is fixed and denote the features of each sample of $\mathcal{D}_i$ as 
$u_j^i \eqdef \phi(x_j^i).$
We denote the joint inputs and outputs for each dataset as $
    X^i \eqdef ( 
        x^i_1, \cdots , x^i_{|\mathcal{D}_i|}
    )^\top \in \mathbb{R}^{|\mathcal{D}_i| \times d} 
$, $
    Y^i \eqdef (y^i_1, \cdots, y^i_{|\mathcal{D}_i|})^\top \in \R^{|\mathcal{D}_i|}
$, and the joint features as $
U^i \eqdef \Phi(X^i) \eqdef 
( 
    \phi(x^i_1), \cdots , \ 
    \phi(x^i_{|\mathcal{D}_i|}) 
 )^\top \in \mathbb{R}^{|\mathcal{D}_i| \times p}. $
We use $I_p$ to denote a $p\times p$ identity matrix. 

\vspace{-.5em}
\section{Algorithms}
\label{s:Main_ss:Algorithm}

We begin by describing our procedure for determining the optimal mixture weights used in ~\eqref{eq:penalized_regret_functional__REGRETFORM}.  This procedure is detailed in Algorithm~\ref{alg:RegretOptimizationHeuristic__WarmStart}.

\begin{algorithm}[ht!]
\caption{{Optimal Information Sharing - Initialization to Algorithm~\ref{alg:RegretOptimization}}}
\label{alg:RegretOptimizationHeuristic__WarmStart}
\begin{algorithmic}
\SetAlgoLined
\Require Datasets $\mathcal{D}_1,\dots,\mathcal{D}_N$, finite-rank kernel $\phi$, hyperparameters\footnote{Different ridge hyperparameters $\kappa$ can be used to define each regressor in Algorithm ~\ref{alg:RegretOptimizationHeuristic__WarmStart}} $\kappa,\eta>0$.
\DontPrintSemicolon
    \State \tcp{Initialize Locally-Optimal Learners and Record Scores on Main Dataset}  
    \State \SetKwBlock{ForParallel}{For $i:1,\dots,N$ in parallel}{end}
        \ForParallel{
        \State $s^i
            \leftarrow  
            \gamma\,\mathcal{W}_1(\hat{\mu}_1,\hat{\mu}_i)
            -
            \frac{\gamma}{|\mathcal{D}_i|^{1/(d+1)}}
        $
        \tcp*{Score $i^{th}$ dataset}
        \State 
        $T^i
        \leftarrow
        I_{s^i\le \eta}
        $
        \tcp*{Test: Is Dataset $i$ Similar to Focal Dataset?}
        } 
\State 
$
w^{\star}\leftarrow \operatorname{Softmin}\big(
s
\big)^{\top}\,
T
$

    \State \Return Optimal Weights $w^\star$
    .
\end{algorithmic}
\end{algorithm}

Suppose now that we have used Algorithm \ref{alg:RegretOptimizationHeuristic__WarmStart} to compute weights $w^\star=(w^\star_1, \cdots, w^\star_N)$ and the locally optimal fKRR parameters $\theta^\star_1$, ..., $\theta^\star_N$, and $w^{\star}$. The key novelty in the next algorithm follows a similar spirit to \cite{JMLR_MaximumBasedDL_LiChenTaiCheng} and \cite{casgrain2021optimizing}, by casting the problem of constructing a regret-optimal algorithm as an optimal control problem. 
We denote 
\allowdisplaybreaks\begin{align} 
 & \Theta(t)\eqdef [ \theta_1^{\top}(t), \cdots, \theta_N^{\top}(t) ]^{\top}, \quad 
 \Theta^\star \eqdef [ \theta^{\star \top}_1, \cdots, \theta^{\star \top}_N ]^{\top} , 
 \notag \\
 & \boldsymbol\alpha(t) \eqdef [\alpha_1^{\top}(t), \cdots, \alpha_N^{\top}(t)]^{\top}, \notag   
\end{align} 
which are all elements of $\R^{Np}$, and introduce the dynamics that $\Theta$ follows    \allowdisplaybreaks\begin{align} 
& \Theta(0)=\Theta^\star, 
\quad 
\Theta(t+1) = \Theta(t) + \boldsymbol\alpha(t) , 
\quad 0\leq t \leq T-1 ,   
\label{Theta-simple}  
\end{align} 
which is equivalently written as 
$
 \Theta(t+1) = \Theta^\star + \sum_{u=0}^{t} \boldsymbol\alpha(u) 
$ for $t=0,\dots,T-1$.

Then, since the optimizers of the energy~\eqref{eq:penalized_regret_functional} coincide with the optimizers of the systemic regret functional~\eqref{eq:penalized_regret_functional__REGRETFORM}, searching for a regret-optimal algorithm reduces to solving for an optimal control $\mathbf{\alpha}$ to minimize the cost 
\eqref{eq:penalized_regret_functional} subject to \eqref{Theta-simple}.  
The optimal control problem 
\eqref{eq:penalized_regret_functional} and \eqref{Theta-simple} is a discrete-time linear quadratic optimal control problem, and can be solved by dynamic programming methods that are standard in the literature. For instance, by \cite[Proposition 5.1]{BO1998}, we obtain the solution for the regret-optimal algorithm in closed-form as solutions of a system of Riccati equations \eqref{eqnP(t) 0} and \eqref{eqnS(t) 0}, given in
Theorem~\ref{thm:RegretOptimal_Dynamics__MainTextFormulation}. 

Algorithm~\ref{alg:RegretOptimization} below implements the optimal control approach in two steps. In Step 1, we solve the equation system \eqref{eqnP(t) 0} and \eqref{eqnS(t) 0} backwards for $t=T$, $T-1$, ..., $1$. Then in Step 2, we obtain the optimal control $\boldsymbol\alpha(t)$ and update the parameter $\Theta(t)$ for $t=0$, $1$, ... $T$. 

\begin{algorithm}[ht]
\caption{Regret-Optimal Optimization Algorithm}
\label{alg:RegretOptimization}
\begin{algorithmic}
\SetAlgoLined
\Require Datasets $\mathcal{D}_1,\dots,\mathcal{D}_N$, $N$.\ Iterations $T\in \mathbb{N}_+$, finite-rank kernel $\phi$ \\ and hyperparameters $\lambda, \beta, \kappa, \eta >0$.
\DontPrintSemicolon
    \State 
    \tcp{Get Initialize Weights and Locally-Optimal fKRR Parameters}
        \State $\theta^\star_1,\dots,\theta^\star_N,w^{\star}$ $\leftarrow$ Run: Algorithm~\ref{alg:RegretOptimizationHeuristic__WarmStart} with $\mathcal{D}_1,\dots,\mathcal{D}_N$, $\phi$, and $\kappa,\eta$.
   
    \State 
    \tcp{Initialize Updates}
    \State $P(T) 
    = [w_1^\star I_p , \cdots, w_N^\star I_p ]^\top \big(\sum_{i=1}^N w^\star_i \sum_{j=1}^{|\mathcal{D}_i|}  u^i_j u_j^{i\top}
    \big) [w_1^\star I_p, \cdots, w_N^\star I_p ] $ 

\State $S(T) = -[w_1^\star I_p , \cdots, w_N^\star I_p ]^\top \sum_{i=1}^N w^\star_i \sum_{j=1}^{|\mathcal{D}_i|}  u^i_j y^i_j$ 

\State $\Theta(0) = \Theta^\star \eqdef (\theta_1^{\star \top}, \cdots, \theta_N^{\star \top} )^\top$ 

    \State \tcp{Generate Iterates}
    \State \For{$t=T-1,\dots,1$}{
        \State \tcp{Update Driving Parameters} 

\State $P(t)= \beta I_{Np} - \beta^2 [(\lambda+\beta)I_{Np} + P(t+1)]^{-1}$

\State $S(t) = \beta [(\lambda + \beta )I_{Np} + P(t+1)]^{-1} (S(t+1) - \lambda \Theta^\star)$
}

\State\For{$t=0, ..., T-1$}{
  \State\tcp{Update Control} 

\State 
$\boldsymbol{\alpha}(t) = -[(\lambda+\beta )I_{Np} + P(t+1)]^{-1} 
[(\lambda I_{Np} + P(t+1))\Theta(t) - \lambda \Theta^\star + S(t+1) ]$

\State 
$\Theta(t+1)=\Theta(t)+\boldsymbol{\alpha}(t)$
}

\State \Return Return the Optimized fKRR $f_{\theta^{w^\star}}$
\end{algorithmic}
\end{algorithm} 

\begin{remark}
By decoupling the optimization of the ``dataset-relevant weights'' $w$ in Algorithm~\ref{alg:RegretOptimizationHeuristic__WarmStart} from the optimization of the parameter $\theta$ in Algorithm ~\ref{alg:RegretOptimization}, we avoid open-loop forward-backward iterations, which iterate between optimizing $w$ and $\theta$. 
\end{remark}

In the next section, we state our primary results, whose proofs are presented in Appendix \ref{s:Proofs_TechnicalLemmata}. We first quantify the optimality of Algorithm~\ref{alg:RegretOptimizationHeuristic__WarmStart}. Next, we establish the optimality, complexity, and adversarial robustness of Algorithm~\ref{alg:RegretOptimization}.  

\section{Main Guarantees}
\label{s:Main}

We now present our theoretical guarantees for Algorithms~\ref{alg:RegretOptimizationHeuristic__WarmStart} and~\ref{alg:RegretOptimization}.

\subsection{{Statistical Guarantees - Algorithm~\ref{alg:RegretOptimizationHeuristic__WarmStart}}}

Assume that either $d$ or $D$ are greater than $1$.  
Let $\mathcal{Z}\subset \mathbb{R}^d\times \mathbb{R}^D$ be compact of diameter $1$ (with respect to the Euclidean metric).
Let $\mu_1,\dots,\mu_N\in \mathcal{P}_1(\mathcal{Z})$, $\mathcal{D}_1,\dots,\mathcal{D}_N\in \mathbb{N}_+$, let $\mathcal{D}_i\eqdef \{(X^i_j,Y^i_j)\}_{j=1}^{|\mathcal{D}_i|}$ and 
$\mathcal{D}\eqdef \{\mathcal{D}_i\}_{i=1}^{N}$ be independent random variables where, for $i=1,\dots,N$, $
(X^i_j,Y^i_j)\sim \mu_i$, and consider the empirical measure
\[
        \hat{\mu}_i
    \eqdef 
        \frac1{|\mathcal{D}_i|}\, \sum_{j=1}^{|\mathcal{D}_i|}\, \delta_{(X^i_j,Y^i_j)}
.
\]
Let $\ell:\mathbb{R}^D\times \mathbb{R}^D\to \mathbb{R}$ be a $L_{\ell}$-Lipschitz loss function.
For every measurable $f:\mathbb{R}^d\to \mathbb{R}^D$ and every $w\in \Delta_N$, we consider its associated \textit{weighted} empirical $\hat{\mathcal{R}}_{w}^{\mathcal{D}}(f)$ and true risk $\mathcal{R}(f)$
\[
        \mathcal{R}(f)
    \eqdef 
        \mathbb{E}_{(X,Y)\sim \mu_1}\big[
            \ell(f(X),Y)
        \big] 
\,\,\mbox{ and }\,\,
        \hat{\mathcal{R}}_{w}^{\mathcal{D}}(f)
    \eqdef 
        \sum_{i=1}^N\,
            \frac{w_i}{|\mathcal{D}_i|}
            \,
            \sum_{j=1}^{|\mathcal{D}_i|}
                \ell(f(X_j^i),Y_j^i)
.
\]
Note that, if $w=e_1\eqdef (1,0,\dots,0)$, then $\hat{\mathcal{R}}_{w}^{\mathcal{D}}(f)$ is precisely the definition of the usual empirical risk of $f$ computed \textit{only} on the focal dataset $\mathcal{D}_1$. 

For $i=1,\dots,N$, define the \textit{true similarity score} of the $i^{th}$ dataset  
\begin{equation}
\label{eq:true_scored_defined}
        s_i
    \eqdef 
        \mathcal{W}_1(\mu_1,\mu_k).
\end{equation}
\begin{theorem}[Non-Asymptotic Transfer Learning Guarantee (General Lipschitz Learners)]
\label{thrm:transfer_bound}
Fix $L\ge 0$, let $\mathcal{F}$ be a non-empty family of $L$-Lipschitz functions mapping $\mathbb{R}^d$ to $\mathbb{R}^D$.  There exists a constant $C\ge 1$ (depending only on $d+D$ and on $\mathcal{Z}$) such that: for every $0< \delta \le 1$, each $\eta\ge 0$, $\gamma>0$ and every $w\in \Delta_N$, the following holds with probability at-least $1-\delta$
\begin{equation}
\label{eq:Bound_defining_Algo_1_loss}
\resizebox{0.91\hsize}{!}{$%
\sup_{f\in \mathcal{F}_L}
\,
    \big|
            \mathcal{R}(f)
        -
            \hat{\mathcal{R}}_{w}^{\mathcal{D}}(f)
    \big|
\le 
        \bar{L}C\,
        \underbrace{
            \overbrace{
                \sum_{i=1}^N
                    w_i
                    \,
                    \biggl(
                           s_i
                        +
                        \frac{1}{|\mathcal{D}_i|^{1/(d+D)}}
                    \biggr)
            }^{
            \text{Power of i$^{th}$ Dataset}
            }
            +
            \frac{\bar{L}C\,}{\gamma}
            \overbrace{
            \sum_{i=1}^N\,
                w_i \,
                \log\Big(
                    C_{\eta}\, 
                    \frac{w_i}{
                        I_{\mathcal{W}_1(\hat{\mu}_1,\hat{\mu}_i)\le \eta}
                    }
                \Big)
            }^{\text{Entropic Penalty:} 
                \operatorname{KL}\big(
                    \mathbb{P}_w
                \|
                    \bar{\mathbb{P}}^{\eta}
                \big)
            }
        }_{\text{Bayesian Objective } \term{t:objective}}
    +
        \underbrace{
            \frac{
                \sqrt{
                    \bar{L}\,
                    \ln(2 N /\delta)
                }
            }{
                \sqrt{2N^{\star}}
            }
        }_{\text{PAC Term } \term{t:high_prob_convergence_term}}
$}%
\end{equation}
where $\bar{L}\eqdef 
L_{\ell} \max\{1,L\}$,  
$
\operatorname{KL}(\mathbb{P}_w\|\bar{\mathbb{P}}^{\eta})
=
\sum_{i=1}^N\,
            \frac{
                I_{\mathcal{W}_1(\hat{\mu}_1,\hat{\mu}_i)\le \eta}
            }{C_{\eta}}
        \,
        \delta_i
$, and $C_{\eta}\eqdef \sum_{u=1}^N\,
                    I_{\mathcal{W}_1(\hat{\mu}_1,\hat{\mu}_u)\le \eta}
$.
\\
If $w$ and $\eta$ are chosen such that $\mathbb{P}_w\ll \bar{\mathbb{P}}^{\eta}$ and $
|w-e_1|,
\eta
\operatorname{KL}\big(
    \mathbb{P}_w
\|
    \bar{\mathbb{P}}^{\eta}
\big)
\in
\mathcal{O}\big(\frac{1}{
    {N^{\star}}^{1/(d+D)}
}\big)
$ then 
\begin{equation*}
\label{eq:Bound_defining_Algo_1_loss}
\sup_{f\in \mathcal{F}_L}
\,
    \big|
            \mathcal{R}(f)
        -
            \hat{\mathcal{R}}_{w}^{\mathcal{D}}(f)
    \big|
\in
\mathcal{O}\big(1/
    {N^{\star}}^{1/(d+D)
}\big) .
\end{equation*}
\end{theorem}

In Theorem~\ref{thrm:transfer_bound} term~\eqref{t:objective} could be optimized whereas~\eqref{t:high_prob_convergence_term} did not depend on the weights $w\in \Delta_N$; thus it could not.  Now, even if the scores $\{s_i\}_{i=1}^N$ are not directly observable, thus term~\eqref{t:objective} is not directly computable from the data, the main result of~\cite{MR4700254} shows that in our setting the \textit{empirical scores} $\{\hat{s}_i\}_{i=1}^N$ converge to the actual scores as the number of data points in each dataset becomes arbitrarily large; where for each $i\in [N]$
\[
        \hat{s}_i
    \eqdef 
        \mathcal{W}_1\big(
            \hat{\mu}_1,\hat{\mu}_i
        \big) . 
\]
Moreover, the main result of~\cite{MR4700254} further guarantees that each $\hat{s}_i$ converges to $s_i$ at-least at the optimal rate of $\mathcal{O}(1/\max\{|\mathcal{D}_1|,|\mathcal{D}_i|\}^{1/(d+D)})$.  Consequentially, minimizing the following objective will minimize~\eqref{t:objective} and thus, it will minimize the upper-bound in Theorem~\ref{thrm:transfer_bound} for the gap between the true risk of any $L$-Lipschitz model (such as any aggregated model) and the weighted empirical risk.  This motivates our performance objective
\begin{equation}
\begin{aligned}
\label{eq:performance_functional}
 &  \min_{w\in \Delta_N}
    \,
    \bigg\{ 
    \underbrace{ 
            \sum_{i=1}^N
            w_i
            \,
            \Big(
                \mathcal{W}_1\big(
                    \hat{\mu}_1,\hat{\mu}_i
                \big)
            +
                \frac{1}{|\mathcal{D}_i|^{1/(d+D)}}
            \Big) 
        }_{\term{t:Objective_Terminal}}      
  +
        \underbrace{ 
            \frac{1}{\gamma}\,
            \operatorname{KL}\big(
                \mathbb{P}_w
            \|
                \bar{\mathbb{P}}^{\eta}
            \big)
        }_{\term{t:EntropicRegularizer_Terminal}} 
        \bigg\} 
\end{aligned} 
\end{equation} 
where we treat $C,\eta>0$ as a hyperparameters.

Let $\mathcal{P}([N])$ consist of all probability measures on the set $[N]\eqdef \{1,\dots,N\}$.  Any probability measure $\mathbb{P}_w$ in $\mathcal{P}([N])$ is uniquely determined by a weight $w$ in the $N$-simplex; consisting of all $w\in [0,1]^N$ whose entries sum to $1$.  Let $I_w$ be a random dataset index in $[N]$ distributed according to $\mathbb{P}_w\eqdef \sum_{i=1}^N\, w_i \delta_i$.  The objective function in~\eqref{eq:performance_functional} is the following \textit{Bayesian optimization problem}
\begin{equation}
\begin{aligned} 
\label{eq:optimizingQ}
  &   \hspace{-0.1cm}   \min_{w\in \Delta_N}
    \,
      \bigg\{  \underbrace{
            \mathbb{E}_{I_w\sim \mathbb{P}_w}
            \Big[
                \Big(
                    \mathcal{W}_1\big(
                        \hat{\mu}_1,\hat{\mu}_{I_w}
                    \big)
                +
                    \frac{1}{|\mathcal{D}_{I_w}|^{1/(d+D)}}
                \Big)
            \Big]
        }_{\term{t:Objective_Terminal}} 
   +
        \underbrace{
            \frac{1}{\gamma}
            \,
            \operatorname{KL}\big(
                \mathbb{P}_w
            \|
                \bar{\mathbb{P}}^{\eta}
            \big)
        }_{\term{t:EntropicRegularizer_Terminal}} 
        \bigg\} .
\end{aligned}
\end{equation}
A direct application of~\citep[Proposition 1]{WangHyndmanKratsios_EMT_2020} shows that there exists a unique optimizer $w^{\star}\in \Delta_N$ of the strictly convex problem~\eqref{eq:optimizingQ} and for each $i=1,\dots,N$ it is given by
\[
    w^{\star}_i
    =
    \frac{
        e^{
        -
        \gamma\,\mathcal{W}_1(\hat{\mu}_1,\hat{\mu}_i)
        -
        \frac{\gamma}{|\mathcal{D}_i|^{1/(d+D)}}
        }\,
        I_{\mathcal{W}_1(\hat{\mu}_1,\hat{\mu}_i)\le \eta}
    }{
        \sum_{u=1}^N\,
        e^{-\gamma\mathcal{W}_1(\hat{\mu}_1,\hat{\mu}_u)
        -\frac{\gamma}{N_u^{1/(d+D)}}
        }\,
            I_{\mathcal{W}_1(\hat{\mu}_1,\hat{\mu}_u)\le \eta}
    }
.
\]
We emphasize that Theorem~\ref{thrm:transfer_bound} is fully general and equally holds for complicated learners such as deep neural networks optimizing arbitrary loss functions, just as much as it holds for our specialized case of kernel ridge regressor case optimizing the MSE loss.
\subsection{Discussion: Optimal Weighting}
Notice, that, as $\eta\to 0$ then $w_i^{\star}=I_{1=i}$ and as $\eta\to \infty$ then $w_i$ tends to the ``unthresheld soft\textit{min} function'' applied to the vector of ``dataset relevances'' $\big(
\mathcal{W}_1(\hat{\mu}_1,\hat{\mu}_i) + \mathcal{D}_i^{-1/(d+D)}
\big)_{i=1}^N$ with temperature parameter $\gamma$.
Let us note that term~\eqref{t:Objective_Terminal} in the Bayesian objective~\eqref{eq:optimizingQ} improves by a factor of 
\allowdisplaybreaks
\begin{align}
 & \hspace{-0.1cm}   \operatorname{T-Gain}_{\eta,\gamma}
    \eqdef  
    \sum_{i=1}^N\,
        \Biggl[
                \frac{
                    e^{
                    -
                    \gamma\,\mathcal{W}_1(\hat{\mu}_1,\hat{\mu}_i)
                    -
                    \frac{\gamma}{|\mathcal{D}_i|^{1/(d+D)}}
                    }\,
                    I_{\mathcal{W}_1(\hat{\mu}_1,\hat{\mu}_i)\le \eta}
                }{
                    \sum_{u=1}^N\,
                    e^{-\gamma\mathcal{W}_1(\hat{\mu}_1,\hat{\mu}_u)
                    -\frac{\gamma}{N_u^{1/(d+D)}}
                    }\,
                        I_{\mathcal{W}_1(\hat{\mu}_1,\hat{\mu}_u)\le \eta}
                }
        -
            I_{i=1}
        \Biggr]
        \,
        \Big(
            \mathcal{W}_1(\hat{\mu_1},\hat{\mu}_i)
            +
            \frac{1}{|\mathcal{D}_i|^{1/(d+D)}}
        \Big)
\end{align} 
when performing transfer learning using the optimal weights $w^{\star}$ rather than performing no transfer learning (i.e.\ weights $w=e_1$).  
The quantity $\operatorname{T-Gain}_{\eta,\gamma}$ in the above expression, thus, quantifies an upper bound on the gain from performing transfer learning using our optimal weighting scheme versus not performing any transfer learning.  

{
Setting $\gamma=0$ yields the equal weighting scheme $w_i^{\star}=1/N$ for each $i$ \textit{included in the transfer learning problem}.  In this case, the gain reduces to
\begin{equation} 
\begin{aligned} 
\label{eq:Weighting_equal}
&    \operatorname{T-Gain}_{\eta,0}
    =
    \sum_{i=1}^N\,
        \Biggl(
                \frac{
                    I_{\mathcal{W}_1(\hat{\mu}_1,\hat{\mu}_i)\le \eta}
                }{
                    \sum_{u=1}^N\,
                        I_{\mathcal{W}_1(\hat{\mu}_1,\hat{\mu}_u)\le \eta}
                }
            -
            I_{i=1}
        \Biggr) 
        \Big(
            \mathcal{W}_1(\hat{\mu_1},\hat{\mu}_i)
            +
            \frac{1}{|\mathcal{D}_i|^{1/(d+D)}}
        \Big) 
\end{aligned} 
\end{equation} 
and we see that $\eta$ controls how many datasets are included based on their similarity to the focal dataset $\mathcal{D}_1$ (encoded into the empirical measure $\hat{\mu}_1$).  Thus, any dataset $\mathcal{D}_i$ is deemed to be too dissimilar to the focal dataset $\mathcal{D}_1$ if its associated empirical measure $\hat{\mu}_i$ is at a distance of $\eta$ or greater from $\hat{\mu}_1$.  In this case, it is \textit{not included} into the transfer learning problem.
\hfill\\
For instance, in the extreme case where $\eta=0$, then~\eqref{eq:Weighting_equal} reduces to $0$ and no gain is made.  Conversely, if $\eta$ is large ($\eta\ge \eta^{\star}\eqdef \max_{i=1,\dots,N}\, \mathcal{W}_1(\hat{\mu}_1,\hat{\mu}_i)$ then,~\eqref{eq:Weighting_equal} simplifies to
\[ 
    \operatorname{T-Gain}_{\eta^{\star},0}
    =
    \sum_{i=1}^N\,
        \Big(
                \frac{
                    1
                }{
                    N
                }
            -
            I_{i=1}
        \Big)
        \,
        \Big(
            \mathcal{W}_1(\hat{\mu_1},\hat{\mu}_i)
            +
            \frac{1}{|\mathcal{D}_i|^{1/(d+D)}}
        \Big)
.
\] 
Finally, we remark that sending $\gamma$ to infinity in~\eqref{eq:optimizingQ}, implies that $w^{\star}$ tends to $e_1$ and $\lim_{\gamma \to 0} w^*_i = \frac{I_{\mathcal{W}_1(\hat{\mu}_1,\hat{\mu}_i)\le \eta}}{\sum_u I_{\mathcal{W}_1(\hat{\mu}_1,\hat{\mu}_u)\le \eta} }$.  In this case, no transfer can happen and thus $
    \lim\limits_{\gamma \to \infty}\, \operatorname{T-Gain}_{\eta,\gamma}=0
$.
}

\subsection{{Optimization Guarantees - Algorithm~\ref{alg:RegretOptimizationHeuristic__WarmStart}}and~\ref{alg:RegretOptimization}}
This section contains the description of the \textit{unique} regret-optimal algorithm and an analysis of its theoretical properties. Throughout our manuscript, we maintain the following two assumptions.  
\begin{assumption}[$\mathcal{D}$-Boundedness]
\label{assm:bddxy} 
There are constants $K_x,K_y>0$ such that all 
$(x^i_j, y^i_j) \in \mathcal{D}$ satisfy $\|u^i_j\|_2 \leq K_x$ and $|y^i_j|^2 \leq K_y$, where $u^i_j = \phi (x^i_j)$. 
\end{assumption}
\vspace{-.5em}
With the notation introduced before, the 
object $[w_1^\star I_p, \cdots, w_N^\star I_p ]$ above is a $p \times Np$ matrix.
Our next result shows that Algorithm~\ref{alg:RegretOptimization}, and its iterates, are regret-optimal. 

\begin{theorem}[Algorithm~\ref{alg:RegretOptimization} Computes the Unique Regret-Optimal Algorithm]
\label{thm:RegretOptimal_Dynamics__MainTextFormulation}
Suppose that Assumption~\ref{assm:bddxy} holds and fix $\Theta^{\star}\in \mathbb{R}^{N\,p}$ as in \eqref{Theta_star}.  Then, the systemic regret functional $\mathcal{R}$ defined in~\eqref{eq:penalized_regret_functional__REGRETFORM} admits a unique minimizer $\hat\Theta_{\cdot}\eqdef (\hat\Theta(t))_{t=0}^T$ of the form~\eqref{Theta-simple__nocontrol} with $\hat\Theta(0) = \Theta^{\star}$ and increments
\allowdisplaybreaks
\begin{equation} 
\hspace{-0.3cm}
 \Delta \hat\Theta(t)  
 =  -[(\lambda + \beta )I_{Np} + P(t+1)]^{-1} 
  \big[ (\lambda I_{Np} + P(t+1)) \hat\Theta(t) 
 - \lambda  \Theta^\star + S(t+1)  \big] ,   \label{optimala 0} 
\end{equation}
where, for $t=1,\dots,T$,  the matrices $P(t)$ and $S(t)$ in \eqref{optimala 0} are determined by
\allowdisplaybreaks\begin{align} 
& \begin{cases}
 P(T) =  [w^\star_1 I_p ,\cdots, w^\star_N I_p]^{\top} 
 \big( \sum_{i=1}^N w^\star_i \sum_{j=1}^{|\mathcal{D}_i|}   u^i_j u^{i\top}_j \big)  
    [ w^\star_1 I_p, \cdots, w^\star_N I_p] ,\\
   P(t)  =  \beta I_{Np} - \beta^2 [(\lambda+\beta) I_{Np} + P(t+1) ]^{-1}  \label{eqnP(t) 0} ,
\end{cases}   \\
& \begin{cases}
 S(T) =  - [ w^\star_1 I_p, \cdots, w^\star_N I_p]^{\top} \sum_{i=1}^N w^\star_i \sum_{j=1}^{|\mathcal{D}_i|}  {u^i_j} y^i_j ,\\
  S(t) =  \beta \big[ (\lambda+\beta)I_{Np} + P(t+1) \big]^{-1} (S(t+1) - \lambda \Theta^\star ) .
  \label{eqnS(t) 0}    
\end{cases} 
\end{align} 
In particular, for each $t=0,\dots,T-1$ we have $
        \Delta \hat\Theta(t) 
    = 
        \boldsymbol{\alpha}(t)
$, where $\boldsymbol{\alpha}$ is as in Algorithm~\ref{alg:RegretOptimization}. 
\end{theorem}

Next, we evaluate the computational complexity of Algorithm~\ref{alg:RegretOptimization} in the regime where the number of datasets is much larger than the typical number of instances/data per dataset.
\begin{theorem}[Computational Complexity]
\label{thm:ComplexityAlgorithm}
If $\frac1{N}\sum_{i=1}^N\,|\mathcal{D}_i| \le N$, then
the computational complexity
of computing the sequence $(\hat \Theta(t))_{t=0}^T$ in the regret-optimal algorithm is $
    \mathcal{O}\big( N^2 p^3 +
        T
        \, 
        (Np)^{2.373}
    \big)
.
$
\end{theorem}



\paragraph{Adversarial Robustness}

We quantify the sensitivity of Algorithm~\ref{alg:RegretOptimization} to adversarial attacks.  We study the robustness of Algorithm~\ref{alg:RegretOptimization} to a family of perturbations that switch a given percentage $0\le q\le 1$ of all training data by selecting fake data points at a distance of at-most $\varepsilon\ge 0$ for any perturbed datapoint.  Thus, for a fixed set of datasets 
$\mathcal{D} = \{ \mathcal{D}_i \}_{i=1}^N$, and $q,\varepsilon$ as above, 
we define the class $\mathbb{D}^{q,\varepsilon}$ of \textit{adversarially generated datasets} of ``magnitude'' $(q,\varepsilon)$ as all sets of datasets 
$\{ \widetilde{\mathcal{D}}_i \}_{i=1}^N$, $\widetilde{\mathcal{D}}_i=\{ \{(\tilde{x}^i_j,\tilde{y}^i_j) \, | \, 1 \leq j \leq |\mathcal{D}_i| \}\subseteq \mathbb{R}^{d}\times \mathbb{R}$ satisfying Assumption~\ref{assm:bddxy}, as well as
\[
        \underbrace{
                    \frac{
                        \#\{
                            (i,j):\, u^i_j\neq \tilde{u}^i_j \mbox{ or } y^i_j\neq \tilde{y}^i_j
                        \}
                    }{
                        \sum_{i=1}^N\,|\mathcal{D}_i|
                    }
                \le 
                    q
        }_{\mbox{Attack Persistence}} 
\hspace{0.2cm} 
\mbox{and} 
\hspace{0.2cm}       \underbrace{
                    \max_{i,j} \left\{ \,
                        \max\left(
                            \|u^i_j-\tilde{u}^i_j\|_2 
                        ,
                            |y^i_j-\tilde{y}^i_j|^2
                        \right) \right\}
                \le 
                    \varepsilon
        }_{\mbox{Attack Severity}}
,
\]
where the indices $(i,j)$ run over $\{(i,j) \, | \, 1 \leq i \leq N,\,1 \leq j \leq |\mathcal{D}_i|\}$ and $\tilde{u}^i_j\eqdef \phi(\tilde{x}^i_j)$.

The updates of the unique regret-optimal algorithm $\hat \Theta^{0\cdots T}$ are determined by the sequence $\boldsymbol{\alpha}_{\cdot}\eqdef (\boldsymbol{\alpha}(t))_{t=0}^T$ computed by Algorithm~\ref{alg:RegretOptimization}.  It is convenient to emphasize this connection on $\boldsymbol{\alpha}_{\cdot}$ and the connection with the dataset $\mathcal{D}$ when expressing the following theorem, with the energy~\eqref{eq:penalized_regret_functional} determining the value of the systematic regret functional $\mathcal{R}$.  We therefore write $
    {\cal L}(\boldsymbol{\alpha};\mathcal{D})
        \eqdef 
    {\cal L}(\hat\Theta^{0\cdots T})
.$
\begin{theorem}[Adversarial Robustness]
\label{thrm:adv_robust}
Fix a $\mathcal{D}$ satisfying Assumption~\ref{assm:bddxy}.  For any level of ``attack persistence'' $0\le q\le 1$ and any degree of ``attack severity'' $\varepsilon\ge 0$ we have that
\[
        \big| {\cal L} (  \boldsymbol\alpha ; \mathcal{D} ) 
        -  
        {\cal L}(  \widetilde{\boldsymbol\alpha} ;  \widetilde{\mathcal{D}} )  \big|
    \le 
        \mathcal{O}\big(
            \varepsilon
            \sqrt{\bar{N}\,q}
        \big)
,
\]
where 
$\widetilde{\boldsymbol\alpha}$ is the output of Algorithm~\ref{alg:RegretOptimization} for any given adversarially-generated set of datasets 
$\{ \widetilde{\mathcal{D}}_i \}_{i=1}^N\in \mathbb{D}^{q,\varepsilon}$ and $\bar{N}=\sum_{i=1}^N |\mathcal{D}_i|$.  Here $\mathcal{O}$ hides a constant depending only on Assumption~\ref{assm:bddxy} and on $\sqrt{p}$.  
\end{theorem}

\section{Ablation Study}
\label{s:Experiments_n_Validation}
We experimentally ablate our theory, first studying the convergence of the regret-optimal algorithm compared to standard gradient descent. Then we focus on \textit{American option pricing} in quantitative finance, although our method could equivalently be used for any other transfer learning problem; see, for example, all the works surveyed in \citet{weiss2016survey}. We consider random kernels whose feature map $\phi$ in~\eqref{eq:FRKR} is a randomized neural network with $\operatorname{ReLU}$ activation%
%
%
; details are in Appendix~\ref{a:ExperimentDetails}.
Since there are no available benchmarks for the regret-optimal meta-optimization problem, we only benchmark against the standard gradient descent in our first set of experiments.

\subsection{Convergence of the Regret-Optimal Algorithm}
\label{sec:Convergence of the Regret Optimal Algorithm}
We first study the evolution of the energy functional \eqref{eq:penalized_regret_functional} and loss achieved in each iteration by our regret-optimal Algorithm~\ref{alg:RegretOptimization} (RO) with weights initialized by Algorithm~\ref{alg:RegretOptimizationHeuristic__WarmStart} by computing ${\cal L}(\Theta^{0\cdots t})$ as defined in \eqref{eq:penalized_regret_functional} but for intermediate iterations $t$, as well as the loss $l(\Theta(t), w;\,\mathcal{D})$. 
We compare the regret-optimal algorithm with the results when using gradient descent (GD) to optimize the loss directly with the version of the RO algorithm where each $w_i$ is fixed to $1/N$ (ARO).

Figure~\ref{fig:convergence analysis} shows the averaged results for $10$ runs using $\lambda=0$ each with a different set of randomly generated datasets and a different randomized neural network as fKRR. For the regret-optimal algorithm and for its accelerated version we use $T=10^3$ iterations, while we use $10^5$ steps for gradient descent with learning rate $7 \cdot 10^{-5}$. We note that for learning rates $\geq 8 \cdot 10^{-5}$ the training becomes unstable such that the loss explodes. 
The gradient descent method gradually, but slowly, reduces the loss.
Since the gradient descent method makes small parameter updates (and since $\lambda=0)$, the regret is nearly $0$ and the energy ${\cal L}$ nearly equal to the loss. Even though it runs $100$ times the number of iterations of our algorithms, it still needs much more training to achieve the same loss.

\begin{figure}[!ht]
        \centering            
        \includegraphics[width=0.2\textwidth]{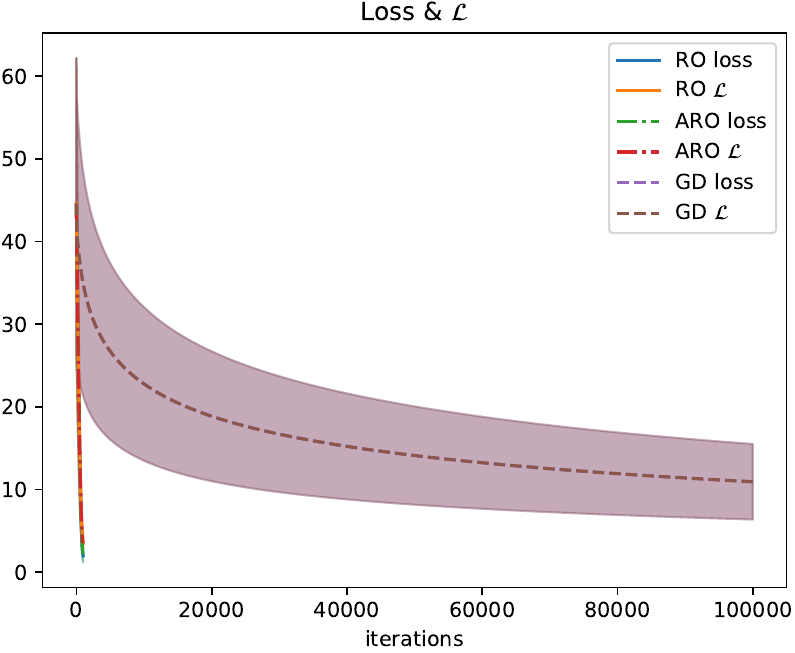}
        \includegraphics[width=0.2\textwidth]{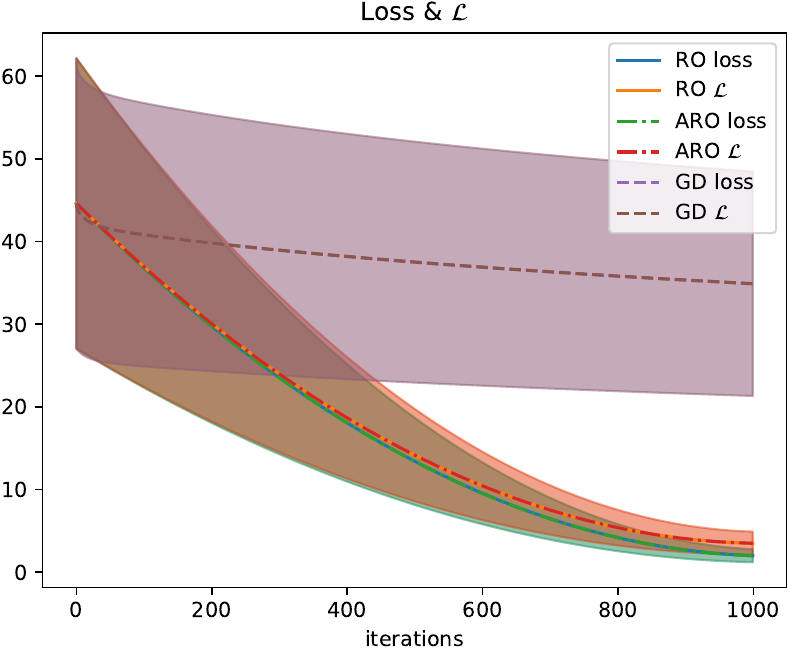}
        \includegraphics[width=0.2\textwidth]{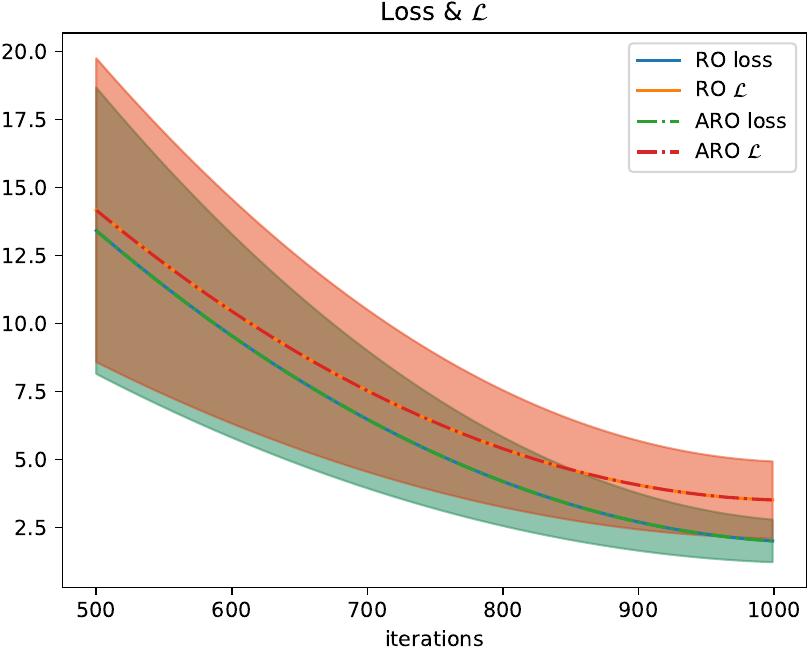}
        \caption{Loss and energy ${\cal L}$ for gradient descent (GD), the regret-optimal algorithm (RO) and the accelerated regret-optimal algorithm (ARO). Means and standard deviations over 10 runs are shown. Left: all iterations; middle: first 1000 iterations; right: iterations 500 to 1000 for our algorithms.}
        \label{fig:convergence analysis}
\end{figure}

\begin{figure}[!ht]
        \centering            
        \includegraphics[width=0.2\textwidth]{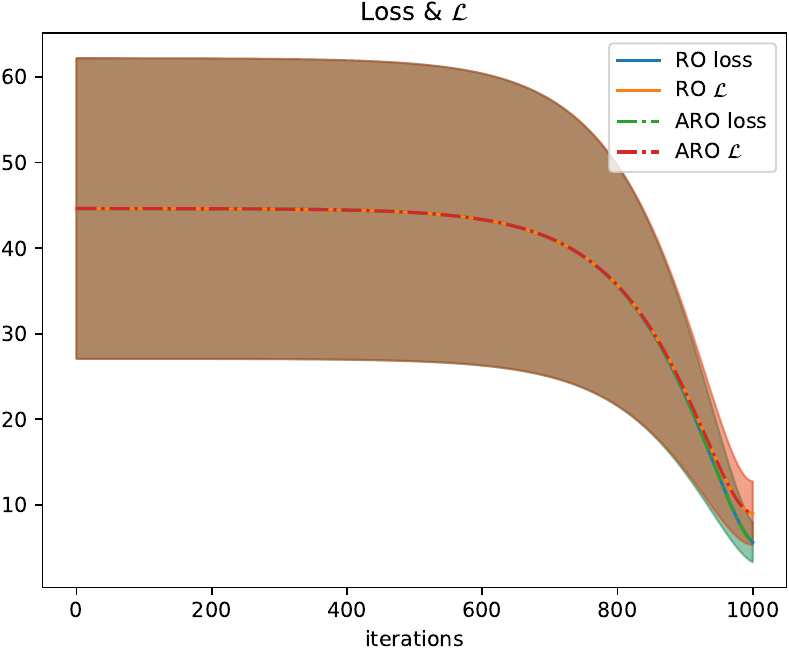}
        \includegraphics[width=0.2\textwidth]{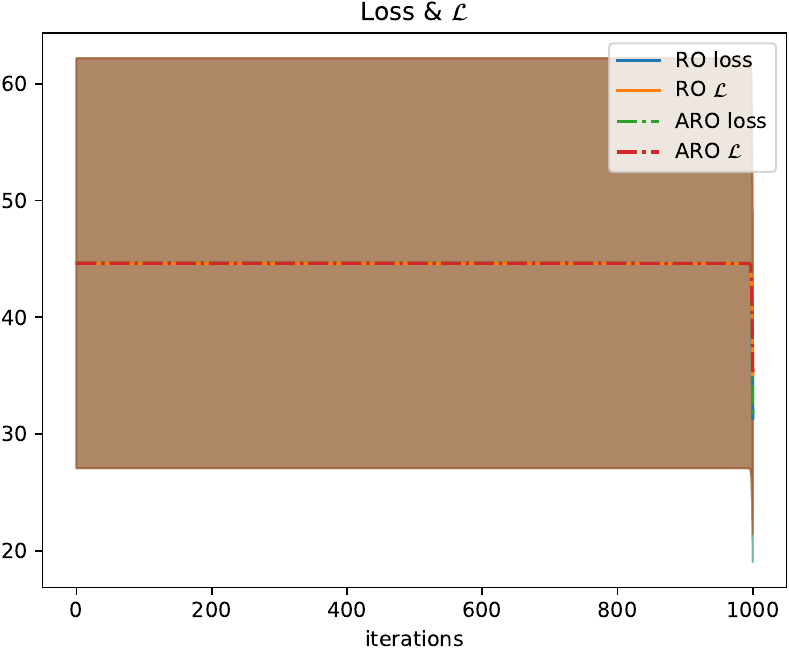}
        \caption{Loss and energy ${\cal L}$ for the regret-optimal algorithm (RO) and the accelerated regret-optimal algorithm (ARO) with means and standard deviations over $10$ runs. Left: $\lambda = 10^{-4}$; right: $\lambda=2$.}
        \label{fig:convergence analysis 2}
\end{figure}

In the regime where $\lambda > 0$ the behaviour of the algorithm changes drastically. In particular, since deviating from $\Theta(0) = \Theta^\star$ is penalised, the algorithm remains at the starting point in the initial steps and only starts to converge when approaching $T$. The exact behaviour depends on the choice of $\lambda$ (and is additionally influenced by the choice of $\beta$). In Figure~\ref{fig:convergence analysis 2} we compare the behaviour for $\lambda=10^{-4}$ and $\lambda = 2$ while otherwise using the same setup as before.  For the smaller value of $\lambda$, the algorithm starts the decent earlier and decreases the loss by more compared to the larger value for $\lambda$, where the algorithm deviates from $\Theta^\star$ only in the very last steps. 

\vspace{-1em}
\subsection{Transfer Learning in American Option Pricing}
\label{sec:American Option Pricing}
We test our algorithm in the context of the optimal stopping problem for American option pricing, a computationally highly challenging problem which has recently become tractable also in higher dimensional settings through deep learning \citep{PatrickDOSJMLR_2019,Ruimeng_QF_OptimStopping,Timo_DOS_2021,herrera2023optimal}. This is a useful example since American option pricing (see Appendix~\ref{a:Recap:AOP_OS}) is computationally expensive, and one often does not have access to large datasets but does have access to several small and similar datasets arising from comparable market regimes.

We generate our synthetic ``true market dynamics'' $\mathcal{D}_1$ and several more or less similar datasets $\mathcal{D}_2, \dotsc, \mathcal{D}_N$. We always use the randomized least squares Monte Carlo (RLSM) algorithm of~\cite{herrera2023optimal} (see Appendix~\ref{a:Recap:AOP_OS}) for American option pricing, and we only vary the optimization method used to train the randomized neural networks in the RLSM.
We compare RO against the following baselines. \textit{Local optimizer (LO-$i$):} The local optimizers $\theta^\star_i$ solves the individual optimisation problem on the datasets $\mathcal{D}_i$ without considering the other datasets. \textit{Mean Local optimizer (MLO):} The mean of the local optimizers $ \bar{\theta}^\star = \frac{1}{N} \sum_{i=1}^N \theta^\star_i$ is the bagging aggregation. \textit{Joint optimizer (JO):} The joint optimizer ${\theta}^\star_{\mathcal{N}}$ with $\mathcal{N}=\{ 1, \dotsc, N\}$ solves the optimisation problem of the pooled regime $\mathcal{D} = \cup_{i=1}^N \mathcal{D}_i$. \textit{Joint Subset optimizer (JSO-$\mathcal{I}$):} For any subset $\mathcal{I} \subset \mathcal{N}$ this is the joint optimizer of the dataset $\mathcal{D}_{\mathcal{I}} = \cup_{i \in \mathcal{I}} \mathcal{D}_i$.

\begin{experiment}\label{exp:1}{\rm 
We study the ability of our method to transfer knowledge of a large number of small equally-sized similar and dissimilar datasets  to the main task.  We generate $13$ different datasets (the main one, 6 similar and 6 dissimilar ones; see \ref{sec:Experiment 1}) with $\nu_1 = 100$ training samples each,  all from the well-studied \textit{Heston model}, defined by
    \begin{equation}\label{equ:heston model}
    \begin{split}
    d X_t &= (r - \delta) X_t dt + \sqrt{v_t} X_t d W_t, \\
    d v_t &= -k (v_t - v_{\infty}) dt + \sigma \sqrt{v_t}d B_t.
    \end{split}
    \end{equation}

    Results are reported in Table~\ref{table:results exp 1 short} and more detailed in \ref{sec:Experiment 1}.
    For our regret-optimal method (RO) we use information sharing levels $\eta \in \{  10, 100, 500\}$, where $\eta=100$ leads to the best result outperforming the local optimizer of the main dataset (LO-1) by $9\%$. It also outperforms all other local optimizers (LO-$i$), the mean local optimizer (MLO) and the joint optimizer (JO).
    For $\eta=10$ we have similar performance as LO-1 and for $\eta=500$ we also have significant outperformance, though smaller than for $\eta=100$ (see \ref{sec:Experiment 1}). 
    
    The local optimizers on the 6 datasets similar to the main one (LO-2, $\dotsc$, LO-7) have similar performance as LO-1, while the other 6 (LO-8, $\dotsc$, LO-13) are significantly worse.
    Therefore, we claim that our algorithm should mainly transfer knowledge from those 6 similar datasets, while avoiding to learn from the others.
    In line with this we provide results of the joint optimizer on the similar datasets $\mathcal{D}_1 \cup \dotsb \cup \mathcal{D}_7$, which performs better than the other baselines and outperforms LO-1 by $6.2\%$. In particular, this confirms that knowledge can be transferred from the datasets $\mathcal{D}_2 , \dotsc , \mathcal{D}_7$. This joint subset optimizer is significantly outperformed by our regret-optimal algorithm RO ($\eta=100$) by about $3\%$, which shows that the transfer learning abilities of our method are better than simply selecting the datasets most similar to $\mathcal{D}_1$.
    
    We additionally compare our method to the local optimizer for the main dataset which uses either $700$ (the number of training samples in the 7 similar datasets) or $50,000$ training samples. Importantly, both are just references that would not be available in a real world setting where the datasets are limited.
    The first one only leads to a slight increase in $RP$ of about $1\%$, which shows that our method extracts nearly as much knowledge out of the (same amount of) samples of the similar (but not equally distributed) datasets. 
    The second one (which is considered a good proxy for the true American option price on the main dataset) has about $10\%$ better performance. This also shows the limitations of our method, in particular, it can only transfer as much knowledge as available in the other datasets.}
\end{experiment}

\vspace{-1em}
\begin{table}[!htb]
    \caption{\textbf{Ablation Results: American Option Pricing.}
    Relative performance (RP) and $95\%$-confidence-intervals for different optimization methods.  
    We compare the local optimizers on the different datasets (LO-$n$), with the mean local optimizer (MLO), the joint optimizer (JO) and our regret-optimal method (RO). The \textit{oracle} local optimizer on the main dataset with additional training samples (standard: $100$ samples per dataset) is included.  
    See Table~\ref{table:results exp 1} for additional details.
    }
    \begin{subtable}{.5\linewidth}
      \centering
            \caption{\textbf{Experiment $1$} (compressed)}
            \resizebox{\linewidth}{!}{%
            \begin{tabular}{@{}ccc@{}}
            \cmidrule[0.3ex](){1-3}
            method &  $RP$ &  $95\%$-CI \\
            \midrule
            LO-$1$ & 1.000 & $[0.991; 1.009]$ \\
            LO-2, $\dotsc$, LO-7 & $\geq 0.966$ & $[\geq 0.957; \leq 0.994]$ \\
            LO-8, $\dotsc$, LO-13 & $\leq 0.725$ & $[\geq 0.696; \leq 0.736]$ \\
            MLO & 0.823 & $[0.813; 0.833]$ \\
            JO & 0.886 & $[0.878; 0.894]$ \\
            JSO-$\{1,\dotsc,7\}$ & 1.062 & $[1.057; 1.067]$ \\
            \textbf{RO ($\eta=100$)} &  \textbf{1.090} & \textbf{[1.086; 1.095]} \\
            \midrule
            \begin{filecontents*}{Tables/training_overview-1-2.csv}
,name,RP,CI
2,LO-1 (700 tr. samp.),1.100,$[1.097; 1.104]$
1,LO-1 (50K tr. samp.),1.194,$[1.192; 1.195]$




\end{filecontents*}
\csvreader[head to column names, late after line=\\]
            {Tables/training_overview-1-2.csv}{}%
            {  \name & \RP & \CI }%
            \bottomrule
            \end{tabular}
            }
            \label{table:results exp 1 short}
    \end{subtable}%
    \begin{subtable}{.5\linewidth}
      \centering
        \caption{\textbf{Experiment $2$}}
        \resizebox{0.825\linewidth}{!}{%
        \begin{tabular}{@{}ccc@{}}
        \cmidrule[0.3ex](){1-3}
        method &  $RP$ &  $95\%$-CI \\
        \midrule
        \begin{filecontents*}{Tables/training_overview-2.csv}
,name,RP,CI
0,LO-1,1.000,$[0.993; 1.007]$
1,LO-2,0.773,$[0.765; 0.782]$
2,LO-3,1.003,$[0.994; 1.012]$
4,MLO,0.932,$[0.920; 0.944]$
3,JO,0.763,$[0.754; 0.772]$
8,RO ($\eta=10$),0.993,$[0.985; 1.000]$
9,RO ($\eta=50$),1.040,$[1.034; 1.046]$
10,\textbf{RO ($\eta=100$)},\textbf{1.040},\textbf{[1.034; 1.047]}

\end{filecontents*}
\csvreader[head to column names, late after line=\\]
        {Tables/training_overview-2.csv}{}%
        {  \name & \RP & \CI }%
        \bottomrule
        \end{tabular}
        }
        \label{table:results exp 2}
    \end{subtable} 
\end{table}

\begin{experiment}\label{exp:2}{\rm 
    We study the adversarial robustness of our model when having a dominating dataset that should be avoided due to its dissimilarity.
    In particular, we generate 3 datasets, a small main dataset, another small and similar dataset to transfer knowledge from and a dominating much larger dissimilar datasets.
    The main dataset is generated from a rough Heston model 
    \begin{align*}
    dX_t &=  (r - \delta) X_t dt +  \sqrt{v_t} X_t dW_t, \\
    v_t &= v_0 + \int_0^t \frac{(t-s)^{H-1/2} }{\Gamma(H + 1/2)} \kappa (v_\infty - v_s)ds  
     + \int_0^t \frac{(t-s)^{H-1/2} }{\Gamma(H + 1/2)} \sigma \sqrt{v_s}dB_s.
    \end{align*}
    which is similar to the Heston model except that the volatility is a rough process with Hurst parameter $H\in(0,1/2]$ instead of $H=1/2$.
    Datasets $\mathcal{D}_2, \mathcal{D}_3$ are from a standard Heston model, where $\mathcal{D}_2$ is dissimilar from $\mathcal{D}_1$ and $\mathcal{D}_3$ is similar to it.
    For $\mathcal{D}_1$ and $\mathcal{D}_3$ we use $\nu_1=100$ training samples, while $\mathcal{D}_2$ has $\nu_1=50,000$ samples, such that it dominates the others in its size.
    All the other choices are the same as in Experiment~\ref{exp:1}.

    The results are reported in Table~\ref{table:results exp 2}. The baselines MLO and  JO perform significantly worse than the LO on the main dataset. As expected, the JO suffers much more from the dominating $\mathcal{D}_2$ than the MLO, since JO weights all samples equally. Indeed, the MLO would suffer more from a large number of datasets that are quite different from the main dataset, no matter their sample size.
    In contrast to this, our regret optimal method with $\eta \in\{50,100\}$ outperforms the LO-1 by $4\%$, showing that also a quite different dataset dominating in sample size is not a problem.}
\end{experiment}


Both of these experiments show that our regret optimal method is well suited for automatic (up to the hyper-parameter $\eta$) dataset selection in transfer learning tasks that have potentially large numbers of datasets and do not allow a (manual) pre-selection. 
In particular, neither a large number of ``bad'' datasets (of which the MLO suffers) nor some ``bad'' datasets that dominate in sample size (of which the JO suffers) constitute a problem for our regret-optimal method.
Moreover, the sample efficiency when transferring knowledge from similar datasets to the main task is very high (outperforming the JO on the respective subset of similar datasets and nearly being on par with the LO on the main dataset using a larger training sample), as we saw in Experiment~\ref{exp:1}.

\vspace{-1em}
\section{Conclusion}
\vspace{-0.5em}
In this work we presented a regret-optimal algorithm for federated transfer learning based on pre-trained models where only the last linear layer is fine-tuned. 
Besides the theoretical properties of this algorithm, we also provided experiments demonstrating the transfer learning capabilities of our method in the context of American option pricing.
From the theoretical side, we would like to extend our method in future work to regression models that do not necessarily permit a closed-form solution.

\vspace{-1em}
\section{Limitations and Future Work}
\vspace{-0.5em}
Our results specialized on kernel ridge regressors.  In future work, we would like to extend our pipeline to deep learning models.

%


\bibliography{Formatting/References}
\bibliographystyle{icml2021}

\newpage
\appendix


\section{Game Theoretic Interpretation of Regret-Optimal Algorithm}
\label{s:GameTheoreticInterpretation}

We now offer a game-theoretic interpretation of the regret functional, as defined in~\eqref{eq:penalized_regret_functional}.  This interpretation views the user as a central planner whose objective is to organize a system of individual agents, representing the pretrained KRR models, to maximize the singular goal of identifying a parameter optimizing~\eqref{eq:penalized_regret_functional__REGRETFORM}.  
We now interpret the roles and interactions of the central planner and each agent as defined by our regret-optimization problem.
\paragraph{The Central Planner:}
The purpose of the central planner is to avoid the situation whereby the structure of the different datasets is ignored in the model selection problem by merging them into a single dataset $\mathcal{D}\eqdef \cup_{i=1}^N\,\mathcal{D}_i$ and thereby reducing the problem to standard kernel ride regression. Though intuitively simple, such merging  approach can be particularly disadvantageous in the presence of heterogeneity among the datasets, since a dataset that is potentially very different from the focal dataset $\mathcal{D}_1$ would be equally influential in the regression problem~.  
Instead, the user, acting as a central planner, organizes the model selection problem through a cooperative objective
\allowdisplaybreaks\begin{align} 
\label{eq:terminal_time_objective appendix}
        l(\Theta, w;\,\mathcal{D} )  
    \eqdef & 
        \sum_{i=1}^N 
        w_i
            \sum_{j=1}^{|\mathcal{D}_i|} 
                (
                f_{\theta^w}(x_j^i)
                - y^i_j  
                )^2
    \,
    \mbox{ with }
    \,
    \theta^w\eqdef  \sum_{i=1}^N 
        w_i\,\theta_i ,
\end{align} 
where $\Theta\eqdef (\theta_1^{\top}, \cdots, \theta_N^{\top})^\top\in \mathbb{R}^{Np \times 1}$. Here, the central planner \textit{ascribes} the influence of each agent on the model selection problem through the weight vector $w=(w_1, \cdots, w_N)\in [0,1]^N$ with $\sum_{i=1}^N\,w_i=1$.  
As can be seen in \eqref{eq:terminal_time_objective appendix}, this influence is two-fold: on the one hand, it impacts the cooperative objective $l$ in the aggregation of the SSE of each player; on the other hand, it affects the proportion of each player's preferred parameter choice entering into the collectively selected parameter $\theta^{w}$.
The choice of $w$ is approached in different ways in the literature, e.g.\ by bagging 
reminiscent of mean-field games \citep{carmona2018probabilistic}, by data-driven weighting procedures \citep{Baxter_2000_JAIR__AmodelIndutvieBiasLearning}, Bayesian aggregation \citep{MR2280214,PACOH_rothfuss_2021,pavasovic2022mars}, or with dictionaries \citep{argyriou2006multi}.  

\paragraph{The Agents:}
When acting in isolation, each agent's preferred model is determined by optimizing a the kernel ridge regression objective using only their dataset.  Under mild conditions\footnote{E.g.\ coercivity and lower semi-continuity in $\theta$.} on $f$, this corresponds to an {\em individual} parameter selection
\begin{align}
    \theta^{\star}_i \in 
    \operatorname{argmin}_{\theta \in \Theta}
        \, 
            \sum_{(x,y)\in \mathcal{D}_i}\,
                (f_\theta(x)-y)^2
        +
            \kappa \,
                \|\theta\|_2, 
\label{eq:Regression_Klassische1}
\end{align}
for $i=1,\dots,N$, which we assume to be fixed and known to all agents prior to their {\em collective} parameter selection. Taking into account the other agents, each agent \textit{acts} by specifying an iterative deterministic algorithm of the form with the intent of maximizing their influence on the joint model selection problem~\eqref{eq:terminal_time_objective appendix}.  Thus, the $i^{th}$ player wants the jointly selected model, specified by $\theta^{w}(T)=\sum_{n=1}^N\, w_i\,\theta_i(T)$, to be as close to $\theta_i^{\star}$ as possible, thereby maximally encoding the characteristics of the $i^{th}$ dataset into the selected model $f_{\theta^w(T)}$. The accumulated \textit{regret} which the $i^{th}$ player incurs by deploying an algorithm $\theta^{0\cdots T}_{i}$ that deviates from their preferred selection $\theta^*_i$ is measured by
\begin{equation}
\label{eq:penalized_regret_functional__REGRETFORM___Individual}
\begin{aligned}
            \sum_{t=0}^{T-1}  
            \,
                \underbrace{
                    \lambda \,
                    \left\| \theta_i(t+1) - \theta^\star_{i} \right\|_2^2
                }_{\text{Preference Strength}}
            + 
                \underbrace{
                      \beta \, \left\| \Delta \theta_i(t) \right\|_2^2 
                }_{\text{Algorithm Stability}}
,
\end{aligned}
\end{equation}
where $\lambda>0$ is a hyperparameter quantifying the player's \textit{attachment} towards their preferred model choice $\theta^{\star}_i$ and the hyperparameter $\beta > 0$ quantifies the the stability of the algorithm during the iterations, where  $\Delta \theta_i(t)\eqdef \theta_i(t+1)-\theta_i(t)$.  

The central planner then organizes the action $\Theta^{0\cdots T}\eqdef 
\left\{\Theta(t)\right\}_{t=0}^{T}$ of the system of $N$ players to reach the objective~\eqref{eq:terminal_time_objective appendix} while encoding their individual regrets~\eqref{eq:penalized_regret_functional__REGRETFORM___Individual}, where
\begin{equation}
\Theta (t) = (\theta_1(t)^\top,\ldots,\theta_N(t)^\top)^\top \in \mathbb{R}^{Np \times 1}. \notag 
\end{equation}

This is achieved by coupling the individual agent's regret functionals~\eqref{eq:penalized_regret_functional__REGRETFORM___Individual} through the \textit{systemic regret functional}
\begin{equation}
\label{eq:penalized_regret_functional__REGRETFORM appendix}
\begin{aligned}
    \mathcal{R}(\Theta^{0\cdots T})
\eqdef &
        \underbrace{
        l (\Theta(T), w;\,\mathcal{D} ) 
        }_{\text{Cooperative Objective}}
    -
                  \,\,  l^{\star} 
   +
        \sum_{t=0}^{T-1}  
        \,
             \big(
                    \underbrace{
                        \lambda \left\| \Theta(t+1) - \Theta^\star \right\|_2^2    
                    }_{\text{Preference Strength}}
                + 
                    \underbrace{
                        \beta  \left\| \Delta \Theta(t) \right\|_2^2 
                    }_{\text{Algo. Stability}}
                \big) 
    ,
\end{aligned}
\end{equation} 
where
\begin{equation}
    \label{Theta_star}
    \Theta^\star\eqdef (\theta_1^{\star \top}, \cdots, \theta_N^{\star\top})^\top
    \in \mathbb{R}^{Np \times 1}
\end{equation}
encodes the individual preferences of the players,
\begin{equation}
    \Delta \Theta(t)\eqdef \Theta(t+1)-\Theta(t) 
    \label{Theta-simple__nocontrol}
\end{equation}
quantifies the integrative updates of any candidate optimizing sequence $\Theta^{0\cdots T}$, and the \textit{ideal terminal loss} is 
\[
        l^{\star}
    \eqdef 
        \underbrace{
            \min_{\Xi \in \mathbb{R}^{Np\times 1}}\, 
                l(\Xi,w; \mathcal{D}) 
        }_{\text{Cooperative Sub-optimality}}
    .
\]
The term $(l(\cdot,\cdot;\cdot)-l^{\star})$ measures the central planner's regret in failing to optimize~\eqref{eq:terminal_time_objective appendix}.  The algorithm stability term in~\eqref{eq:penalized_regret_functional__REGRETFORM appendix} appears in forward-backward proximal splitting algorithms with quadratic objectives (e.g. \cite{MR3619041,MR4490614}) and several standard proximal algorithms (see \cite[Section 5.1]{bertsekas2015convex}).

The systemic regret functional~\eqref{eq:penalized_regret_functional__REGRETFORM appendix} acts as a performance criterion for any optimization algorithm, played by $N$ players, initialized at an arbitrary $\Theta(0)$, and running for $T$ iterations.  
In what follows, we take the initial condition of the algorithm as $\Theta(0)=\Theta^\star$ for simplicity, but other choices are possible. In other words, the regret of each player is initialized at zero, as they all start from their individually preferred parameter $\theta^\star_i$, and starts to accumulate as the system moves away from $\Theta^\star$ towards the optimizer of  \eqref{eq:penalized_regret_functional__REGRETFORM appendix}.

\section{Proofs and Technical Results}
\label{s:Proofs_TechnicalLemmata}

\begin{proof}[{Proof of Theorem~\ref{thrm:transfer_bound}}]
For every $L(>0)$-Lipschitz $f\in \mathcal{F}$ and each $w\in \Delta_N$, we have that
\allowdisplaybreaks
\begin{align*}
\numberthis
\label{eq:begin_ineq_empiricialpreocessbound}
   \big|
            \mathcal{R}(f)
        -
            \hat{\mathcal{R}}_{w}^{\mathcal{D}}(f)
    \big|
= &
    \biggl|
            \mathbb{E}_{(X,Y)\sim \mu_1}\big[
            \ell(f(X),Y) \big] 
        -
            \sum_{i=1}^N\,
                w_i
            \mathbb{E}_{(X,Y)\sim \hat{\mu}_i}\big[
                \ell(f(X),Y)
            \big]
    \biggr|
\\
\numberthis
\label{eq:simplexbelonging}
= &
    \biggl|
            \sum_{i=1}^N
                w_i\,
                \mathbb{E}_{(X,Y)\sim \mu_1}\big[
                \ell(f(X),Y) 
                \big] 
        -
            \sum_{i=1}^N\,
                w_i
            \mathbb{E}_{(X,Y)\sim \hat{\mu}_i}\big[
                \ell(f(X),Y)
            \big]
    \biggr|
\\
= &
    \biggl|
            \sum_{i=1}^N
                w_i\,
                \biggl(
                    \mathbb{E}_{(X,Y)\sim \mu_1}\big[
                    \ell(f(X),Y) 
                    \big] 
                -
                \mathbb{E}_{(X,Y)\sim \hat{\mu}_i}\big[
                    \ell(f(X),Y)
                \big]
            \biggr)
    \biggr|
\\
\le &
    \sum_{i=1}^N
        w_i\,
    \biggl|
        \mathbb{E}_{(X,Y)\sim \mu_1}\big[
            \ell(f(X),Y) 
            \big] 
        -
        \mathbb{E}_{(X,Y)\sim \hat{\mu}_i}\big[
            \ell(f(X),Y)
        \big]
    \biggr|
\\
\numberthis
\label{eq:KRDuality}
\le &
    \sum_{i=1}^N
        w_i\,
        \bar{L}
        \,
        \mathcal{W}_1(\mu_1,\hat{\mu}_i)
\\
\le &
    \sum_{i=1}^N
        w_i\,
        \bar{L}
        \,
        \big(
            \mathcal{W}_1(\mu_1,\mu_i)
            +
            \mathcal{W}_1(\mu_i,\hat{\mu}_i)
        \big)
\\
\numberthis
\label{eq:scores_definition_used}
= &
    \sum_{i=1}^N
        w_i\,
        \bar{L}
        \,
        \big(
            s_i
            +
            \mathcal{W}_1(\mu_i,\hat{\mu}_i)
        \big)
\\
\numberthis
\label{eq:truescores_definition_used}
= &
        \bar{L}
        \sum_{i=1}^N
            w_i\,
               s_i
    +
        \bar{L}
        \sum_{i=1}^N
            w_i\,
            \mathcal{W}_1(\mu_i,\hat{\mu}_i)
.
\end{align*}
where~\eqref{eq:simplexbelonging} held since $w$ belongs to the $N$-simplex $\Delta_N$,
~\eqref{eq:KRDuality} held by Kantorovich duality (since $\mathcal{Z}$ is compact every probability thereon belongs to the $1$-Wasserstein space over $(\mathcal{Z},\|\cdot\|_2)$) and since the map $\mathcal{Z}\in (x,y)\mapsto \ell(f(x),y)\in \mathbb{R}$ is $
\bar{L}
\eqdef \max\{1,L\}$-Lipschitz, see e.g.~\cite[Theorem 5.10 as simplified in Equation (5.11)]{VillaniOTBook_2009},
~\eqref{eq:scores_definition_used} held by definition of the scores $\{s_i\}_{i=1}^N$ (see~\eqref{eq:true_scored_defined}).
Since the upper bound in~\eqref{eq:truescores_definition_used} held independently of the choice of $L$-Lipschitz $f\in \mathcal{F}$, then taking the supremum over all $L$-Lipschitz function $f\in \mathcal{F}$ across
\eqref{eq:begin_ineq_empiricialpreocessbound}-\eqref{eq:truescores_definition_used} yields
\begin{equation}
\label{eq:sup_learning_bound}
    \sup_{f\in \mathcal{F}_L}\,
    \big|
            \mathcal{R}(f)
        -
            \hat{\mathcal{R}}_{w}^{\mathcal{D}}(f)
    \big|
\le 
        \bar{L}
        \sum_{i=1}^N
            w_i\,
               s_i
    +
        \bar{L}
        \sum_{i=1}^N
            w_i\,
            \mathcal{W}_1(\mu_i,\hat{\mu}_i)
\end{equation}
where $\mathcal{F}_L\eqdef \{f\in \mathcal{F}:\, \operatorname{Lip}(f)\le L\}$ (note that we can include the case where $\operatorname{Lip}(f)=0$ trivially).

Let $[N]\eqdef \{1,\dots,N\}$.
Next, for each $t>0$, we bound the following probability beginning with a simple union bound and then applying~\cite[Lemma 16]{hou2023instance} $N$ times (once for each $\mathcal{W}_1(\mu_i,\hat{\mu}_i)$)
\allowdisplaybreaks
\begin{align*}
\numberthis
\label{eq:UnionBound_BEGIN}
& \mathbb{P}\big(
        (\forall i \in [N])\,
            \big|
                \mathcal{W}_1(\mu_i,\hat{\mu}_i)
                -
                \mathbb{E}\big[
                    \mathcal{W}_1(\mu_i,\hat{\mu}_i)
                \big]
            \big|
        < t
    \big)
\\
& \ge 
    1 - 
        \sum_{i=1}^N
        \,
        \mathbb{P}\big(
            \big|
                \mathcal{W}_1(\mu_i,\hat{\mu}_i)
                -
                \mathbb{E}\big[
                    \mathcal{W}_1(\mu_i,\hat{\mu}_i)
                \big]
            \big|
        \ge 
        t
    \big)
\\
& \ge 
    1- 
    \sum_{i=1}^N\,
       2
       e^{-2\mathcal{D}_i\,t^2}
\\
\numberthis
\label{eq:UnionBound_END}
& \ge
        1
    - 
       2
       N
       e^{-2N^{\star}\,t^2}
\end{align*}
(we have used the fact that $\operatorname{diam}(\mathcal{X})=1$)
where $N^{\star}\eqdef \max_{i=1,\dots,N}\, \mathcal{D}_i$.
Let $\delta \in (0,1]$ and solve the following for $t$: 
\[
    \delta 
    =
    2
       N
       e^{-2N^{\star}\,t^2}
\]
yields
\begin{equation}
\label{eq:desired_delta_confidenceLevel}
        t 
    = 
        \biggl(
            \frac{\ln(2 N /\delta)}{2N^{\star}}
        \biggr)^{1/2}
.
\end{equation}
Thus,~\eqref{eq:desired_delta_confidenceLevel} and our union bound in~\eqref{eq:UnionBound_BEGIN}-\eqref{eq:UnionBound_END} imply that: with probability at-least $1-\delta$ the following holds 
\begin{equation}
\label{eq:measure_bound__noExpect}
    \max_{i\in [N]}
        \mathcal{W}_1(\mu_i,\hat{\mu}_i)
    -
        \mathbb{E}\big[
            \mathcal{W}_1(\mu_i,\hat{\mu}_i)
        \big]
    \le
        \biggl(
            \frac{\ln(2 N /\delta)}{2N^{\star}}
        \biggr)^{1/2}
.
\end{equation}
Applying the bound on the expected $1$-Wasserstein distance between the empirical measure $\hat{\mu}_i$ and the true data-generating measure for the $i^{th}$ dataset once for each $i\in [N]$,~\cite[Lemma 16 and Table 2]{hou2023instance}, to~\eqref{eq:measure_bound__noExpect} implies that there exists some constant $C\ge 1$ (depending only on $d+D$ and on $\mathcal{Z}$) such that: the following holds with probability at-least $1-\delta$
\begin{equation}
\label{eq:measure_bound}
 \hspace{-0.2cm}
    \max_{i\in [N]}
        \mathcal{W}_1(\mu_i,\hat{\mu}_i)
    -
        \frac{C}{|\mathcal{D}_i|^{1/(d+D)}}
    \le
        \biggl(
            \frac{\ln(2 N /\delta)}{2N^{\star}}
        \biggr)^{1/2}
,
\end{equation}
where we note that we have used the fact that $d+D\ge 3$ since either $d>1$ or $D>1$ (to avoid the critical regime where $d=D$ which gives a an additional $\log(N)$ factor to the second term in~\eqref{eq:measure_bound}).
Merging~\eqref{eq:measure_bound} with the bounds in~\eqref{eq:sup_learning_bound} implies that: with probability at-least $1-\delta$ we have
\begin{align*}
\numberthis
\label{eq:nearly_completed_bound}
& \sup_{f\in \mathcal{F}_L}
\,
    \big|
            \mathcal{R}(f)
        -
            \hat{\mathcal{R}}_{w}^{\mathcal{D}}(f)
    \big|
\\ 
\le &
        \bar{L}
        \sum_{i=1}^N
            w_i\,
               s_i
    +
        \bar{L}
        \sum_{i=1}^N
            w_i\,
        \biggl(
                \frac{C}{|\mathcal{D}_i|^{1/(d+D)}}
            +
                \biggl(
                    \frac{\ln(2 N /\delta)}{2N^{\star}}
                \biggr)^{1/2}
        \biggr)
\\
= &
        \bar{L}
        \sum_{i=1}^N
            w_i\,
               s_i
    +
        \bar{L}
        \sum_{i=1}^N
            w_i\,
            \frac{C}{|\mathcal{D}_i|^{1/(d+D)}}
    +
        \bar{L}\,
        \sum_{i=1}^N
            w_i\,
                \biggl(
                    \frac{\ln(2 N /\delta)}{2N^{\star}}
                \biggr)^{1/2}
\\
\numberthis
\label{eq:simplex_again}
= &
        \bar{L}
        \sum_{i=1}^N
            w_i\,
               s_i
    +
        C\,
        \bar{L}
        \sum_{i=1}^N
            \,
            \frac{w_i}{|\mathcal{D}_i|^{1/(d+D)}}
    +
        \bar{L}\,
        \biggl(
            \frac{\ln(2 N /\delta)}{2N^{\star}}
        \biggr)^{1/2}
\\
= &
        \bar{L}
        \sum_{i=1}^N
            w_i
            \,
            \Big(
               s_i
            +
            \frac{C}{|\mathcal{D}_i|^{1/(d+D)}}
            \Big)
    +
        \bar{L}\,
        \frac{
            \sqrt{
                \ln(2 N /\delta)
            }
        }{
            \sqrt{2N^{\star}}
        }
\\
\numberthis
\label{eq:general_bound__RHS}
\le &
        \bar{L}\,C
        \sum_{i=1}^N
            w_i
            \,
            \Big(
               s_i
            +
            \frac{1}{|\mathcal{D}_i|^{1/(d+D)}}
            \Big)
    +
        \bar{L}\,
        \frac{
            \sqrt{
                \ln(2 N /\delta)
            }
        }{
            \sqrt{2N^{\star}}
        }
\end{align*}
where~\eqref{eq:simplex_again} again follows from the fact that $w\in \Delta_N$.  
In particular, if $
|w-e_1|\in
\mathcal{O}\big(\frac{1}{
    {N^{\star}}^{1/(d+D)}
}\big)
$ then: for every $\delta>0$, the right-hand side of~\eqref{eq:general_bound__RHS} converges to $0$ and is of the order of $\mathcal{O}\big(\frac{1}{
    {N^{\star}}^{1/(d+D)}
}\big)$.

Finally, for each $\eta\ge 0$ define the probability measure $\bar{\mathbb{P}}^{\eta}$ on $[N]$ by
\[
        \bar{\mathbb{P}}^{\eta}
    \eqdef 
        \sum_{i=1}^N\,
            \frac{
                I_{\mathcal{W}_1(\hat{\mu}_1,\hat{\mu}_i)\le \eta}
            }{
                \sum_{u=1}^N\,
                    I_{\mathcal{W}_1(\hat{\mu}_1,\hat{\mu}_u)\le \eta}
            }
        \,
        \delta_i
    \eqdef 
    \sum_{i=1}^N\,
            C_{\eta}^{-1}
                I_{\mathcal{W}_1(\hat{\mu}_1,\hat{\mu}_i)\le \eta}
        \,
        \delta_i
,
\]
where $C_{\eta}\eqdef \sum_{u=1}^N\,
                    I_{\mathcal{W}_1(\hat{\mu}_1,\hat{\mu}_u)\le \eta}$.
Then, for each $w\in \Delta_N$, consider the probability measure $\mathbb{P}_w\eqdef \sum_{i=1}^N\, w_i\delta_i$ on $[N]$.  The KL-divergence between $\mathbb{P}_w$ and $\bar{\mathbb{P}}^{\eta}$ is non-negative and given by
\[
    \operatorname{KL}(\mathbb{P}_w\|\bar{\mathbb{P}}^{\eta})
=
    \sum_{i=1}^N\,
        w_i \log\Big(
            C_{\eta}\, 
            \frac{w_i}{
                I_{\mathcal{W}_1(\hat{\mu}_1,\hat{\mu}_i)\le \eta}
            }
        \Big)
.
\]
Let $w\in \Delta_N$ be such that $\mathbb{P}_w\ll \bar{\mathbb{P}}^{\eta}$; so that $\operatorname{KL}(\mathbb{P}_w\|\bar{\mathbb{P}}^{\eta})$ is finite.  Then, for every $\eta\ge 0$ and $\gamma>0$ the right-hand side of~\eqref{eq:general_bound__RHS} can be further bounded above by the following \textit{finite} value: for every $0<\delta \le 1$ the following holds with probability at-least $1-\delta$
\begin{align}
\label{eq:completed_bound}
\sup_{f\in \mathcal{F}_L}
\,
    \big|
            \mathcal{R}(f)
        -
            \hat{\mathcal{R}}_{w}^{\mathcal{D}}(f)
    \big|
& \le 
        \bar{L}
        C\,
        \sum_{i=1}^N
            w_i
            \,
            \Big(
               s_i
            +
            \frac{1}{|\mathcal{D}_i|^{1/(d+D)}}
            \Big)
\\
\nonumber
&
    +
        \frac{\bar{L}C\,}{\gamma}
        \sum_{i=1}^N\,
        w_i \log\Big(
            C_{\eta}\, 
            \frac{w_i}{
                I_{\mathcal{W}_1(\hat{\mu}_1,\hat{\mu}_i)\le \eta}
            }
        \Big)
\\
\nonumber
    &
    +
        \bar{L}\,
        \frac{
            \sqrt{
                \ln(2 N /\delta)
            }
        }{
            \sqrt{2N^{\star}}
        }
.
\end{align}
If, $\mathbb{P}_w\ll \bar{\mathbb{P}}^{\eta}$ and, further, $w$ and $\eta$ are chosen such that $
|w-e_1|,
\eta\,
        \operatorname{KL}\big(
            \mathbb{P}_w
        \|
            \bar{\mathbb{P}}^{\eta}
        \big)
\in
\mathcal{O}\big(\frac{1}{
    {N^{\star}}^{1/(d+D)}
}\big)
$ then the right-hand side of~\eqref{eq:completed_bound} is of the order of $\mathcal{O}(1/N^{1/(d+D)})$.
\end{proof}

\begin{lemma}[$(w^{\star},\mathcal{D})$-Compatibility]
\label{lm:PD}
The data ${\cal D} = \cup_{i=1}^N\,\mathcal{D}_i$ and the weights $w^{\star}$ satisfy 
\allowdisplaybreaks
\begin{align} 
&
\sum_{i=1}^N \sum_{j=1}^{|\mathcal{D}_i|} u^i_j u^{i\top}_j \geq 0 , 
 \label{Sum(uij)Tuij>=0}\\  
& 
[w^\star_1 I_p ,\cdots, w^\star_N I_p]^{\top} 
 \bigg( \sum_{i=1}^N
 w^\star_i \,
 \sum_{j=1}^{|\mathcal{D}_i|}  u^i_j u^{i\top}_j \bigg) 
  [ w^\star_1 I_p, \cdots, w^\star_N I_p ] \geq 0 
 .
 \label{WeightedSum(uij)Tuij>=0}
 \end{align} 
\end{lemma} 
\begin{remark}
\label{rem:Notation_InnerProduct}
We keep our notation light and use $u^{\top}v$ to implement the Euclidean inner-product between any two vectors $u,v\in \mathbb{R}^k$ for any $k\in \mathbb{N}_+$. 
We use $| \cdot |$ to denote the Euclidean norm of a vector or Fr\"{o}benius norm of a matrix, whichever is applicable. 
For two square matrices $A$ and $B$, $A\leq B$ ($A<B$, resp.) means that $B-A$ is positive semi-definite (positive definite, resp.); see \citep[page 146 - Equation (7)]{LaxBookLinearAlgebra_2007} for details.
\end{remark} 

Lemmas~\ref{lm:|AB|<=|A||B|} and \ref{lm:|A1+An|2<=(|A1|2+|An|2)n} give some elementary properties of matrix algebra, which will be used repeatedly in the subsequent proofs. 
\begin{lemma} 
\label{lm:|AB|<=|A||B|}
The Euclidean norm $|\cdot|$ of matrices has the following properties: 

\noindent\emph{i)}. For any matrices $A\in \mathbb{R}^{n_1 \times n_2}$ and $B\in\mathbb{R}^{n_2\times n_3}$, we have that
$|AB| \leq |A|\cdot |B|$;   

\noindent\emph{ii)}. For any matrices $A$, $B\in \mathbb{R}^{n_1 \times n_2}$, we have $|A+B| \leq |A|+|B|$.  
\end{lemma} 
\begin{proof}  
Denote $A=(A_{ij})_{1\leq i\leq n_1 , 1\leq j \leq n_2}$ and $B=(B_{jk})_{1\leq j \leq n_2, 1\leq k \leq n_3}$. 
By the Cauchy-Schwarz inequality, we have for each column vector $B_{\cdot k}$ of $B$ that 
\begin{align} 
 & |A B_{\cdot k}|^2 = \sum_{i=1}^{n_1} \big( \sum_{j=1}^{n_2} A_{ij} B_{jk} \big)^2 
  \leq \sum_{i=1}^{n_1} \big( \sum_{j=1}^{n_2} A_{ij}^2 \sum_{j=1}^{n_2} B_{jk}^2 \big) 
  = \big( \sum_{i=1}^{n_1} \sum_{j=1}^{n_2} A_{ij}^2 \big) \sum_{j=1}^{n_2} B_{jk}^2 
  = |A|^2 \sum_{j=1}^{n_2} B_{jk}^2 . 
  \notag 
\end{align} 
It then follows that 
\begin{align} 
 |AB|^2 = & |(A B_{\cdot 1}, A B_{\cdot 2}, ..., A B_{\cdot n_3} ) |^2 
 =  |A B_{\cdot 1}|^2 + |A B_{\cdot 2}|^2 + \cdots + |A B_{\cdot n_3}|^2 
 \leq  |A|^2 \sum_{j=1}^{n_2} \sum_{k=1}^{n_3} B_{jk}^2 = |A|^2 |B|^2 , \notag 
\end{align} 
and assertion i) is proved. 

For $A$, $B\in \mathbb{R}^{n_1\times n_1}$, we have 
\begin{align} 
|A+B|^2 = & \sum_{i=1}^{n_1} \sum_{j=1}^{n_2} (A_{ij} + B_{ij})^2  \notag \\ 
 = &  \sum_{i=1}^{n_1} \sum_{j=1}^{n_2} A_{ij}^2 + \sum_{i=1}^{n_1} \sum_{j=1}^{n_2} B_{ij}^2 
 + 2 \sum_{i=1}^{n_1} \sum_{j=1}^{n_2} A_{ij} B_{ij} \notag \\ 
 \leq & |A|^2 + |B|^2 + 2 |A|^2 |B|^2 
  = (|A|+|B|)^2 , \notag 
\end{align} 
which proves assertion ii). 
\end{proof}

\begin{lemma} 
\label{lm:|A1+An|2<=(|A1|2+|An|2)n} 
Let $w_1$, $w_2$, ..., $w_n$ be positive numbers such that $\sum_{k=1}^n w_k = 1$. 
For any $n$ matrices $A_1$, ..., $A_n\in \mathbb{R}^{n_1 \times n_2}$, it satisfies that 
$\big| \sum_{k=1}^n w_k A_k \big|^2 \leq \sum_{k=1}^n w_k |A_k|^2$; 
particularly, when $w_1=w_2=\cdots = w_n=1/n$, we have 
$\big|\sum_{k=1}^n  A_k \big|^2 \leq n \sum_{k=1}^n |A_k|^2$ . 
\end{lemma} 
\begin{proof} 
Denote $A_k = ( A_{k, ij})_{1\leq i\leq n_1,  1\leq j \leq n_2}$, $k=1$, ..., $n$. 
By the Jensen's inequality, we have 
\begin{align} 
 ( w_1 A_{1, ij} + \cdots + w_n A_{n, ij})^2 \leq  w_1 A_{1, ij}^2 + \cdots + w_n A_{n, ij}^2  . \notag 
\end{align}  
It then follows that 
\begin{align} 
  |w_1 A_1 + \cdots + w_n A_n|^2 
 = & \sum_{i=1}^{n_1} \sum_{j=1}^{n_2} (w_1 A_{1, ij} + \cdots + w_n A_{n, ij})^2   \notag \\ 
 \leq &  \sum_{i=1}^{n_1} \sum_{j=1}^{n_2} w_1 A_{1, ij}^2 + \cdots + w_n A_{n, ij}^2   \notag \\ 
  = &   w_1 \sum_{i=1}^{n_1} \sum_{j=1}^{n_2} A_{1, ij}^2 + \cdots + w_n \sum_{i=1}^{n_1} \sum_{j=1}^{n_2} A_{n, ij}^2  \notag \\ 
  = &   w_1 |A_1|^2 + \cdots + w_n |A_n|^2  .   
  \notag 
\end{align} 
\end{proof} 

Next, prove our main result, namely Theorem~\ref{thm:RegretOptimal_Dynamics__MainTextFormulation}. The result is first rewritten, in the optimal control-theoretic notation introduced in Section~\ref{s:Main_ss:Algorithm}.

\begin{theorem}[The Regret-Optimal Algorithm - Control Theoretic Form]
\label{thm:RegretOptimal_Dynamics}
Suppose that 
Assumption~\ref{assm:bddxy} holds.  Then, the optimal control problem 
\eqref{eq:penalized_regret_functional} and \eqref{Theta-simple} admits a unique solution for $t=0, 1, ..., T-1$ such that 
\allowdisplaybreaks
\begin{equation}
\hspace{-0.3cm}
\begin{aligned} 
 \boldsymbol\alpha(t)  
 = & -[(\lambda + \beta )I_{Np} + P(t+1)]^{-1}  
 \big[ (\lambda I_{Np} + P(t+1)) \Theta(t) 
 - \lambda  \Theta^\star + S(t+1)  \big] ,   \label{optimala} 
\end{aligned}
\end{equation}
and the resulting cost is 
\begin{align} 
L^\star(\boldsymbol\alpha; \mathcal{D}) 
= \Theta^{\star \top} P(0) 
\Theta^\star + 2 S^\top(0) \Theta^\star + r(0) .
\label{optimalL} 
\end{align} 
The $P(t)$, $S(t)$, and $r(t)$ in \eqref{optimala} and \eqref{optimalL} are determined by the backward equation system on $t=1,\dots,T$:  
\allowdisplaybreaks\begin{align} 
& \begin{cases}
   P(t)  =  \beta I_{Np} - \beta^2 [(\lambda+\beta) I_{Np} + P(t+1) ]^{-1} , \label{eqnP(t)} \\ 
   P(T) =  [w^\star_1 I_p ,\cdots, w^\star_N I_p]^{\top} 
 \big( \sum_{i=1}^N w^\star_i \sum_{j=1}^{|\mathcal{D}_i|}   u^i_j u^{i\top}_j \big)  
  [ w^\star_1 I_p, \cdots, w^\star_N I_p] ,  
\end{cases}   \\
& \begin{cases}
  S(t) =  \beta \big[ (\lambda+\beta)I_{Np} + P(t+1) \big]^{-1} (S(t+1) - \lambda \Theta^\star )
  \label{eqnS(t)} ,     \\ 
   S(T) =  - [ w^\star_1 I_p, \cdots, w^\star_N I_p]^{\top} \sum_{i=1}^N w^\star_i \sum_{j=1}^{|\mathcal{D}_i|}  {u^i_j} y^i_j ,\end{cases} \\
&\begin{cases} 
 r(t) =  - ( S(t+1) - \lambda \Theta^\star )^\top
 [(\lambda+\beta)I_{Np} + P(t+1)]^{-1} \cdot \\ 
 \hspace{1.1cm} ( S(t+1) - \lambda \Theta^\star )     
 + \lambda \| \Theta^\star \|_2^2 + r(t+1), 
 \label{eqnr(t)} \\ r(T) = \sum_{i=1}^N w^\star_i \sum_{j=1}^{|\mathcal{D}_i|} |y^i_j|^2 . 
 \end{cases}  
\end{align} 
\end{theorem}

\begin{proof}[Proof of Theorem~\ref{thm:RegretOptimal_Dynamics} (and therefore of Theorem~\ref{thm:RegretOptimal_Dynamics__MainTextFormulation})]
Let $V(t, \Theta)$ be the value function associated with the optimal control problem
\eqref{eq:penalized_regret_functional} and \eqref{Theta-simple}. 
By the dynamic programming principle, $V(t, \Theta)$  satisfies the equation 
\allowdisplaybreaks
\begin{align} 
& V (t, \Theta) = \min_{\boldsymbol\alpha} 
 \big[   \lambda \left| \Theta + \boldsymbol\alpha - \Theta^\star \right|^2 
 + \beta \left| \boldsymbol\alpha \right|^2 
 + V (t+1, \Theta + \boldsymbol\alpha ) \big] , \quad t=0, 1, ..., T-1 ,   \label{dpVw} \\
 & V(T, \Theta ) 
= l (\theta^{w^\star}; \mathcal{D} ) .   \label{V(T,Theta)} 
\end{align}

The loss function can be written as a linear quadratic function of 
$\Theta=[\theta_1^\top, \cdots, \theta_N^\top]^\top$ such that 
\allowdisplaybreaks\begin{align} 
   l (\theta^{w^\star}; \mathcal{D} ) 
 = & 
 \sum_{i=1}^N w^\star_i \sum_{j=1}^{|\mathcal{D}_i|}  \Big[ u^{i\top}_j \sum_{k=1}^N w^\star_k \theta_k - y^i_j  \Big]^2 \notag \\ 
 = &  \sum_{i=1}^N w^\star_i \sum_{j=1}^{|\mathcal{D}_i|} 
 \Big[ \sum_{k=1}^N w_k^{\star 2} \theta_k^\top u^i_j u^{i\top}_j \theta_k 
   -2 \sum_{k=1}^N w_k^\star y^i_j u^{i \top}_j \theta_k + |y^i_j|^2 \Big]
 \notag \\ 
= &  \Big(\sum_{k=1}^N w^\star_k \theta_k^\top \Big) \Big( \sum_{i=1}^N w^\star_i \sum_{j=1}^{|\mathcal{D}_i|}  u^i_j u^{i\top}_j  \Big) \Big(\sum_{k=1}^N w^\star_k  \theta_k\Big) 
\notag \\ 
& - 2 \sum_{i=1}^N w^\star_i \sum_{j=1}^{|\mathcal{D}_i|} y^i_j u^{i\top}_j \sum_{k=1}^N w^\star_k \theta_k  + \sum_{i=1}^N w^\star_i \sum_{j=1}^{|\mathcal{D}_i|} | y^i_j |^2 \notag \\ 
= & \Theta^\top [w^\star_1 I_p ,\cdots, w^\star_N I_p]^\top 
 \Big( \sum_{i=1}^N w^\star_i \sum_{j=1}^{|\mathcal{D}_i|}  u^i_j u^{i\top}_j \Big)  
 [ w^\star_1 I_p, \cdots, w^\star_N I_p ] \Theta \notag \\ 
 & + \sum_{i=1}^N w^\star_i \sum_{j=1}^{|\mathcal{D}_i|} | y^i_j |^2  
 - 2 \sum_{i=1}^N w^\star_i \sum_{j=1}^{|\mathcal{D}_i|} y^i_j u^{i \top}_j   
  [ w^\star_1 I_p, \cdots, w^\star_N I_p ] \Theta . \label{lwT}  
\end{align}

We claim that $V(t, \Theta)$ takes the linear quadratic form 
\allowdisplaybreaks\begin{align} 
 V (t, \Theta) 
=  \Theta^\top P(t) \Theta + 2 S^\top (t) \Theta + r(t) . \label{ansatzVw}
\end{align} 
To see this, note from \eqref{V(T,Theta)} and \eqref{lwT} that 
$V(T, \Theta)$ takes the linear quadratic form of \eqref{ansatzVw}. By solving the dynamic programming equation \eqref{dpVw} at $t=T-1$ for the optimal $\boldsymbol\alpha$ and substituting it back into the equation, we obtain $V(T-1, \Theta)$ also taking the linear quadratic form of \eqref{ansatzVw}. By backward induction we have that $V(t,\Theta)$ takes the form of \eqref{ansatzVw} for all $t=0,...,T$.

\begin{remark} Here $r(t)$ is a constant term independent of the ``state variable'' $\Theta$, which we note exists because the cost $L^\star$ defined by \eqref{eq:penalized_regret_functional} contains terms independent of $\Theta(t)$.
\end{remark}

Combining \eqref{V(T,Theta)}, 
\eqref{lwT}, and \eqref{ansatzVw}  verifies that $P(T)$ and $S(T)$ satisfy the terminal condition as given by \eqref{eqnP(t)} and \eqref{eqnS(t)}. 
We substitute \eqref{ansatzVw} into \eqref{dpVw} and reorganize the terms to get 
\begin{align}
  & \Theta^\top P(t) \Theta + 2 S^\top(t) \Theta + r(t)  \notag \\
   = &  \min_{\boldsymbol\alpha} \Big\{  \boldsymbol\alpha^\top [(\lambda + \beta ) I_{Np} + P(t+1)] \boldsymbol\alpha  
   + 2 [ \Theta^\top P(t+1)  
 + S^\top (t+1) +\lambda (\Theta^\top - \Theta^{\star \top})]  \boldsymbol\alpha 
 \notag \\ 
 & \hspace{2cm} + \Theta^\top ( \lambda I_{Np} + P(t+1) )  \Theta 
 + 2 ( S^\top(t+1) - \lambda \Theta^\star ) \Theta 
+ \lambda |\Theta^\star|^2 + r(t+1) \Big\} . 
\label{DPVansatz}
\end{align} 
By the first order condition\footnote{NB, we can argue this way, since the problem is convexified upon fixing $w^{\star}$.  In fact, this is the control-theoretic motivation for using Algorithm~\ref{alg:RegretOptimizationHeuristic__WarmStart} to decouple the optimization of $w$ and of $\theta_{\cdot}$ from one another, and not treating them as a single (non-convex if left coupled) control problem.} for the right-hand side of \eqref{DPVansatz} with respect to $\boldsymbol\alpha$, we have  
\allowdisplaybreaks\begin{align} 
 0 = &   [(\lambda+\beta) I_{Np} + P(t+1)] \boldsymbol\alpha  
  +  [ ( \lambda I_{Np} + P(t+1) ) \Theta  
 + S(t+1) - \lambda  \Theta^\star ]  . 
 \label{1stOrdera}
\end{align}
By Lemma~\ref{lm:PD}, we have 
$P(T)+(\lambda+1)I_{Np} >0$, and therefore from \eqref{1stOrdera} we obtain the optimal $\boldsymbol{\alpha}$ taking the form of \eqref{optimala}.

We further substitute $\boldsymbol{\alpha}$ of the form \eqref{optimala} into the right side of \eqref{DPVansatz} to obtain 
\allowdisplaybreaks\begin{align} 
 & \Theta^\top  P(t) \Theta  + 2 S^\top(t) \Theta  + r(t) \notag \\ 
 = & - \big[ \Theta^\top (\lambda I_{Np} + P(t+1)) 
 + S^\top(t+1) - \lambda \Theta^{\star \top}  
 \big] \cdot \big[(\lambda + \beta )I_{Np} + P(t+1)\big]^{-1} \cdot \notag \\ 
 & \big[  (\lambda I_{Np} + P(t+1)) \Theta  + S(t+1) 
  -  \lambda  \Theta^\star  \big]  
  + \Theta^\top (\lambda I_{Np} + P(t+1)) 
  \Theta \notag \\ 
 &  + 2 [ S^\top(t+1) - \lambda \Theta^{\star \top}]\Theta + \lambda |\Theta^\star|^2 + r(t+1) . 
  \label{eqn:Vansatz}
\end{align} 
Reorganizing the terms on right side of \eqref{eqn:Vansatz}, we have 
\begin{align} 
 & \Theta^\top  P(t) \Theta  + 2 S^\top(t) \Theta  + r(t) \notag \\ 
 = & - \Theta^\top   (\lambda I_{Np} + P(t+1) ) [(\lambda + \beta )I_{Np} + P(t+1)]^{-1} 
 (\lambda I_{Np} + P(t+1))   \Theta \notag \\ 
 & + \Theta^\top ( \lambda I_{Np} + P(t+1) ) \Theta + 2 [ S^\top(t+1) - \lambda \Theta^{\star \top}]\Theta  \notag \\  
 & - 2 [S^\top(t+1) - \lambda \Theta^{\star \top}] 
 [(\lambda +\beta )I_{Np} + P(t+1) ]^{-1} (\lambda I_{Np} + P(t+1)) \Theta \notag \\  
 & - (S^\top(t+1) - \lambda \Theta^{\star\top}) [(\lambda + \beta )I_{Np} + P(t+1)]^{-1} (S(t+1) - \lambda \Theta^\star)  
 \notag \\
 & + \lambda |\Theta^\star|^2 + r(t+1) . 
  \label{eqn:Vansatz-simplified} 
\end{align} 
Matching the coefficients of the quadratic terms of $\Theta$ on both sides of \eqref{eqn:Vansatz-simplified}, we obtain 
\begin{align} 
 P(t)  
  = &  -  (\lambda I_{Np} + P(t+1) ) [(\lambda + \beta )I_{Np} + P(t+1)]^{-1} 
 (\lambda I_{Np} + P(t+1))   \notag \\ 
 & +   \lambda I_{Np} + P(t+1) .  
 \label{eqn:P-unsimplified}
\end{align} 
The right side of 
\eqref{eqn:P-unsimplified} further simplifies to  
\begin{align} 
& -  (\lambda I_{Np} + P(t+1) ) [(\lambda + \beta )I_{Np} + P(t+1)]^{-1} 
 (\lambda I_{Np} + P(t+1))   
  +  \lambda I_{Np} + P(t+1)   \notag \\ 
 = &   -  [ ( \lambda + \beta ) I_{Np} + P(t+1) - \beta I_{Np} ]  [(\lambda + \beta )I_{Np} + P(t+1)]^{-1} 
 (\lambda I_{Np} + P(t+1)) \notag \\
 & +  \lambda I_{Np} + P(t+1)  \notag \\ 
 = & -   ( \lambda  I_{Np} + P(t+1) )  + \beta [(\lambda + \beta )I_{Np} + P(t+1)]^{-1} 
 (\lambda I_{Np} + P(t+1))  
  +  \lambda I_{Np} + P(t+1) \notag \\ 
  = & 
  \beta [(\lambda + \beta )I_{Np} + P(t+1)]^{-1} 
 (\lambda I_{Np} + P(t+1))  \notag \\ 
 = & 
  \beta [(\lambda + \beta )I_{Np} + P(t+1)]^{-1} 
 [ ( \lambda + \beta ) I_{Np} + P(t+1) - \beta I_{Np} ] \notag  \\ 
 = & \beta I_{Np} - \beta^2 [(\lambda + 1)I_{Np} + P(t+1)]^{-1} . 
 \label{eqn:P-simplified}
\end{align} 
Combining \eqref{eqn:P-unsimplified}
and \eqref{eqn:P-simplified} gives \eqref{eqnP(t)} for $t=T-1$. 

Matching the coefficients of the linear terms of $\Theta$ on both sides of \eqref{eqn:Vansatz-simplified} and taking the transpose, we obtain  
\begin{align} 
 S(t) = & - (\lambda I_{Np} + P(t+1) ) [(\lambda + \beta )I_{Np} + P(t+1) ]^{-1} ( S(t+1) - \lambda \Theta^\star ) \notag \\ 
& +  S(t+1) - \lambda \Theta^\star .  
 \label{eqn:S-unsimplified} 
\end{align} 
The right side of \eqref{eqn:S-unsimplified} further simplifies to 
\begin{align} 
& - (\lambda I_{Np} + P(t+1) ) [(\lambda +  \beta )I_{Np} + P(t+1) ]^{-1} ( S(t+1) - \lambda \Theta^\star ) 
 + S(t+1) - \lambda \Theta^\star 
  \notag \\ 
= &   - [ (\lambda + \beta ) I_{Np} + P(t+1) - \beta I_{Np} ] [(\lambda + \beta )I_{Np} + P(t+1) ]^{-1} ( S(t+1) - \lambda \Theta^\star ) \notag \\ 
& + S(t+1) - \lambda \Theta^\star 
\notag \\ 
 = & - ( S(t+1) - \lambda \Theta^\star ) 
  + \beta [(\lambda +\beta )I_{Np} + P(t+1) ]^{-1} ( S(t+1) - \lambda \Theta^\star ) 
  + S(t+1) - \lambda \Theta^\star 
  \notag \\ 
 = &  \beta [(\lambda + \beta )I_{Np} + P(t+1) ]^{-1} ( S(t+1) - \lambda \Theta^\star ) . \label{eqn:S-simplified}
\end{align}
Combining \eqref{eqn:S-unsimplified} and \eqref{eqn:S-simplified} gives \eqref{eqnS(t)} for $t=T-1$. 

Matching the terms independent of $\Theta$ on both sides of \eqref{eqn:Vansatz-simplified} gives 
\begin{align} 
 r(t) = &  - ( S(t+1) - \lambda \Theta^\star )^\top
 [(\lambda+ \beta )I_{Np} + P(t+1)]^{-1} \cdot ( S(t+1) - \lambda \Theta^\star ) \notag    \\  
 &  + \lambda |\Theta^\star|^2 + r(t+1) , \notag 
\end{align} 
which is \eqref{eqnr(t)} for $t=T-1$.

Since 
$P(T) + (\lambda + \beta )I_{Np} > \beta I_{Np}$, 
we have 
\begin{align} 
P(T-1) = \beta I_{Np} 
 - \beta^2 [(\lambda + \beta )I_{Np} + P(T)]^{-1} > 0 . \notag 
\end{align} 
By backward induction on $t$, we can show that 
$\boldsymbol{\alpha}$ takes the form of \eqref{optimala} for all $t=0, 1, ..., T-1$, and $P$ and $S$ satisfy the system \eqref{eqnP(t)} and \eqref{eqnS(t)}.  
\end{proof} 

\begin{corollary} 
\label{cor:M(P)invbdd}
Under Assumption~\ref{assm:bddxy}, the solution $P$ and $S$ of the system \eqref{eqnP(t)}-\eqref{eqnS(t)} satisfies that for all $t=1$, $2$, ..., $T$, 
\begin{align} 
& | [(\lambda + \beta) I_{Np} + P(t)]^{-1} |\leq  (\lambda + \beta)^{-1},  \label{M(P)invbd}  \\ 
 & |S(t)| \leq C \cdot N 
 \label{|S|bd} , 
\end{align} 
where $C> 0$ is a constant depending on $\lambda$, $\beta$, $(w_1^\star, ..., w_N^\star)$, $\Theta^\star$, $K_x$, $K_y$ and $T$. 
\end{corollary}

\begin{proof}[Proof of Corollary~\ref{cor:M(P)invbdd}] 
Let $\| \cdot \|$ be the spectral norm of a matrix such that $\|A\|\eqdef \sup_{x\neq 0} |Ax|/|x| $. 
Since $P(t)$ is positive semi-definite, we have $(\lambda + \beta) I_{Np} + P(t) \geq (\lambda + \beta) I_{Np}$ and  
\begin{align} 
0\leq [(\lambda + \beta) I_{Np} + P(t) ]^{-1} \leq (\lambda + \beta)^{-1} I_{Np}   ,  \notag 
\end{align} 
which, by the Courant-Fisher Theorem (see \cite{hornjohnson2012}), implies that $\|  [(\lambda + \beta) I_{Np} + P(t) ]^{-1} \| \leq \| (\lambda+\beta)^{-1} I_{Np} \| = (\lambda + \beta)^{-1}$. 
Since all norms of finite dimensional normed spaces are equivalent, and in particular, $|\cdot|\le \|\cdot\|$ on the space of $N_p\times N_p$-matrices, then we find that
\begin{align} 
  \big| [(\lambda + \beta) I_{Np} + P(t) ]^{-1} \big| \leq \|  [(\lambda + \beta) I_{Np} + P(t) ]^{-1} \| , \notag 
\end{align}  
and thus $\big| [(\lambda + \beta) I_{Np} + P(t) ]^{-1} \big| \leq (\lambda + \beta)^{-1}$,  
which proves \eqref{M(P)invbd}. 

Next, we prove \eqref{|S|bd}. 
By Lemmas~\ref{lm:|AB|<=|A||B|} and \ref{lm:|A1+An|2<=(|A1|2+|An|2)n}, and by assumption that
\begin{equation}
\label{eq:moredatasets_than_dataperdataset}
        \frac1{N}\sum_{i=1}^N\,|\mathcal{D}_i| 
    \le 
        N
.
\end{equation}
we obtain 
\allowdisplaybreaks
\begin{align} 
 | S(T) |^2  
 = & \Big( \sum_{i=1}^N w^{\star 2}_i \Big) 
 \Big| \sum_{i=1}^N w^\star_i   \sum_{j=1}^{|\mathcal{D}_i|}  u^i_j y^i_j  \Big|^2  
 \leq  \Big( \sum_{i=1}^N w^{\star 2}_i \Big) 
 \sum_{i=1}^N w_i^\star  \Big| \sum_{j=1}^{|\mathcal{D}_i|}   u^i_j y^i_j \Big|^2  \notag  \\ 
 \leq &  \Big( \sum_{i=1}^N w^{\star 2}_i \Big) 
 \sum_{i=1}^N w_i^\star |\mathcal{D}_i|  \sum_{j=1}^{|\mathcal{D}_i|}   | u^i_j y^i_j |^2  
 \leq  \Big( \sum_{i=1}^N w^{\star 2}_i \Big) \sum_{i=1}^N w_i^\star |\mathcal{D}_i|^2  K_x^2 K_y^2  \notag \\ 
  \leq &  \Big( \sum_{i=1}^N w^{\star 2}_i \Big) \sum_{i=1}^N w_i^\star N^2  K_x^2 K_y^2 ,  
  \label{eq:TBBounded} 
\end{align} 
and therefore the estimate \eqref{|S|bd} holds for $t=T$. 

Assume by induction that \eqref{|S|bd} holds for $t=u$.  
By \eqref{eqnS(t)}, Lemma~\ref{lm:|AB|<=|A||B|}, and the triangle inequality, we have 
\begin{align} 
 |S(u-1)| \leq &  \beta \big| \big[ (\lambda+\beta)I_{Np} + P(u) \big]^{-1} \big| \cdot \big| S(u) - \lambda \Theta^\star  \big| \notag \\ 
  \leq & \beta \big| \big[ (\lambda+\beta)I_{Np} + P(u) \big]^{-1} \big| \cdot ( |S(u)| + \lambda | \Theta^\star | ) . 
  \notag 
\end{align} 
The estimate \eqref{|S|bd} for $t=u-1$ then follows from \eqref{M(P)invbd} and the induction hypothesis. 
Thus, by induction we have shown that \eqref{M(P)invbd} holds for all $t=1$, $2$, ..., $T$. 
\end{proof}

We proceed to prove Theorem~\ref{thm:ComplexityAlgorithm} about the complexity of the regret-optimal algorithm. 
To facilitate the proof, we introduce the following Lemma~\ref{lm:matCompComplexity} that summarizes complexities of matrix operations. 
\begin{lemma} 
\label{lm:matCompComplexity} 
Suppose all elementary arithmetic operations (e.g.\ addition, subtraction, multiplication) of two real numbers are of constant (i.e.\ $\mathcal{O}(1)$) complexity. Then elementary matrix computations have the following complexities: 

\begin{enumerate}  
\item[(i)] Adding two $n_1 \times n_2$ matrices has a complexity of $O(n_1 n_2)$. 

\item[(ii)] Multiplying an $n_1 \times n_2$ matrices by a scalar has a complexity of $O(n_1 n_2)$. 

\item[(iii)] Multiplying an $n_1 \times n_2$ matrix with an $n_2 \times n_3$ matrix has a complexity of $O(n_1 n_2 n_3)$. 
In particular, multiplying two $n\times n$ matrices has a complexity of $\mathcal{O}(n^{2.373})$.  
\end{enumerate} 
\end{lemma} 

\begin{proof} 
Adding two $n_1\times n_2$ matrices involves adding each entry of one matrix with the corresponding entry of the other matrix, and each entry addition has a complexity of $\mathcal{O}(1)$. Hence the total complexity is $\mathcal{O}(n_1 n_2)$. 

Multiplying an $n_1 \times n_2$ matrix by a scalar involves multiplying each entry by the scalar, and the scalar multiplication of each entry has a complexity of $\mathcal{O}(1)$. Hence the total complexity is $O(n_1 n_2)$. 

Multiplication of an $n_1 \times n_2$ matrix and an $n_2 \times n_3$ matrix requires multiplying the $n_1$ row vectors of the first matrix by the $n_3$ column vectors of the second matrix, which needs a total of $n_1 n_3$ multiplications. Multiplication of two $n_2$-dimensional vectors has a complexity of $O(n_2)$. Therefore the total complexity is $O(n_1n_2n_3)$.  
When multiplying two matrices of $n\times n$-dimension with an optimized CW-like algorithm the complexity is $\mathcal{O}(n^{2.373})$. 
\end{proof}

\begin{proof}[{Proof of Theorem~\ref{thm:ComplexityAlgorithm}}] 
\hfill\\
\noindent\textbf{Complexity of Computing $P(T)$ is $\mathcal{O}\big(N^2p^3\big)$}\hfill\\
Since $u^i_j \in \mathbb{R}^{p\times 1}$, computing each 
$u^i_j u^{i\top}_j$ has $\mathcal{O}(p^2)$ complexity, by Lemma \ref{lm:matCompComplexity}-(iii). 
Multiplying the $p\times p$ matrix $u^i_j u^{i\top}_j$ by the scalar $w^\star_i$ has a complexity of $\mathcal{O}(p^2)$, by Lemma~\ref{lm:matCompComplexity}-(iii).   
Then computing the $p\times p$ matrix $\sum_{i=1}^N w^\star_i \sum_{j=1}^{S_i}\, u^i_j u^{i\top}_j$ by adding the $\bar{N}$ matrices of dimension $p\times p$ has complexity $\mathcal{O}(\bar{N}\,p^2)$, by Lemma \ref{lm:matCompComplexity}-(i).  
Since $[w_1^\star I_p ,\cdots, w_N^\star I_p]^{\top}$ is $Np\times p$-dimensional, then computing the $np\times p$-dimensional product $[w_1^\star I_p ,\cdots, w_N^\star I_p]^{\top}\sum_{i=1}^N w^\star_i \sum_{j=1}^{S_i}\, u^i_j u^{i \top}_j$ has $\mathcal{O}(Np^3)$ complexity, by Lemma~\ref{lm:matCompComplexity}-(iii). 
Likewise, the product 
\[ 
    \Big([w_1^\star I_p ,\cdots, w_N^\star I_p]^{\top}\sum_{i=1}^N  w^\star_i \sum_{j=1}^{S_i}\, u^i_j u^{i\top}_j \Big)[w_1^\star I_p ,\cdots, w_N^\star I_p]
\]
has a complexity of $\mathcal{O}(N^2p^3)$.  Therefore, computing $P(T)$ itself is of $\mathcal{O}\big(\bar{N}p^2+N^2p^3\big)$ complexity.  
Assumption~\eqref{eq:moredatasets_than_dataperdataset} implies that $\bar{N}\leq N^2$; hence, the complexity of computing $P(T)$ is $\mathcal{O}\big(N^2p^3\big)$.

\noindent
\textbf{Complexity of Computing $S(T)$ is $\mathcal{O}\big(Np\max\{N,p\}\big)$}
\hfill\\ 
Since $u^i_j\in\mathbb{R}^{p\times 1}$ and $y^i_j\in \mathbb{R}$, 
computing the product $u^i_j y_j^i$ has complexity $\mathcal{O}(p)$, by Lemma \ref{lm:matCompComplexity}-(iii). 
And multiplying the $p\times 1$ dimensional vector $u^i_j y^i_j$ by the scalar $w^\star_i$ has complexity $\mathcal{O}(p)$, by Lemma \ref{lm:matCompComplexity}-(ii). Then the complexity of computing $\sum_{i=1}^N w^\star_i \sum_{j=1}^{S_i}\,u^i_j y_j^i$ by adding $\bar{N}$ vectors of dimension $p$ is $\mathcal{O}\big(\bar{N}p\big)$, by Lemma~\ref{lm:matCompComplexity}-(i). 
Since $- [ w_1^\star I_p, \cdots, w_N^\star I_p]^{\top}$ is an $Np\times p$-dimensional matrix and $\sum_{i=1}^N w^\star_i \sum_{j=1}^{S_i}\, u^i_j y_j^i$ is a $p\times 1$-dimensional vector, then computing their $Np\times 1$ dimensional product has $\mathcal{O}(Np^2)$ complexity, according to Lemma~\ref{lm:matCompComplexity}-(iii).  
Therefore, the computing $S(T)$ has a complexity of $\mathcal{O}(\bar{N}p + N P^2)$. 
Since~\eqref{eq:moredatasets_than_dataperdataset} implies that $\bar{N}\le N^2$ then, computing $S(T)$ had a computational cost of $\mathcal{O}\big(Np\max\{N,p\}\big)$.

\noindent
\textbf{Complexity of Computing $P(t)$ and $S(t)$ for $t=T-1,\dots, 1$ is $\mathcal{O}\big(T(Np)^{2.373}\big)$}
\hfill\\
The complexity of computing $(\lambda +\beta )I_{Np} + P(t+1)$ for each $t$ is $\mathcal{O}(Np)$, since we only need to add scalars on the matrix's \textit{diagonal}.  By \cite{LeGall_Complexity_CWLikeMatrixInversion_2014}, computing the matrix inverse $[(\lambda+\beta)I_{Np}+P(t+1)]^{-1}$ with a CW-like algorithm has a complexity of $\mathcal{O}\big((Np)^{2.373}\big)$.  
Since the addition $\beta I_{Np} - \beta^2 [(\lambda+ \beta )I_{Np} + P(t+1) ]^{-1}$ only requires us to multiply all entries of $ [(\lambda+\beta)I_{Np} + P(t+1) ]^{-1}$ by $-\beta^2$ and add along the diagonal, then it has a complexity of $\mathcal{O}\big((Np)^2 + Np\big)$, by Lemma~\ref{lm:matCompComplexity}-(i) and (ii). 
Performing the matrix addition $S(t+1)-\lambda \Theta^\star$ has complexity $\mathcal{O}(Np)$ by Lemma~\ref{lm:matCompComplexity}-(i).  
Computing the $Np\times 1$-dimensional product 
$[(\lambda+1)I_{Np} + P(t+1)]^{-1} (S(t+1) - \lambda \Theta^\star)$ 
by multiplying an $Np\times Np$ matrix by an $Np\times 1$ matrix has complexity $\mathcal{O}(N^2 p^2)$, according to Lemma~\ref{lm:matCompComplexity}-(iii). 
Therefore, computing each $P(t)$ and $S(t)$ for $t=T-1,\dots,1$, has complexity  $\mathcal{O}\big((Np)^{2.373}\big)$.  Since there are $T-1$ such matrices to compute, then the total cost of computing every $P(T-1),\dots,P(1)$ and $S(T-1)$, ..., $S(1)$ is $\mathcal{O}\big(T(Np)^{2.373}\big)$.

\noindent
\textbf{The Cost of Computing Each $\Theta(t)$ is $\mathcal{O}(T\,N^2p^2)$}
\hfill\\
To compute $\Theta(t)$ for all $t=1$, $\dots$, $T$, starting from $\Theta(0)=\Theta^\star$, 
we need to compute $\mathbf{a}(t)$ given by \eqref{optimala} for $t=0$, $1$, $\dots$, $T-1$.  
Hence we need to quantify the complexity of computing $\mathbf{a}(t)$. 
By \eqref{eqnP(t)}, we can obtain $-[(\lambda + \beta)I_{Np} + P(t+1)]^{-1}$ by computing 
$\beta^{-2}(P(t)- \beta I_{Np})$, which involves adding a scalar to the diagonal of an $Np\times Np$ matrix, and then multiplying the matrix by a scalar. 
The cost of computing the product
$(\lambda I_{Np} + P(t+1)) \Theta(t)$ is $\mathcal{O}(N^2 p^2)$ and the cost of computing all sums in $(\lambda I_{Np} + P(t+1)) \Theta(t) 
 - \lambda \Theta^\star + S(t+1)  $ are of lower order complexity.  Finally, the cost of computing the matrix-product between $-[(\lambda + \beta)I_{Np} + P(t+1)]^{-1} $ and $
 \big[ (\lambda I_{Np} + P(t+1)) \Theta(t) 
 - \lambda \Theta^\star + S(t+1)  \big]$ is $\mathcal{O}(N^2p^2)$, by Lemma~\ref{lm:matCompComplexity}-(iii). Therefore, the cost of computing each $\mathbf{a}(t)$ is $\mathcal{O}(N^2p^2)$, and therefore the cost of computing $\Theta(t)$ for all $t=1$, ..., $t=T$ from $\Theta(0)=\Theta^\star$ is $\mathcal{O}(T N^2 p^2)$.  

\noindent
\textbf{Tallying up Complexities: The Complexity of Computing $(\Theta(t))_{t=0}^T$ is 
$\mathcal{O}( N^2 p^3 + T\, (N p)^{2.373})$.}
\hfill\\
Looking over all computations involved, we deduce that the cost of computing the sequence $(\Theta(t))_{t=1}^T$ is (in order) dominated by the computational complexity of computing the sequence $(P(t))_{t=1}^{T-1}$ and $P(T)$; whence it is 
$\mathcal{O}( N^2 p^3 + T\, (N p)^{2.373})$. 
\end{proof}

The following Lemma~\ref{lem:stability_wrtData}, 
Proposition~\ref{prop:stabilityControl}, and Corollary~\ref{cor:stabilityLast} are devoted to proving Theorem~\ref{thrm:adv_robust}, the adversarial robustness of Algorithm~\ref{alg:RegretOptimization}.

We introduce the linear map $M(\cdot): \mathbb{R}^{Np\times Np} \to \mathbb{R}^{Np\times Np}$ such that $M(Q) = (\lambda + \beta) I_{Np} + Q$ for any $Q\in \mathbb{R}^{Np\times Np}$. 
\begin{lemma} 
\label{lem:stability_wrtData}
Given any two data sets $\mathcal{D}=\cup_{i=1}^N \mathcal{D}_i=\cup_{i=1}^N \cup_{j=1}^{|\mathcal{D}_i|} 
 \{(x^i_j, y^i_j) \}$ and 
$\widetilde{\mathcal{D}}=\cup_{i=1}^N \widetilde{\mathcal{D}}_i = \cup_{i=1}^N \cup_{j=1}^{|\mathcal{D}_i|}  \{(\tilde{x}^i_j, \tilde{y}^i_j) \}$ satisfying Assumption \ref{assm:bddxy}, let $(P, S)$ and $(\widetilde{P}, \widetilde{S})$ be the solutions of \eqref{eqnP(t)}-\eqref{eqnS(t)} corresponding to $\mathcal{D}$ and $\widetilde{\mathcal{D}}$, respectively. Then, for all $t=1, ..., T$, we have 
\allowdisplaybreaks\begin{align}  
    \big| P(t) - \widetilde{P}(t) \big| 
 \le &
     C \cdot \Big[ \sum_{i=1}^N w^\star_i |\mathcal{D}_i| \sum_{j=1}^{|\mathcal{D}_i|} | u^i_j - \tilde{u}^i_j |^2 \Big]^{1/2} , 
 \label{eqn:estP-tildeP} 	 \\ 
     \big| [ M(P(t)) ]^{-1} - [ M(\widetilde{P}(t)) ]^{-1} \big| 
 \le&  
     C \cdot \Big[ \sum_{i=1}^N w^\star_i |\mathcal{D}_i| \sum_{j=1}^{|\mathcal{D}_i|} | u^i_j - \tilde{u}^i_j |^2 \Big]^{1/2} ,  
    \label{eqn:estM-tildeM} \\ 
        |S(t) - \widetilde{S}(t)| 
    \le &
      C \cdot \Big[ \sum_{i=1}^N w^\star_i |\mathcal{D}_i|  \sum_{j=1}^{|\mathcal{D}_i|} | u^i_j - \tilde{u}^i_j |^2 
     + |y^i_j - \tilde{y}^i_j |^2 \Big]^{1/2} , 
\label{eqn:estS-tildeS} 
\end{align}
where $C> 0$ is a constant depending on $\lambda$, $\beta$, $N$, $\bar{N}$, $(w_1^\star, ..., w_N^\star)$, $\Theta^\star$, $K_x$, $K_y$ and $T$. 
\end{lemma}

\begin{proof}
We let $C$ be a constant depending on $\lambda$, $\beta$, $N$, $\bar{N}$, $\Theta^\star$, $(w_1^\star, ..., w_N^\star)$, $K_x$, $K_y$, and $T$,  
and allow $C$ to vary from place to place throughout the proof. 
We employ backward induction starting from $t=T$ to prove \eqref{eqn:estM-tildeM}, \eqref{eqn:estP-tildeP}, and 
\eqref{eqn:estS-tildeS}. 

From \eqref{eqnP(t)}, the terminal condition $| P(T) - \widetilde{P}(T) |^2$ can be written as 
\allowdisplaybreaks\begin{align} 
 \big| P(T) - \widetilde{P}(T) \big|^2  
 = &   \Big| [w_1^\star I_p ,\cdots, w_N^\star I_p ]^\top
 \Big[ \sum_{i=1}^N w_i^\star  \sum_{j=1}^{|\mathcal{D}_i|}  ( u^i_j u^{i\top}_j  - \tilde{u}^i_j \tilde{u}^{i\top}_j  ) \Big]  
 [ w_1^\star I_p, \cdots, 
 w_N^\star I_p ]  \Big|^2 \notag \\ 
 = & \sum_{k, l=1}^N 
  w_k^{\star 2} 
  w_l^{\star 2} 
 \Big| \sum_{i=1}^N w^\star_i \sum_{j=1}^{|\mathcal{D}_i|}  ( u^i_j u^{i\top}_j  - \tilde{u}^i_j \tilde{u}^{i\top}_j  ) \Big|^2 . \notag 
 \end{align} 
By Lemmas~\ref{lm:|AB|<=|A||B|} and \ref{lm:|A1+An|2<=(|A1|2+|An|2)n}, we have 
 \begin{align} 
 \big| P(T) - \widetilde{P}(T)  \big|^2  
 \leq &   \sum_{k, l=1}^N w_k^{\star 2} w_l^{\star 2} 
  \sum_{i=1}^N w_i^\star   
\Big| \sum_{j=1}^{|\mathcal{D}_i|} 
    (  u^i_j u^{i\top}_j  - \tilde{u}^i_j \tilde{u}^{i\top}_j  ) \Big|^2 
\notag \\ 
 \leq &   \sum_{k, l=1}^N w_k^{\star 2} w_l^{\star 2} 
  \sum_{i=1}^N w_i^\star |\mathcal{D}_i|  \sum_{j=1}^{|\mathcal{D}_i|} 
    \big| u^i_j u^{i\top}_j  - 
\tilde{u}^i_j \tilde{u}^{i\top}_j  \big|^2 
\notag \\ 
 = &   \sum_{k, l=1}^N w_k^{\star 2} w_l^{\star 2} 
  \sum_{i=1}^N 
  w_i^\star |\mathcal{D}_i| \sum_{j=1}^{|\mathcal{D}_i|}  
\big| ( u^i_j - \tilde{u}^i_j  )u^{i\top}_j  
 + \tilde{u}^i_j (u^{i\top}_j - \tilde{u}^{i\top}_j)   \big|^2  \notag \\ 
\leq &   \sum_{k, l=1}^N w_k^{\star 2} w_l^{\star 2} 
    \sum_{i=1}^N w_i^\star |\mathcal{D}_i| \sum_{j=1}^{|\mathcal{D}_i|}  
 2 \big( \big|  u^i_j - \tilde{u}^i_j   \big|^2 \cdot \big| u^i_j \big|^2 
 + \big| \tilde{u}^i_j \big|^2 \cdot \big| u^i_j - \tilde{u}^i_j  \big|^2 \big) 
 \notag \\ 
 \leq &   \sum_{k, l=1}^N w_k^{\star 2} w_l^{\star 2} 
  \sum_{i=1}^N w_i^\star |\mathcal{D}_i|  \sum_{j=1}^{|\mathcal{D}_i|} 4 K_x^2 
\big| u^i_j - \tilde{u}^i_j   \big|^2 
. \label{estP(T)-tildeP(T)}   
\end{align}
Therefore \eqref{eqn:estP-tildeP} holds for $t=T$.

Assume by induction that 
\eqref{eqn:estP-tildeP} holds for $t=u$. We show that \eqref{eqn:estP-tildeP} also holds for $t=u-1$. 
From \eqref{eqnP(t)}, we have  
\allowdisplaybreaks\begin{align} 
 P(u-1) - \widetilde{P}(u-1)  
= & [ M( \widetilde{P}(u) ) ]^{-1} 
 - [ M(P(u)) ]^{-1}\notag  \\ 
 = & [ M( P(u) ) ]^{-1} ( P(u) - \widetilde{P}(u) )  
[ M( \widetilde{P}(u) ) ]^{-1} ,  \notag 
\end{align} 
and by Lemma~\ref{lm:|AB|<=|A||B|} we further have 
\allowdisplaybreaks  
\begin{align} 
 \big| P(u-1) - \widetilde{P}(u-1)  \big|  
 \leq & \big| [ M( P(u) ) ]^{-1} 
 \big| \cdot \big| P(u) - \widetilde{P}(u) \big| \cdot 
\big| [ M( \widetilde{P}(u) ) ]^{-1} \big| . \label{|tildeP-P|recursion} 
\end{align} 
By Corollary~\ref{cor:M(P)invbdd} and the induction hypothesis for $t=u$,  
\eqref{|tildeP-P|recursion} implies that \eqref{eqn:estP-tildeP} holds for $t=u-1$. 
We have shown by induction that \eqref{eqn:estP-tildeP} holds for all $t=1$, ..., $T$.

Since 
\begin{align} 
  [M(P(t))]^{-1} - [M(\widetilde{P}(t))]^{-1} 
  =  [ M( P(t) ) ]^{-1} (  \widetilde{P}(t) - P(t) )  
[ M( \widetilde{P}(t) ) ]^{-1} ,   
\notag 
\end{align} 
by Lemma~\ref{lm:|AB|<=|A||B|} we have   
\begin{align} 
 \big| [M(P(t))]^{-1} - [M(\widetilde{P}(t))]^{-1} \big| 
  \leq  \big| [ M( P(t) ) ]^{-1} \big| \cdot  
  \big|  \widetilde{P}(t) - P(t) \big| \cdot  
\big| [ M( \widetilde{P}(t) ) ]^{-1} \big|  . 
\label{eqn:M-tildeM<=P-tildeP}
\end{align} 
The estimate \eqref{eqn:estM-tildeM} then follows from \eqref{eqn:M-tildeM<=P-tildeP}, \eqref{eqn:estP-tildeP}, and Corollary~\ref{cor:M(P)invbdd}. 


Now we establish the estimate \eqref{eqn:estS-tildeS}. 
By \eqref{eqnS(t)}, the terminal condition $S(T)-\widetilde{S}(T)$ satisfies 
\allowdisplaybreaks 
\begin{align} 
 | S(T) - \widetilde{S}(T) |^2 
 = & 
 \Big| [w_1^\star, \cdots, w_N^\star]^\top \sum_{i=1}^N w_i^\star \sum_{j=1}^{|\mathcal{D}_i|} 
  ( u^i_j y^i_j - \tilde{u}^i_j \tilde{y}^i_j ) \Big|^2 \notag \\ 
  =& \Big( \sum_{i=1}^N w^{\star 2}_i \Big) 
 \Big| \sum_{i=1}^N w^\star_i \sum_{j=1}^{|\mathcal{D}_i|} ( u^i_j y^i_j - \tilde{u}^i_j \tilde{y}^i_j ) \Big|^2 .  \notag
\end{align} 
By Lemmas~\ref{lm:|AB|<=|A||B|} and \ref{lm:|A1+An|2<=(|A1|2+|An|2)n}, we obtain 
\allowdisplaybreaks
\begin{align} 
 | S(T) - \widetilde{S}(T) |^2  
 \leq & \Big( \sum_{i=1}^N w^{\star 2}_i \Big) 
 \sum_{i=1}^N w^\star_i  \Big| \sum_{j=1}^{|\mathcal{D}_i|} ( u^i_j y^i_j - \tilde{u}^i_j \tilde{y}^i_j ) \Big|^2  \notag  \\ 
  \leq & \Big( \sum_{i=1}^N w^{\star 2}_i \Big) 
 \sum_{i=1}^N w^\star_i |\mathcal{D}_i| \sum_{j=1}^{|\mathcal{D}_i|}  \big|  u^i_j y^i_j - \tilde{u}^i_j \tilde{y}^i_j \big|^2  \notag   \\ 
 = & \Big( \sum_{i=1}^N w^{\star 2}_i \Big) 
 \sum_{i=1}^N w^\star_i |\mathcal{D}_i| \sum_{j=1}^{|\mathcal{D}_i|}  \big|  (u^i_j - \tilde{u}^i_j) y^i_j + \tilde{u}^i_j (y^i_j - \tilde{y}^i_j) \big|^2  \notag \\ 
  \leq & \Big( \sum_{i=1}^N w^{\star 2}_i \Big) 
 \sum_{i=1}^N w^\star_i |\mathcal{D}_i| \cdot 2  \sum_{j=1}^{|\mathcal{D}_i|} \big[ \big|  (u^i_j - \tilde{u}^i_j) y^i_j \big|^2 
 + \big|  \tilde{u}^i_j (y^i_j - \tilde{y}^i_j) \big|^2 \big]   \notag \\    
 \leq & \Big( \sum_{i=1}^N w^{\star 2}_i \Big) 
 \sum_{i=1}^N w^\star_i |\mathcal{D}_i| \cdot 2  \sum_{j=1}^{|\mathcal{D}_i|} \big[ |  u^i_j - \tilde{u}^i_j |^2 \cdot | y^i_j |^2 
 + |  \tilde{u}^i_j |^2 \cdot | y^i_j - \tilde{y}^i_j  |^2 \big]   \notag\\ \leq & \Big( \sum_{i=1}^N w^{\star 2}_i \Big) 
 \sum_{i=1}^N w^\star_i |\mathcal{D}_i| \cdot 2  \sum_{j=1}^{|\mathcal{D}_i|} \big[ K_y^2 \cdot |  u^i_j - \tilde{u}^i_j |^2  
 + K_x^2 \cdot | y^i_j - \tilde{y}^i_j  |^2 \big]   \notag 
 \\ 
 \label{estST-tildeST} 
 \leq & \Big( \sum_{i=1}^N w^{\star 2}_i \Big) 2 (K_x^2 + K_y^2)
 \sum_{i=1}^N w^\star_i |\mathcal{D}_i|   \sum_{j=1}^{|\mathcal{D}_i|} \big[ |  u^i_j - \tilde{u}^i_j |^2  
 +  | y^i_j - \tilde{y}^i_j  |^2 \big] , 
\end{align} 
and thus \eqref{eqn:estS-tildeS} holds for $t=T$. 
Assume by induction that \eqref{eqn:estS-tildeS} holds for $t=u$, we show that \eqref{eqn:estS-tildeS} also holds for $t=u-1$. 
From \eqref{eqnS(t)}, the difference $S(t)-\widetilde{S}(t)$ can be written as 
\allowdisplaybreaks\begin{align} 
 S(u-1) - \widetilde{S}(u-1) = & [ M(\widetilde{P}(u)) ]^{-1}[S(u) - \widetilde{S}(u) ]  \notag \\ 
 & + \big[ ( M(P(u)) )^{-1} - 
  ( M(\widetilde{P}(u)))^{-1} \big]  
 ( S(u) - \lambda \Theta^\star ) , 
\notag 
\end{align} 
and satisfies by Lemma~\ref{lm:|AB|<=|A||B|} and Corollary~\ref{cor:M(P)invbdd} that 
\begin{align} 
| S(u-1) - \widetilde{S}(u-1) | \leq &
 \big| [ M(\widetilde{P}(u)) ]^{-1} \big| \cdot |S(u) - \widetilde{S}(u) |  \notag \\ 
 & + \big| ( M(P(u)) )^{-1} - 
  ( M(\widetilde{P}(u)))^{-1} \big| \cdot   
 ( |S(u)| + \lambda | \Theta^\star | ) . 
 \label{|tildeS-S|recursion} 
\end{align} 
By Corollary~\ref{cor:M(P)invbdd}, \eqref{eqn:estM-tildeM}, and the induction hypothesis, 
\eqref{|tildeS-S|recursion} implies that 
\eqref{eqn:estS-tildeS} holds for $t=u-1$.  
We have shown by induction that \eqref{eqn:estS-tildeS} holds for all $t=1$, ..., $T$. 
\end{proof} 

\begin{proposition} 
\label{prop:stabilityControl}
For two arbitrary data sets $\mathcal{D}=\cup_{i=1}^N \mathcal{D}_i=\cup_{i=1}^N \cup_{j=1}^{|\mathcal{D}_i|} 
 \{(x^i_j, y^i_j) \}$ and 
$\widetilde{\mathcal{D}}=\cup_{i=1}^N \widetilde{\mathcal{D}}_i = \cup_{i=1}^N \cup_{j=1}^{|\mathcal{D}_i|}  \{(\tilde{x}^i_j, \tilde{y}^i_j) \}$ satisfying Assumption \ref{assm:bddxy}, 
let $\boldsymbol\alpha =(\alpha_1^{\top}, \cdots, \alpha_N^{\top})^{\top}$ and 
$\widetilde{\boldsymbol\alpha}=(\widetilde{\alpha}_1^{\top}, \cdots, \widetilde{\alpha}_N^{\top})^{\top}$ be the optimal controls for \eqref{eq:penalized_regret_functional} and \eqref{Theta-simple} corresponding to $\mathcal{D}$ and $\widetilde{\mathcal{D}}$, respectively. Then, we have
\allowdisplaybreaks\begin{align} 
 \sum_{t=0}^{T-1} |\boldsymbol\alpha(t) - \widetilde{\boldsymbol\alpha}(t) | 
 \le C  \Big[ \sum_{i=1}^N w^{\star 2}_i |\mathcal{D}_i| \sum_{j=1}^{|\mathcal{D}_i|} | u^i_j - \tilde{u}^i_j |^2 
 + |y^i_j - \tilde{y}^i_j|^2  \Big]^{1/2} , 
\label{eqn:stabilityControl}
\end{align} 
where $C > 0$ is a constant depending on $\lambda$, $\beta$, $N$, $\bar{N}$, $\Theta^\star$, 
$(w^\star_1, \cdots, w^\star_N)$, $K_x$, $K_y$, and $T$. 
\end{proposition} 

\begin{proof}[{Proof of Proposition~\ref{prop:stabilityControl}}]
Let $C$ be a constant depending on $\lambda$, $\beta$, $N$, $\bar{N}$, $\Theta^\star$,  
$(w^\star_1, \cdots, w^\star_N)$, $K_x$, $K_y$, and $T$,  
and $C$ is allowed to vary from place to place throughout the proof. 
By \eqref{optimala}, the optimal control $\boldsymbol\alpha(t)$ along the optimal trajectory can be written as 
\allowdisplaybreaks\begin{align} 
  \boldsymbol\alpha(t) 
 = & [ (M(P(t+1)))^{-1} - I_{Np} ]  \Big[ \Theta^\star 
 + \sum_{u=0}^{t-1} \boldsymbol\alpha(u) \Big] \notag \\ 
 & +  [ M(P(t+1)) ]^{-1} ( \lambda  \Theta^\star - S(t+1) ) ,   \notag 
\end{align}
and the difference $\boldsymbol{\alpha}(t)-\widetilde{\boldsymbol\alpha}(t)$ can be written as    
\allowdisplaybreaks\begin{align}  
  \boldsymbol\alpha(t) - \widetilde{\boldsymbol\alpha}(t) 
 = &  \big[(M(P(t+1)))^{-1} - (M(\widetilde{P}(t+1)) )^{-1} \big] 
 \Big[ \Theta^\star + \sum_{u=0}^{t-1}  \boldsymbol\alpha(u)   + \lambda \Theta^\star - S(t+1) \Big]  \notag \\ 
 & + \big[ ( M(\widetilde{P}(t+1)) )^{-1} - I_{Np} \big] 
 \sum_{u=0}^{t-1} (\boldsymbol\alpha(u) - \widetilde{\boldsymbol\alpha}(u)) \notag \\ 
& + (M(\widetilde{P}(t+1)))^{-1} (S(t+1) - \widetilde{S}(t+1)) . 
\label{a-tildea}
\end{align}

For $t=0$, the difference $\boldsymbol\alpha(0)-\widetilde{\boldsymbol\alpha}(0)$ can be written as  
\allowdisplaybreaks\begin{align} 
  \boldsymbol\alpha(0) - \widetilde{\boldsymbol\alpha}(0)  
 = & \big[ (M(P(1)))^{-1} - (M(\widetilde{P}(1)))^{-1}  \big] 
 \big[ (\lambda + 1)   \Theta^\star - S(1) 
 \big] \notag \\ 
 & - [ M(\widetilde{P}(1)) ]^{-1} (S(1)-\widetilde{S}(1)) , \notag  
\end{align} 
and satisfies by Lemma~\ref{lm:|AB|<=|A||B|} that 
\begin{align} 
\big| \boldsymbol\alpha(0) - \widetilde{\boldsymbol\alpha}(0) \big|   
 \leq &   \big| (M(P(1)))^{-1} - (M(\widetilde{P}(1)))^{-1}  \big| \cdot  
 \big[ (\lambda + 1)   | \Theta^\star| +  | S(1) | 
 \big] \notag \\ 
 & +  \big| [ M(\widetilde{P}(1)) ]^{-1} \big| \cdot  \big| S(1)-\widetilde{S}(1) \big| .  \notag  
\end{align} 
It then follows from Corollary~\ref{cor:M(P)invbdd}, \eqref{eqn:estM-tildeM}, and \eqref{eqn:estS-tildeS} that 
\allowdisplaybreaks\begin{align} 
  \big| \boldsymbol\alpha(0) - \widetilde{\boldsymbol\alpha}(0) \big| \leq C 
 \cdot \Big[ \sum_{i=1}^N w^\star_i |\mathcal{D}_i| \sum_{j=1}^{|\mathcal{D}_i|} | u^i_j 
 - \tilde{u}^i_j |^2 
 + |y^i_j - \tilde{y}^i_j|^2  \Big]^{1/2} . \notag   
\end{align}
Assume by induction that  
\allowdisplaybreaks\begin{align} 
 \sum_{u=0}^{t-1} \big| \boldsymbol\alpha(u) - \widetilde{\boldsymbol\alpha}(u) \big| \leq C 
 \cdot \Big[ \sum_{i=1}^N w^\star_i |\mathcal{D}_i| \sum_{j=1}^{|\mathcal{D}_i|} | u^i_j - \tilde{u}^i_j |^2 
 + |y^i_j - \tilde{y}^i_j |^2  \Big]^{1/2} .  \notag   
\end{align} 
By Lemma~\ref{lm:|AB|<=|A||B|}, \eqref{a-tildea} implies that   
\allowdisplaybreaks\begin{align}  
 & | \boldsymbol\alpha(t) - \widetilde{\boldsymbol\alpha}(t) | \notag \\ 
 \leq &  \big| (M(P(t+1)))^{-1} - (M(\widetilde{P}(t+1)) )^{-1} \big| \cdot 
 \Big[ |\Theta^\star| + \sum_{u=0}^{t-1}  | \boldsymbol\alpha(u) |   
 + \lambda | \Theta^\star | +| S(t+1) | \Big]  \notag \\ 
 & + \big| ( M(\widetilde{P}(t+1)) )^{-1} - I_{Np}  \big| 
 \sum_{u=0}^{t-1} | \boldsymbol\alpha(u) - \widetilde{\boldsymbol\alpha}(u) | 
+ | (M(\widetilde{P}(t+1)))^{-1} | \cdot | S(t+1) - \widetilde{S}(t+1) | .  \notag   
\end{align} 
By \eqref{eqn:estM-tildeM},  \eqref{eqn:estS-tildeS}, and the induction hypothesis,
the above inequality implies that that 
\allowdisplaybreaks\begin{align}
 |\boldsymbol\alpha(t) - \widetilde{\boldsymbol\alpha}(t)| 
 \leq C  \Big[ \sum_{i=1}^N w^\star_i |\mathcal{D}_i| \sum_{j=1}^{|\mathcal{D}_i|} | u^i_j - \tilde{u}^i_j |^2 
 + |y^i_j - \tilde{y}^i_j|^2  \Big]^{1/2} , \notag  
\end{align} 
and therefore 
\allowdisplaybreaks\begin{align} 
 \sum_{u=0}^{t} \big| \boldsymbol\alpha(u) - \widetilde{\boldsymbol\alpha}(u) \big| \leq C 
 \cdot \Big[ \sum_{i=1}^N w^\star_i |\mathcal{D}_i| \sum_{j=1}^{|\mathcal{D}_i|} | u^i_j - \tilde{u}^i_j |^2 
 + |y^i_j - \tilde{y}^i_j |^2  \Big]^{1/2} .  \notag   
\end{align} 
The estimate \eqref{eqn:stabilityControl} is then proved by induction.  
\end{proof}

\begin{corollary} 
\label{cor:stabilityLast}
Under the hypothesis of Proposition \ref{prop:stabilityControl}, the costs \eqref{eq:penalized_regret_functional} under the optimal controls $\mathbf{a}$ 
and $\widetilde{\mathbf{a}}$ satisfy the estimate 
\allowdisplaybreaks\begin{align} 
 \big| L^\star ( \boldsymbol\alpha ; \mathcal{D} ) 
 -  L^\star (  \widetilde{\boldsymbol\alpha} ;  \widetilde{\mathcal{D}} )  \big|  
 \leq  C \cdot \Big[ \sum_{i=1}^N w^\star_i |\mathcal{D}_i| \sum_{j=1}^{|\mathcal{D}_i|} | u^i_j - \tilde{x}^i_j |^2  + |y^i_j - \tilde{y}^i_j |^2 \Big]^{1/2} . 
 \label{eqn:stabilityLast} 
\end{align} 
where $C>0$ is a constant depending on $\lambda$, $\beta$, $N$, $\bar{N}$, $\Theta^\star$, 
$(w^\star_1, \cdots, w^\star_N)$, $K_x$, $K_y$, and $T$. 
\end{corollary}

\begin{proof}[{Proof of Corollary~\ref{cor:stabilityLast}}] 
Throughout the proof $C$ is a constant depending on $\lambda$, $\beta$, $\bar{N}$, $\Theta^\star$, 
$(w_1^\star, \cdots, w_N^\star)$, $K_x$, $K_y$ and $T$, and may vary from place to place.  
By \eqref{ansatzVw} and \eqref{Theta-simple}, the difference of the costs can be written as 
\allowdisplaybreaks\begin{align} 
  L^\star ( \boldsymbol\alpha ; \mathcal{D} ) 
 -  L^\star (  \widetilde{\boldsymbol\alpha} ;  \widetilde{\mathcal{D}} )  
= & \Theta^\top (0) ( P(0) - \widetilde{P}(0) ) \Theta(0) 
 + 2 ( S(0) - \widetilde{S}(0) )^\top \Theta(0) + r(0)-\widetilde{r}(0) ,  \notag \\ 
  = & 
   \Theta^{\star \top} ( P(0) - \widetilde{P}(0) ) \Theta^\star  
 + 2 ( S(0) - \widetilde{S}(0) )^\top \Theta^\star + r(0)-\widetilde{r}(0) .  
 \notag 
\end{align} 
By the triangular inequality, the difference satisfies  
\allowdisplaybreaks\begin{align} 
  \big| L^\star (  \boldsymbol\alpha ; \mathcal{D} ) 
 -  L^\star (  \widetilde{\boldsymbol\alpha} ;  \widetilde{\mathcal{D}} ) \big|  
\leq  |  P(0) - \widetilde{P}(0) | \cdot |\Theta^\star|^2  
 + 2 | S(0) - \widetilde{S}(0) | \cdot | \Theta^\star | + | r(0)-\widetilde{r}(0) | . 
\label{estLast-tildeLast} 
\end{align}
We claim that $r(0)-\widetilde{r}(0)$ satisfies the estimate 
\allowdisplaybreaks\begin{align} 
 | r(0)-\widetilde{r}(0) | \leq  C \cdot \Big[ \sum_{i=1}^N w^\star_i |\mathcal{D}_i| \sum_{j=1}^{|\mathcal{D}_i|} | u^i_j - \tilde{u}^i_j |^2 
 + |y^i_j - \tilde{y}^i_j |^2 \Big]^{1/2} .  
\label{eqn:estr-tilder(0)} 
\end{align} 
Then the desired estimate \eqref{eqn:stabilityLast} is established by \eqref{estLast-tildeLast}, \eqref{eqn:estP-tildeP}, 
\eqref{eqn:estS-tildeS}  and \eqref{eqn:estr-tilder(0)}.

Now we prove the claim \eqref{eqn:estr-tilder(0)}. 
From \eqref{eqnr(t)}, the difference $r-\widetilde{r}$ satisfies  
\allowdisplaybreaks\begin{align} 
 & r(t) - \widetilde{r}(t) =  r(t+1) - \widetilde{r}(t+1) 
 - ( S(t+1) - \widetilde{S}(t+1) ) \big\{ (M(P(t+1)))^{-1} 
 \big[ S(t+1) - 2 \lambda \Theta^\star \big]  \notag \\  
 & \hspace{2cm} + (M(\widetilde{P}(t+1)))^{-1} \widetilde{S}(t+1) \big\}  \notag \\ 
 &\hspace{2cm} - \widetilde{S}^\top(t+1) \big[ (M(P(t+1)))^{-1} 
 - (M(\widetilde{P}(t+1)))^{-1} \big]  \big[ S(t+1) - 2 \lambda \Theta^\star \big]    \notag \\ 
 &\hspace{2cm} - \Theta^{\star \top} \lambda  \big[ (M(P(t+1)))^{-1} 
 - (M(\widetilde{P}(t+1)))^{-1} \big]  \lambda  \Theta^\star , \notag  \\ 
& r(T) - \widetilde{r}(T) =   \sum_{i=1}^N w^\star_i \sum_{j=1}^{|\mathcal{D}_i|} ( |y^i_j|^2 - |\widetilde{y}^i_j|^2 ) . \label{eqnr(t)-tilder(t)}  
\end{align} 
Inductively, $r(0)-\widetilde{r}(0)$ can be written in terms of $r(T)-\widetilde{r}(T)$ as   
\allowdisplaybreaks\begin{align} 
 & r(0) - \widetilde{r}(0) \notag \\ 
= & r(T) - \widetilde{r}(T) 
 - \sum_{t=1}^T \Big\{ ( S(t) - \widetilde{S}(t) ) \big[ (M(P(t)))^{-1} 
 \big( S(t) - 2 \lambda \Theta^\star \big)    
  + (M(\widetilde{P}(t)))^{-1} \widetilde{S}(t) \big]   \notag \\ 
 & + \widetilde{S}^\top(t) \big[ (M(P(t)))^{-1} 
 - (M(\widetilde{P}(t)))^{-1} \big]  \big[ S(t) - 2 \lambda  \Theta^\star \big]     
  + \Theta^{\star \top} \lambda  \big[ (M(P(t)))^{-1} 
 - (M(\widetilde{P}(t)))^{-1} \big]  \lambda \Theta^\star 
  \Big\} . \notag  
\end{align}
By Lemma~\ref{lm:|AB|<=|A||B|}, we have  
\allowdisplaybreaks\begin{align} 
&   | r(0) - \widetilde{r}(0) | \notag \\ 
 & \leq  | r(T) - \widetilde{r}(T) |
  + \sum_{t=1}^T \Big\{ \big| S(t) - \widetilde{S}(t) \big| 
 \cdot \Big[ \big| (M(P(t)))^{-1} \big| \cdot
 \big(  | S(t) | + 2 \lambda | \Theta^\star| \big)    
  + \big| (M(\widetilde{P}(t)))^{-1} \big| \cdot \big| \widetilde{S}(t) \big|  \Big]   \notag \\ 
 & \hspace{2cm} +\big| (M(P(t)))^{-1} 
 - (M(\widetilde{P}(t)))^{-1} \big|  \cdot  
 \big( \big| \widetilde{S}(t) \big| \cdot \big| S(t) - 2 \lambda  \Theta^\star \big|     
  +  \lambda | \Theta^\star |^2 \big) 
  \Big\} . \label{estr(0)-tilder(0)}  
\end{align}
By \eqref{eqnr(t)-tilder(t)}, Lemma \ref{lm:|A1+An|2<=(|A1|2+|An|2)n}, and the Cauchy-Schwarz inequality, the terminal condition satisfies 
\allowdisplaybreaks\begin{align} 
  |r(T) - \widetilde{r}(T)| 
  = &  \Big| \sum_{i=1}^N w^\star_i \sum_{j=1}^{|\mathcal{D}_i|} 
  ( y^i_j - \tilde{y}^i_j ) (y^i_j + \tilde{y}^i_j ) \Big| 
  \leq \sum_{i=1}^N w^\star_i \sum_{j=1}^{|\mathcal{D}_i|}  | (y^i_j - \tilde{y}^i_j) (y^i_j + \tilde{y}^i_j ) | \notag \\ 
 \leq & \sum_{i=1}^N w^\star_i \sum_{j=1}^{|\mathcal{D}_i|}  |y^i_j + \widetilde{y}^i_j| \cdot 
 |y^i_ j - \widetilde{y}^i_j |  
  \leq  2 K_y 
 \sum_{i=1}^N w^\star_i \sum_{j=1}^{|\mathcal{D}_i|}  |y^i_j -  \widetilde{y}^i_j| 
  \notag \\ 
 \leq & 2 K_y \Big( \sum_{i=1}^N w^\star_i  \sum_{j=1}^{|\mathcal{D}_i|}  |y^i_j - \widetilde{y}^i_j|^2 \Big)^{1/2} .  
 \label{estr(T)-tilder(T)}  
\end{align} 
The estimate \eqref{eqn:estr-tilder(0)} then follows from 
\eqref{estr(0)-tilder(0)}, \eqref{estr(T)-tilder(T)}, \eqref{eqn:estM-tildeM}, \eqref{eqn:estS-tildeS}, and \eqref{|S|bd}. 
\end{proof}

We may now derive the proof of Theorem~\ref{thrm:adv_robust} as a direct consequence of Corollary~\ref{cor:stabilityLast}.
\begin{proof}[{Proof of Theorem~\ref{thrm:adv_robust}}]
Fix $\widetilde{\mathcal{D}} \in \mathbb{D}^{p,\varepsilon}$ and let $I\subseteq \{(i,j):i=1,\dots,N,\,j=1,\dots,|\mathcal{D}_i|\}$ consist of all indices $(i,j)$ for which $\{(i,j):\, u^i_j\neq \widetilde{u}^i_j \mbox{ and } y^i_j\neq \widetilde{y}^i_j\}$.  By Corollary~\ref{cor:stabilityLast}, there is a constant $C>0$ which is independent of $\mathcal{D}$ such that
\allowdisplaybreaks
\allowdisplaybreaks\begin{align}
\nonumber
    \big| L^\star (  \boldsymbol\alpha ; \mathcal{D} ) 
    -  L^\star (  \widetilde{\boldsymbol\alpha} ;  \widetilde{\mathcal{D}} )  \big|^2
 \le &
    \widetilde{C} \sum_{i=1}^N w_i^\star |\mathcal{D}_i| \sum_{j=1}^{|\mathcal{D}_i|} | u^i_j - \widetilde{u}^i_j |^2  + |y^i_j - \widetilde{y}^i_j |^2 
\\ 
\nonumber
\le & 
\widetilde{C} \cdot \sum_{(i,j)\in I}\, w_i^\star |\mathcal{D}_i|  ( | u^i_j - \widetilde{u}^i_j |^2  + |y^i_j - \widetilde{y}^i_j |^2 ) 
\\ 
\nonumber
\le &
\widetilde{C} \cdot \sum_{(i,j)\in I}\, w_i^\star |\mathcal{D}_i| \cdot \max\{| u^i_j - \widetilde{u}^i_j |^2, |y^i_j - \widetilde{y}^i_j |^2\}
\\ 
\nonumber
= &
\widetilde{C} \cdot \sum_{(i,j)\in I}\, w_i^\star |\mathcal{D}_i| \cdot \max\{| u^i_j - \widetilde{u}^i_j |, |y^i_j - \widetilde{y}^i_j |\}^2
\\ 
\nonumber
\le &
\widetilde{C} \cdot \sum_{(i,j)\in I}\, w_i^\star |\mathcal{D}_i| \cdot 2\varepsilon^2
\\ 
\nonumber
\le &
\widetilde{C} \#I\, p \, 2\varepsilon^2
 ,
\end{align}
where we have set $\widetilde{C}\eqdef C^{1/2}$.
\end{proof}

\section{Intuition for our Transfer Learning Algorithm}\label{sec:Intuition for our Transfer Learning Algorithm}
In this section we give some intuition on our algorithm, which, at the same time, can also be understood as a motivation for the way it is defined. But first we elaborate on our method actually achieving transfer learning and not multi-task learning.

\subsection{Is it really transfer learning?}\label{sec:Is it really transfer learning?}
The terminal cost in \eqref{eq:penalized_regret_functional} uses a weighted combination of the loss on all datasets, leading to the question whether our method rather performs multi-task learning. This is not the case as explained below.

\Cref{alg:RegretOptimizationHeuristic__WarmStart} evaluates the predictive performance of each $\theta^\star_i$, $i\in[N]$ on the focal dataset $\mathcal{D}_1$ to determine the weights $w^\star_i$, $i \in [N]$, which 
rank the similarity of the datasets $\mathcal{D}_i$ to the focal dataset $\mathcal{D}_1$.
In particular, they quantify how much knowledge $\mathcal{D}_i$ provides for the task on $\mathcal{D}_1$. 
Based on those weights the ensemble parameters are biased towards maximizing performance on the primary dataset $\mathcal{D}_1$.  Contrarily, in multi-task learning, the weights $w^{\star}$ should rather be selected to be all equal (adjusting for dataset imbalances) so as to ensure that the aggregated learner is not biased towards any individual dataset.

The extreme cases best explain this. If all datasets are i.i.d.\ copies of $\mathcal{D}_1$, then they are equally weighted through $w^\star$ and the model is simply fine-tuned on the union of all samples as terminal cost (as it should be). 
On the other hand, if all datasets are completely orthogonal to each other, then the full weight is on $D_1$ and no "orthogonal" knowledge is transferred (again as it should be).  Indeed, in this case, the kernel regressors pre-trained on datasets $\mathcal{D}_i$ for $i>1$ will not exhibit any predictive power on the focal dataset $\mathcal{D}_1$ and visa-versa, so \Cref{alg:RegretOptimizationHeuristic__WarmStart} will set $w^{\star}=(1,0,0,\dots,0)$.  Therefore, \Cref{alg:RegretOptimization} returns
 $$
 \sum_i w_i \theta_i^\star = 1\theta_1^{\star} + 0\sum_{i>1}\theta_i^{\star} = \theta_1^{\star}
 .
 $$
The aggregated model will have no predictive performance on datasets $\mathcal{D}_i$ for $i\neq 1$ (since the datasets were assumed to be orthogonal to one another).  

Thus, Algorithm $2$ does not perform multi-task learning but transfer learning, as can also be seen in our experiments in \Cref{sec:American Option Pricing}.
Our algorithm can, however, be modified to perform multi-task learning.  Though we are not interested in this for this research paper, we do describe how it can be done in our main rebuttal, above.

\subsection{What the hyper-parameters $\lambda, \beta$ and $T$ control and how to choose them}\label{What the hyper-parameters lambda, beta and T control and how to choose them}
\subsubsection{The transfer learning factor $\lambda$}\label{sec:The transfer learning factor lambda}
The hyper-parameter $\lambda$ controls the level of transfer learning, by controlling 
how much the parameters $\theta_i$ are fine-tuned starting from $\theta^\star_i$. In particular, Algorithm 1 first determines how useful each of the datasets is for the focal task, and the terminal cost $l$ can be interpreted as a loss on the ``joint dataset weighted by $w^\star$''.
To understand how the $\lambda$ influences the result, it is best to consider the extreme cases and why they can be suboptimal in certain data regimes. For simplicity, we assume that $\beta=0$.

The choice $\lambda \approx \infty$ means that no deviation from $\theta^\star$ is allowed. Hence, this simply leads to the parameter combination $\sum_i w_i^\star \theta^\star_i$, so no additional fine-tuning on the ``joint weighted'' dataset happens. 
This is suboptimal if there is one very large and very useful dataset and otherwise small datasets.  For simplicity, consider $N=2$ datasets with the focal dataset $\mathcal{D}_1$ being comparatively small to $\mathcal{D}_2$.  Then, the pre-trained parameter $\theta^\star_1$ would be less informative but would be weighted as equally important to the more informative parameter $\theta^\star_2$. 
Note here that $\theta^\star_2$ evaluated on $\mathcal{D}_1$ cannot lead to a higher weight than for $\theta^\star_1$, since this is the optimal weight for $\mathcal{D}_1$. This is a natural limitation of Algorithm 1, since it cannot know that $\mathcal{D}_2$ is actually a larger i.i.d.\ copy of $\mathcal{D}_1$.
This shows that for $\lambda \approx \infty$, the inability to fine-tune will clearly produce a suboptimal aggregation.

On the other hand, setting $\lambda=0$, then  $\theta^{w^{\star}} = \sum_i w_i \theta_i$ 
purely minimize the weighted loss
$$
\sum_i w_i \sum_j (f_{\theta^w}(x_j^i) - y_j^i)^2 ,
$$
completely forgetting about the pre-trained parameters $\theta_1^{\star},\dots,\theta_N^{\star}$.
This is suboptimal if one dataset is much larger than the others and only slightly useful for $\mathcal{D}_1$. Indeed this would lead (in the limit where the size of this large dataset goes to $\infty$) to an optimizer of the large dataset simply because it dominates the loss as long as its weight $w_i^\star > 0$. 

In either case, we see that intermediate values of $\lambda$ yield the most desirable outputs of Algorithm $2$.

\subsubsection{Convergence behaviour controlled by $\beta$ and $T$}\label{sec:Convergence behaviour controlled by beta and T}
First, we note that $\beta$ controls the stability of the convergence by penalizing large steps in the sequence $\Theta^{0 \dotsc T}$. This is only needed when intermediate steps $\Theta(t)$ are of importance, i.e., if we want to obtain stable iterates, since otherwise we can simply focus on the terminal cost under the transfer learning constraint, by setting $\beta = 0$.  

In the regime $\beta = 0$ only the final $\Theta(T)$ is of importance, since our algorithm admits a closed form solution. In particular, it is optimal that all prior parameters for $t < T$ satisfy $\Theta(t) = \Theta^\star$. Hence, $T=1$ is sufficient and at the same time the most efficient and therefore optimal choice. 
Hence, we used $T=1$ in \Cref{sec:American Option Pricing}, where only the final $\Theta(T)$ is of interest. 

Contrarily, we note that for $T=1$ the stability term coincides with the transfer learning term, hence, they can be accumulated such that $\lambda + \beta$ acts as the factor for transfer learning.

This leaves us with the question whether there is a benefit of the regime $\beta > 0, T > 1$ at all. The answer is affirmative, since this regime "spreads out the convergence in time," which allows for the option of implicit regularization through an early stopping procedure. Though the investigation of the statistical properties of early stopping of our algorithm is beyond the scope of this paper, one would expect that using this regime  as a vehicle towards early stopping regularization would likely improve the statistical properties of our aggregated learning; as the benefits to model generalization of early stopping on classical learning algorithms are well documented \citep{zhang2005boosting,YaoRosascoCoponnetto_2007_CA,wei2017early,heissWhutteTeichmann__2022infinitely__Preprint}
In this regime, the total iteration number $T$ indeed represents a trade-off between the computing budget and the convergence of the loss function.
When computational resources permit, choosing $T$ large enough will ensure that the output $\sum_i w_i^\star \theta_t$ of the algorithm will converge to its $T=\infty$ limit; by a similar argument as in the proof of \citet[Theorem~4]{casgrain2021optimizing}.

\section{Experiment Details}
\label{a:ExperimentDetails}

\subsection{Hardware Specification}
\label{a:ExperimentDetails__ss:hardware}
All experiments were run on an
$2x$ Intel Xeon E5-$2699$ v$4$ processor, which as 22 Cores, operates at $2.20$GHz, and has $512$ GB RAM.

\subsection{Convergence Ablation}
\label{a:ExperimentDetails__ss:convergence}
We use $\Theta(0) = \Theta^\star_{(N)}$ and replace $\Theta^\star$ by $\Theta^\star_{(N)}$ in\eqref{eq:penalized_regret_functional}.
 
For all algorithms we use a ridge coefficient $\kappa=0$, regret coefficients $\lambda$ as specified, $\beta=1$  and Algorithm~\ref{alg:RegretOptimizationHeuristic__WarmStart} with information sharing level $\eta=10$ to get the corresponding optimal weights $w^\star$ and to initialize $\Theta(0)$ with the locally optimal parameters.

For our convergence analysis we use $N=5$ randomly generated datasets. In particular, each dataset $\mathcal{D}_i$ corresponds to a neural network $g_{\omega_i}$ with fixed randomly-generated hidden weights and biases.  Here $\omega_i$ denotes the vector of all these hidden weights and bias parameters. The networks have $d=5$ dimensional input, one hidden layer with $10$ hidden nodes, $\operatorname{ReLU}$ activation function and map to a $1$-dimensional output. The dataset is generated by sampling $S_i = 100$ inputs $x^i_j$ from a $d$-dimensional standard normal distribution and the corresponding target values are given by $g_{\omega_i}(x^i_j)$, i.e.,
\begin{equation*}
    \mathcal{D}_i = \{ (x^i_j, g_{\omega_i}(x^i_j)) \, | \, 1 \leq i \leq S_i, x^i_j \sim N(0,I_d) \}.
\end{equation*}

Our feature map $\phi$ in~\eqref{eq:FRKR} is a randomized neural network%
\footnote{Also called a random feature model or an extreme learning machine} %
with one hidden layer, $500$ hidden nodes and $\operatorname{ReLU}$ activation. 
In particular, the weights of the hidden layers are randomly initialized and fixed and only the weights of the last layer, corresponding to $\theta$ in \eqref{eq:FRKR}, are optimized. We note that by our choice for the number of hidden nodes to equal the total number of samples, the algorithms can in principle learn to have perfect replication on the training set, i.e., loss $l(\Theta(t), w;\,\mathcal{D}) = 0$.
For a fair comparison of the algorithms, we use the same randomly generated datasets and randomized neural networks for each method.

\begin{remark}
The comparison of our algorithms to gradient descent optimizing the loss is most meaningful in the $\lambda=0$ regime since otherwise, the objectives and, therefore, the behaviour are different. On the other hand, using $\beta > 0$ doesn't change the objective but only forces our algorithm to make equally sized steps $\Delta \Theta (t)$, where the step size depends on $\beta$ and $T$.
\end{remark}

\subsection{Details: American Option Pricing}
\label{a:ExperimentDetails__ss:AOPricing}
In this application we are only interested in the terminal parameters $\Theta(T)$, while a stable convergence behaviour is not important.
Hence, we choose $T=1$ to fully focus on the terminal loss under the transfer learning constraint (cf.\ \Cref{sec:Convergence behaviour controlled by beta and T}). 

As fKRR we use the same randomized neural networks as in Section~\ref{sec:Convergence of the Regret Optimal Algorithm}.
The random weights of the randomized neural networks, as well as the randomly sampled paths both introduce some randomness into the pricing problem. Therefore we follow \cite{herrera2023optimal} by always running the experiments $n_{runs} = 100$ times with different seeds and taking means over the prices. For easier comparison, we then consider the relative performance ($RP$) on our main dataset, computed as the mean price of any given method divided by the mean price of the local optimizer on the main dataset.
Additionally, we compute asymptotically valid (by the central limit theorem) $95\%$-confidence-intervals as $RP \pm z_{0.975}\hat\sigma/\sqrt{n_{runs}}$, where $z_{0.975}$ is the $0.975$-quantile of the standard normal distribution and $\hat\sigma$ is the standard deviation  of $RP$ (computed as the sample standard deviation of the prices over the runs divided by the mean price of the local optimizer on the main dataset).
To make the results comparable between the optimization methods, we use the same seeds for all of them, such that run $i$ always uses the exact same paths and the exact same random weights in the randomized neural network, independently of the chosen method.
Moreover, we always use $\nu_2 \geq 50,000$ evaluation paths, such that the noise of the Monte Carlo approximation of the expectations is relatively small when evaluating the prices, even if we use much smaller training set sizes $\nu_1$.

\subsubsection{Experiment 1}
\label{sec:Experiment 1}
In experiment $1$, the main dataset $\mathcal{D}_1$ has the parameters $r=0.05$, $\delta=0.1$, $k=2$, $v_{\infty}=0.01$, $\sigma=0.2$, $\rho=-0.3$ (the correlation between the two Brownian motions $W$ and $B$) and starting values $X_0=100$ and $v_0=v_\infty$.
The other datasets are define as all combinations of $r \in \{0.05, 0.5\}$, $\sigma \in \{ 0.15, 0.2, 0.25 \}$, $v_{\infty} \in \{ 0.005, 0.015 \}$ and otherwise the same parameters as the main dataset.
The biggest difference between the datasets constitutes the drift $r$, while the other parameters only lead to small gaps in the dataset distributions. In line with this, the 6 datasets with $r=0.05$ are similar to $\mathcal{D}_1$, while the other 6 datasets with $r00.5$ are very different from $\mathcal{D}_1$.
For all datasets we use a max call option with strike $K=100$ on $d=2$ i.i.d.\ stocks of the given model with $M=9$ equidistant exercise dates until the maturity $T_{mat} = 3$. The discounting factor is $e^{-r\,T_{mat}\,t/M}$ for $0\leq t \leq M$.
Moreover, we use a randomized neural network with one hidden layer with $300$ nodes and all methods use the ridge coefficient $\kappa=2$ and, where applicable, regret coefficients $\lambda=2, \beta=1$.

More detailed results of Table~\ref{table:results exp 1 short} are listed in Table~\ref{table:results exp 1}.
\begin{table}[h]
    \centering
    \caption{Relative performance (RP) and $95\%$-confidence-intervals for different optimization methods in Experiment~\ref{exp:1}. We compare the local optimizers on the different datasets (LO-$n$), with the mean local optimizer (MLO), the joint optimizer (JO) and our regret-optimal method (RO). The \textit{oracle} local optimizer on the main dataset with additional training samples (standard: $100$ samples per dataset) is included.  
    When the information sharing parameter $\eta$ is set to $100$, the regret-optimal (RO) algorithm outperforms all ``non-oracle'' baselines. }
    \begin{tabular}{@{}ccc@{}}
    \cmidrule[0.3ex](){1-3}
    method &  $RP$ &  $95\%$-CI \\
    \midrule
    \begin{filecontents*}{Tables/training_overview-1-1.csv}
,name,RP,CI
0,LO-1,1.000,$[0.991; 1.009]$
1,LO-2,0.985,$[0.977; 0.994]$
2,LO-3,0.966,$[0.958; 0.973]$
3,LO-4,0.975,$[0.966; 0.984]$
4,LO-5,0.966,$[0.957; 0.975]$
5,LO-6,0.979,$[0.971; 0.988]$
6,LO-7,0.978,$[0.971; 0.986]$
7,LO-8,0.721,$[0.711; 0.730]$
8,LO-9,0.720,$[0.711; 0.729]$
9,LO-10,0.706,$[0.696; 0.716]$
10,LO-11,0.725,$[0.715; 0.736]$
11,LO-12,0.707,$[0.697; 0.716]$
12,LO-13,0.718,$[0.708; 0.727]$
14,MLO,0.823,$[0.813; 0.833]$
13,JO,0.886,$[0.878; 0.894]$
3,JO (datasets 1-7),1.062,$[1.057; 1.067]$
18,RO ($\eta=10$),1.000,$[0.990; 1.010]$
19,\textbf{RO ($\eta=100$)}, \textbf{1.090},\textbf{[1.086; 1.095]}
20,RO ($\eta=500$),1.050,$[1.043; 1.057]$




\end{filecontents*}
\csvreader[head to column names, late after line=\\]
    {Tables/training_overview-1-1.csv}{}%
    {  \name & \RP & \CI }%
    \midrule
    \csvreader[head to column names, late after line=\\]
    {Tables/training_overview-1-2.csv}{}%
    {  \name & \RP & \CI }%
    \bottomrule
    \end{tabular}
    \label{table:results exp 1}
    \end{table}

\subsubsection{Experiment 2}
\label{sec:Experiment 2}
The rough Heston model
\begin{align*}
    dX_t &=  (r - \delta) X_t dt +  \sqrt{v_t} X_t dW_t, \\
    v_t &= v_0 + \int_0^t \frac{(t-s)^{H-1/2} }{\Gamma(H + 1/2)} \kappa (v_\infty - v_s)ds + \int_0^t \frac{(t-s)^{H-1/2} }{\Gamma(H + 1/2)} \sigma \sqrt{v_s}dB_s.
    \end{align*}
is similar to the Heston model except that the volatility is a rough process with Hurst parameter $H\in(0,1/2]$, where the case $H=1/2$ coincides with the standard Heston model.
For the main dataset, we use $H=0.1$ and otherwise the same parameters as for the main dataset in Experiment~\ref{exp:1}. The second and third dataset are generated from a standard Heston model (i.e.\  $H=1/2$), with the same parameters, except that the drift is set to $r=0.5$ for the second dataset.
In particular, dataset 3 is the closest non-rough dataset to the main dataset, while dataset 2 has a very different distribution.

\section{Background}
\label{a:Recal}
\subsection{Finite Rank Kernel Ridge Regression}
\label{a:Recal__ss:FKRR}
In this paper, we consider regret-optimal memoryless training dynamics for \textit{\textbf{f}inite-rank \textbf{k}ernel \textbf{r}idge \textbf{r}egressors} (fKRR) trained from $N$ finite datasets drawn from $N$ (possibly) different data-generating distributions.
We are considering a scenario in which these datasets can originate from $N$ distinct tasks or be adversarially generated perturbations of the initial dataset, and we aim for our fKRR to exhibit robustness towards them.

Any continuous \textit{feature map} $\phi:\mathbb{R}^d\rightarrow \mathbb{R}^p$, induces a hypothesis of fKRR $f_{\theta}:\mathbb{R}^d\rightarrow \mathbb{R}$ each of which is uniquely determined by a \textit{trainable parameter} $\theta\in \mathbb{R}^p$ via
\begin{equation}
\label{eq:FRKR 1}
f_{\theta}(x)=\phi(x)^\top\theta;
\end{equation}
where all vectors are represented as column-vectors and not row vectors.

The hypothesis class $\mathcal{H}$ possesses a natural reproducing kernel Hilbert space (RKHS) structure with inner-product given for $f_{\theta},f_{\tilde{\theta}}\in \mathcal{H}$ by $\langle f_{\theta},f_{\tilde{\theta}}\rangle_{\phi}\eqdef \sum_{k=1}^p\, \theta_k\tilde{\theta}_k$.  This RKHS structure suggests an optimal selection criterion to decide which hypothesis in $\mathcal{H}$ best describes a \textit{single} training set $\mathcal{D}_1\eqdef \{(x^1_j,y^1_j)\}_{j=1}^{S^1}$ of $S^1\in \mathbb{N}_+$, by solving the penalized least-squares problem

\begin{align} 
\label{eq:FRKR 2}
 \theta^\star_1 \in \argmin_{\theta \in \mathbb{R}^p } \big\| \Phi(x^1) \theta  - y^1  \big\|_2^2 
  + \kappa \big\| \theta \big\|_2^2 , 
\end{align} 
where 
\begin{align} 
\Phi(x^1) \eqdef 
\big( \phi(x^1_1), \ 
\phi(x^1_2), \ \cdots , \ 
 \phi(x^1_{S^1}) \big)^\top \in \mathbb{R}^{S^1 \times p} , \ 
y^1 \eqdef (y^1_1, y^1_2, \cdots, y^1_{S^1})^\top 
 \in \mathbb{R}^{S^1} , \notag  
\end{align} 
and $\kappa>0$ is a fixed hyperparameter.  In this classical case, the Representer Theorem \cite{KimeldorfWahba__Representer_OriginalI,KimeldorfWahba__Representer_OriginalII} states that the optimal choice of $f_{\theta}$ to the convex optimization problem~\eqref{eq:FRKR 2} is parameterized by
\begin{equation} 
\label{eq:loss_one_dataset}
 \theta^\star_1 = 
 ( \Phi^\top(x^1) \Phi(x^1)  
  + \kappa I_p )^{-1} 
  \Phi^\top(x^1) y^1 ,
\end{equation} 
Alternatively, the optimal parameter choice $\theta^\star_1$ can be arrived at via \textit{gradient descent} since~\eqref{eq:FRKR 2} has a strictly convex loss with a Lipschitz gradient.
The training of fKRR on a single dataset is very-well understood.  Likewise, other aspects of fKRR are well understood; namely, their statistical properties \cite{KRRExactTestError}, their in-sample behaviour in \cite{amini2022target}, and even the approximation power and risk of randomly generated fKRR are respectively known \cite{gonon2022approximation} and \cite{gonon2020risk}.

Using the fKRR \eqref{eq:FRKR 1}, the joint terminal time objective simplifies to
\allowdisplaybreaks\begin{align} 
\label{eq:terminal_time_objective__fKRR}
\begin{aligned}
        l(\theta_1,\dots,\theta_N, w;\,\mathcal{D} )  
    \eqdef & 
        \sum_{i=1}^N 
        w_i
        \biggl(
            \sum_{j=1}^{S_i} 
                \big| 
                \langle 
                    u^i_j
                    ,
                    \theta^{w}
                \rangle - y^i_j  \big|^2 
        \biggr)
    \\
    \mbox{ s.t. }
    \,
    \theta^w\eqdef & \sum_{i=1}^N \theta_i
        w_i
 ,
\end{aligned}
\end{align} 
where $\langle\cdot,\cdot\rangle$ denotes the Euclidean inner-product, $u^i_j\eqdef \phi(x^i_j)$ and $w = (w_1, \cdots, w_N) \in [0,1]^N$ with $\sum_{i=1}^N\,w_i=1$.

\section{Background: American Option Pricing and Optimal Stopping}
\label{a:Recap:AOP_OS}
This section contains a brief recap of the optimal stopping problem in American option pricing.

An American option is a financial derivative that gives its owner the right to execute a certain predetermined trade at any time between the start date and the maturity of the option. For example, a call (put) option gives its owner the right to buy (sell) a certain stock from the writer of the option for a fixed price $K$, called the strike. If the current stock price $X_t$ is higher (lower) than $K$ then the option owner can make a profit of $(X_t - K)_+$ (or $(K - X_t)_+$ respectively) by directly selling (buying) the stock again in the open market after having exercised their right to buy (sell) from the option writer. 
While American options allow for continuous-in-time execution, it is common practice to approximate them via Bermudan options, which can be executed on a predetermined time grid $t_0 < \dotsb t_M$. We briefly overview the American option pricing problem, details of which can be found in \citep[Chapter 25]{PeshirShiryaev_FreeBoundaryProblems_2006Book}.

In order to capture the computational challenges posed by the American option pricing problem in higher dimensions, we consider a stochastic process $(X_t)_{t=0}^M \in (\R^d)^{(M+1)}$ describing the price of $d$ stocks on the time grid 
and let $(Z_t)_{t=0}^M \in \R^{(M+1)}$ be the corresponding discounted payoff process of an American option written on a one-dimensional process obtained from $X_t$ (e.g. the largest price among the $d$ stocks). We define the filtration $\mathbb{F} = (\mathcal{F}_t)_{t=0}^M$ generated by $X$ via $\mathcal{F}_t = \sigma(X_s | s\leq t)$. Then the price of the American option is given by the starting value $U_0$ of the \emph{Snell envelope} \citep{SnellOGPaper_1952_SemiMartingaleOptimalStopping}, given by
\begin{equation}\label{equ:snell envelope 2}
U_m = \operatorname{sup}_{\tau \in \mathcal{T}_m} \E[ Z_{\tau} \, | \, \mathcal{F}_m ],
\end{equation}
where $\mathcal{T}_m$ is the set of all stopping times $\tau \geq m$. 
The smallest optimal stopping time is
\begin{align}\label{equ: optimal stopping time}
\tau_M &:= M, \\
\tau_m &:= \begin{cases}
m, & \text{if } Z_m \geq \E[U_{m+1}  \, | \, \mathcal{F}_m],\\
\tau_{m+1}, & \text{otherwise},
\end{cases} \nonumber
\end{align}
hence, to compute the price of the American option it suffices to compute the \emph{continuation values} $\E[U_{m+1}  \, | \, \mathcal{F}_m]$ for all $m$. 
In the following, we will use the \textit{randomized least squares Monte Carlo} (RLSM) algorithm of \cite{herrera2023optimal}, which utilises randomized neural networks to approximate the continuation values via backward induction (note the backward recursive scheme that \eqref{equ:snell envelope 2} and \eqref{equ: optimal stopping time} imply).
In particular, this leads to $M$ backward-recursively defined regression problems that need to be solved sequentially. Since \eqref{equ:snell envelope 2} is a maximisation problem and since we only approximate the optimal stopping time, we compute lower bounds of the price $U_0$.
The standard procedure is to fix some law for the process $X$ and to sample a dataset of $\nu = \nu_1+\nu_2$ i.i.d.\ paths of $X$, where $\nu_1$ of them are used for training the randomized neural networks (approximating the continuation value functions) and $\nu_2$ of them are used to evaluate the resulting price $U_0$ (using these trained networks). 
Through the independence of the two parts of the dataset it is clear that any over-fitting bias the networks might learn will \emph{not} lead to prices that are too large. This is important, since the quality of a method can therefore be evaluated by the resulting price,with a higher price indicating a better result. 

\newpage
\section{An Accelerated Heuristic}
\label{s:Main__ss:AcceleratedHeuristic}
We now discuss a situation in which the computational complexity of the regret-optimal algorithm can be reduced.  This is achieved by exploiting symmetries in the optimal policy computed by Algorithm~\ref{alg:RegretOptimization}.

Examining the explicit update rules, we notice that \textit{if} the weights $w$ defining the energy~\eqref{eq:penalized_regret_functional} were to be $\overline{1/N}\eqdef (1/N,\dots,1/N)$ then the matrices defining the optimal policy/control's updates would become highly symmetric Toeplitz matrices, depending on exactly three $p$-dimensional block sub-matrices.  This allows us to parameterize the entire regret-optimal algorithm, for the sub-optimal weights $\overline{1/N}$, using only three low-dimensional systems whose dimension is \textit{independent} of the number of datasets $N$.  In line with this, Algorithm~\ref{alg:RegretOptimizationHeuristic} implements the regret-optimal algorithm for the weight specification $\overline{1/N}$ by only keeping track of these three optimal low-dimensional systems, and then efficiently recombining them by leveraging the Toeplitz structure defining the optimal policy/control.  These low-dimensional systems are defined using \textit{helper functions}, listed in equations \eqref{gamma1}-\eqref{eq:RegretOptimizationHeuristic_Helper__1}.

The approximate near-optimality of Algorithm~\ref{alg:RegretOptimizationHeuristic} only depends on the stability of the optimal policy/control, defining the algorithm's updates as a function of the weights $w$.  If the optimal weights defining the systemic regret functional~\eqref{eq:penalized_regret_functional__REGRETFORM} are close to $\overline{1/N}$, then Algorithm~\ref{alg:RegretOptimizationHeuristic} is nearly regret-optimal.  Otherwise, the sub-optimality it incurs in favor of computational efficiency is explicitly quantifiable and depends only on the difference between the optimal weights computed in Algorithm~\ref{alg:RegretOptimizationHeuristic__WarmStart} and the symmetric weights $\overline{1/N}$.

Note now record, the matrix-valued functions $\gamma_1(\cdot)$ and $\gamma_2(\cdot)$ used to define the updates in Algorithm~\ref{alg:RegretOptimizationHeuristic}.  

\begin{algorithm}[htp!]%
\caption{Accelerated (Nearly) Regret-Optimal Algorithm}
\label{alg:RegretOptimizationHeuristic}
\begin{algorithmic}
\SetAlgoLined
\Require Datasets  $\mathcal{D}_1,\dots,\mathcal{D}_N$, $N$.\ Iterations $T\in \mathbb{N}_+$, finite-rank kernel $\phi$, hyperparameters $\lambda, \beta, \kappa, \eta >0$.
\DontPrintSemicolon
    \State  
\tcp{Get Initialize Weights and Locally-Optimal fRKR Parameters}
       \State $\theta^\star_1,\dots,\theta^\star_N,w^{\star}$ $\leftarrow$ Run: Algorithm~\ref{alg:RegretOptimizationHeuristic__WarmStart} with $\mathcal{D}_1,\dots,\mathcal{D}_N$, $\phi$, and $\kappa,\eta$.
    \State 
  \tcp{Initialize Updates}
    \State $\pi_1(T) 
    = \pi_2(T) 
    = (1/N^3) \sum_{i=1}^N \sum_{j=1}^{|\mathcal{D}_i|}  u^i_j u^{i\top}_j$
    \State $\pi_3(T) 
    =  - (1/N^2) \sum_{i=1}^N \sum_{j=1}^{|\mathcal{D}_i|}  {u^i_j} y^i_j $
 \State $\theta_1(0) = \theta^\star_1$, $\theta_2(0) = \theta^\star_2$, $\dots$, $\theta_N(0) =\theta^\star_N$
\State $\theta^\star_{(N)} = \frac{1}{N} \sum_{i=1}^N \theta^\star_i .$
    \State \tcp{Generate Iterates}
    \State \For{$t=T-1,\dots,0$}{
        \State \tcp{Update Low-Dimensional Driving Parameters} 

\State$\pi_1(t)= \beta I_p- \beta^2 \gamma_1(t+1)$ 
\tcp{$\gamma_1(\cdot)$ is defined by \eqref{gamma1}}

\State $\pi_2(t) = - \beta^2 \gamma_2(t+1)$ 
\tcp{$\gamma_2(\cdot)$ is defined by \eqref{gamma2}}

\State $\pi_3(t)
= \beta \big[ \gamma_1(t+1)
+ (N-1) \gamma_2(t+1)
\big] ( \pi_3(t+1) -\lambda \theta^\star_{(N)})$ }

\State\For{$t=0, ..., T-1$}{ 
  \State \tcp{Update Low-Dimensional Control} 
 \State \For{$i=1, ..., N$}{
  \State $\widehat{\alpha}_i(t)
  =  \widehat{\alpha}\big(\theta_i(t) ,\ 
 (\theta_j(t))_{j = 1,j\neq i}^N ,\ (\gamma_i(t+1))_{i=1}^2, 
   \ (\pi_i(t+1))_{i=1}^3 \big)$ 
   
 \State $\theta_i(t+1) =\theta_i(t) + \widehat{\alpha}_i(t)$ } } 
 
 \State \Return Synchronized fKRR 
$f_{\theta^{w^\star}}$
\end{algorithmic}
\end{algorithm} 

The matrix-valued functions $\gamma_1(\cdot)$ and $\gamma_2(\cdot)$ in Algorithm~\ref{alg:RegretOptimizationHeuristic} are defined as 
\allowdisplaybreaks 
\begin{align} 
& \gamma_1(\cdot) \eqdef 
\big\{ (\lambda+\beta )I_p
+\pi_1(\cdot) \notag \\
& \hspace{2cm} - (N-1) \pi_2(\cdot)\big[ (\lambda+ \beta )
I_p + \pi_1(\cdot) 
+ (N-2)\pi_2(\cdot)]^{-1} \pi_2(\cdot) \big\}^{-1} , 
\label{gamma1} \\ 
 & \gamma_2(\cdot)  \eqdef 
 - \gamma_1(\cdot) \pi_2(\cdot) 
 \big[ (\lambda+\beta )I_p + \pi_1(\cdot) 
 + (N-2)\pi_2(\cdot) \big]^{-1},  \label{gamma2}
\end{align} 
where 
\begin{align*}
    \pi_1(T) = \pi_2(T) = \frac{1}{N^3} \sum_{i=1}^N \sum_{j=1}^{|\mathcal{D}_i|}  u^i_j u^{i\top}_j 
    \qquad 
    \mbox{ and }
    \qquad
    \pi_3(T) =  - \frac{1}{N^2} \sum_{i=1}^N \sum_{j=1}^{|\mathcal{D}_i|}  {u^i_j} y^i_j 
\end{align*}
and, for $t=T-1,\dots,0$,
\begin{align*}
  \pi_1(t) &= \beta I_p- \beta^2 \gamma_1(t+1) \\ 
  \pi_2(t) &= - \beta^2 \gamma_2(t+1) \\
  \pi_3(t) &= \beta \big[ \gamma_1(t+1)
+ (N-1) \gamma_2(t+1)
\big] \left( \pi_3(t+1) -\lambda \theta^\star_{(N)}\right)
\end{align*}
where $\theta^\star_{(N)} = \frac{1}{N} \sum_{i=1}^N \theta^\star_i$.

The updates of the low-dimensional systems which Algorithm~\ref{alg:RegretOptimizationHeuristic} exploits are implemented via the helper function $\hat \alpha$, defined as 
\allowdisplaybreaks 
\begin{align} 
\label{eq:RegretOptimizationHeuristic_Helper__1}
  \widehat{\alpha}\big( \theta_i, ((\theta_j)_{j=1,j\neq i}^N) , (\gamma_i)_{i=1}^2,(\pi_i)_{i=1}^3 \big)  \eqdef &  
 - (\gamma_1-\gamma_2)(\lambda I_p + \pi_1 - \pi_2 )\theta_i \notag \\ & 
 - \big\{ \gamma_1 \pi_2 + \gamma_2 [ \lambda I_p + \pi_1 + (N-2)\pi_2 ] \big\} \sum_{j=1}^N \theta_j \notag \\ 
 & - \big[ \gamma_1 + (N-1) \gamma_2 \big] 
 (\pi_3 - \lambda \theta^\star_{(N)}) . 
 \end{align}

\vspace{0.2in} 
The degree to which Algorithm~\ref{alg:RegretOptimizationHeuristic} is nearly regret-optimal depends on how far the weights $w$ determined by Algorithm~\ref{alg:RegretOptimizationHeuristic__WarmStart} are from the equal weighting $\overline{1/N}$.  This is precisely quantified by the following results.

\begin{theorem}[Near-Regret Optimality]\label{thrm:Epsilon_Efficiency}
Suppose that $\mathcal{D}$ satisfies Assumption~\ref{assm:bddxy}, and let $\hat\Theta^{0\cdots T}$ denote the unique regret-optimal algorithm\footnote{That is, the unique optimizer of~\eqref{eq:penalized_regret_functional}} computed according to Algorithm \ref{alg:RegretOptimization}. Then, the output $\vartheta^{0\cdots T}$ of Algorithm~\ref{alg:RegretOptimizationHeuristic} satisfies
\vspace{0.2in}
\begin{enumerate}
    \item[(i)] \textbf{Near-Optimality:} The difference $\big|
            {\cal L}(\Theta^{0\cdots T})
            -
            {\cal L}(\vartheta^{0\cdots T})
        \big|$ is of the order
    \begin{align}
    &
        \mathcal{O}\Big(
                   \bar{N} N^3 \max_{1\leq i, k, l \leq N} \Big| \frac{1}{N^3} -  w^\star_k w^\star_l w^\star_i\Big| 
                  + \bar{N} N^{3/2} \max_{1\leq k, l\leq N} \Big| \frac{1}{N^2} - w^\star_k w^\star_l \Big| 
                  \notag \\ 
& \hspace{5cm} + \bar{N}  \max_{1\leq i \leq N} \Big|\frac{1}{N} - w_i^\star \Big| 
 + N \big\|\Theta^\star - \Theta^\star_{(N)} \big\|_2 
        \Big)
        \notag
    ,
    \end{align}
    \item[(ii)] \textbf{Near-Optimal Increments:} The difference $\|\Delta \Theta^{0\cdots T-1}-\Delta \vartheta^{0\cdots T-1} \|_{\ell^1}$ is of the order
    \begin{align} 
       & 
        \mathcal{O}\Big(
                \bar{N} N^{3/2} \max_{1\leq i, k, l \leq N} \Big| \frac{1}{N^3} -  w^\star_k w^\star_l w^\star_i \Big| \notag \\  
 & 
 + \bar{N} N^{1/2} \max_{1\leq k, l\leq N} \Big| \frac{1}{N^2} - w^\star_k w^\star_l \Big| 
  +  \big\|\Theta^\star - \Theta^\star_{(N)}\big\|_2
        \Big),
        \notag
    \end{align}
\end{enumerate}
where $\mathcal{O}$ hides a constant depending only on $\lambda$, $\beta$, 
$\Theta^{\star}$, $(w_1^\star, ..., w_N^\star)$, $T$, $K_x$, and in $K_y$.  
\end{theorem}

The next result shows that Algorithm~\ref{alg:RegretOptimizationHeuristic} accelerates the regret-optimal algorithm by a linear factor in $N$ and by a quadratic factor in $p$, when $N$ is large.  
\begin{theorem}[$\mathcal{O}\big(N\,p^2\big)$-Acceleration over Regret-Optimal Algorithm]
\label{thrm:accelertion_heuristic_regret_optimal_algorithm}
With the convention that floating-point arithmetic operations are $\mathcal{O}(1)$, we have that:
\begin{enumerate}
    \item[(i)] \textbf{Complexity:} Algorithm~\ref{alg:RegretOptimizationHeuristic} has a complexity of $\mathcal{O}(TNp\max\{N,p\}))$
    \item[(ii)] \textbf{Acceleration over Regret-Optimal Algorithm:} The regret-optimal algorithm has a complexity of $\mathcal{O}\big(TN^3 p^3\big)$.
\end{enumerate}
When $N\ge p$, Algorithm~\ref{alg:RegretOptimizationHeuristic} requires $\mathcal{O}(N\,p^2)$ fewer operations than the regret-optimal algorithm.  
\end{theorem}

Finally, we establish that Algorithm~\ref{alg:RegretOptimizationHeuristic} is the regret-optimal algorithm, optimizing the systemic regret functional~\eqref{eq:penalized_regret_functional}, in the ``symmetric'' case where the optimal weights $w^{\star}$ are all $(1/N,\dots,1/N)$.  
\begin{theorem}[Regret-Optimality in the Symmetric Case]
\label{thrm:regretoptimality_SymmetricCase}
Suppose that $w^{\star}_i=1/N$ for $i=1,\dots,N.$  Then, $\hat\Theta^{0\cdots T}=\vartheta^{0\cdots T}$.  
\end{theorem}



The computational complexity of the matrix-valued function $P(\cdot)$, as determined by \eqref{eqnP(t)}, does not primarily stem from the high dimensionality of the output matrices, but rather from their inherent lack of symmetries.
As we will see later, such symmetry will allow us to efficiently encode $P(\cdot)$ into low-dimensional structures, therefore completely mitigating the high-dimensionality effect.

We therefore introduce the following control, which is optimal for a ``symmetrized version'' of the control problem~\eqref{optimala} where $w^{\star}=(1/N,\dots,1/N)$.  This advantages this surrogate problem is that the associated $P$ function, which we denote by $\widehat{P}$, outputs highly symmetric\footnote{As formalized in Proposition~\ref{prop:P1/Nsubmat}.} Toeplitz matrices which are determined by exactly three factors. This allows us to encode the entire control problem into a control problem on $\mathbb{R}^{\tilde{O}(P)}$ whose dimension is independent of $N$. 
Furthermore, the optimal controls of both problems can be related by stability estimates, depending only on the deviation of $(1/N,\dots,1/N)$ from $w^{\star}$, which allows us to infer the degree of sub-optimality of our $\widehat{P}$.  Therefore, we consider the control
\begin{align} 
 \widehat{\boldsymbol{\alpha}}(t)  
 = & -[(\lambda + \beta )I_{Np} + \widehat{P}(t+1)]^{-1} 
 \big[ (\lambda I_{Np} + \widehat{P}(t+1)) \Theta(t) 
 - \lambda  \Theta^\star_{(N)} + \widehat{S}(t+1)  \big]   \label{hata} 
 .
\end{align} 
We will show that the algorithm under control \eqref{hata} enjoys lower computational complexity in comparison with the regret-optimal algorithm \eqref{optimala}, whose computational complexity is summarized through Theorem~\ref{thm:ComplexityAlgorithm}. Furthermore, we will then establish the near-regret optimality of \eqref{hata}, stated in Theorem \ref{thrm:Epsilon_Efficiency}.

The associated matrix-valued functions $\widehat{P}$ and $\widehat{S}$ are uniquely determined by 
\allowdisplaybreaks\begin{align} 
& \begin{cases}
   \widehat{P}(t)  =  \beta I_{Np} - \beta^2 [(\lambda+\beta ) I_{Np} + \widehat{P}(t+1) ]^{-1} , \label{eqnhatP1/N(t)} \\ 
   \widehat{P}(T) =  [ (1/N) I_p ,\cdots, (1/N) I_p]^{\top} 
 \Big( \sum_{i=1}^N (1/N) \sum_{j=1}^{|\mathcal{D}_i|} u^i_j u^{i\top}_j  \Big)  
 [ (1/N) I_p , \cdots, (1/N) I_p ] , 
\end{cases} \\ 
& \begin{cases}
  \widehat{S}(t) =  
  \beta [(\lambda + \beta ) I_{Np} + \widehat{P}(t+1) ]^{-1} 
  \big( \widehat{S}(t+1) - \lambda \Theta^\star_{(N)} \big)
  \label{eqnhatS1/N(t)} ,     \\ 
        \widehat{S}(T) 
    =  
  - [(1/N)I_p, \cdots, (1/N)I_p ]^{\top} \sum_{i=1}^N(1/N)\sum_{j=1}^{|\mathcal{D}_i|}  u^i_j y^i_j  
   ,  
\end{cases} 
\end{align} 
where $[ I_p, \cdots, I_p ]$ denotes the $Np\times p$-dimensional matrix formed by concatenating $N$ copies of the $p\times p$-dimensional identity matrix $\frac{1}{N} I_p$.   The parameter $\Theta^\star_{(N)}$ is defined as 
\allowdisplaybreaks\begin{align} 
 \Theta^\star_{(N)} \eqdef [\theta_{(N)}^{\star \top}, \cdots, \theta_{(N)}^{\star \top}]^{\top} , \quad 
\theta^\star_{(N)} \eqdef \frac{1}{N} \sum_{i=1}^N \theta^\star_i . \notag 
\end{align}

\begin{remark} 
According to \eqref{Sum(uij)Tuij>=0} in Lemma~\ref{lm:PD}, $\widehat{P}(T)\geq 0$ and therefore \eqref{eqnhatP1/N(t)} admits a unique solution. 
\end{remark} 

\begin{remark}
\label{rmk:M(hatP)bdd}
Similar to Corollary~\ref{cor:M(P)invbdd}, we have $\big| [ (\lambda+\beta)I_{Np} + \widehat{P}(t) ]^{-1} \big| \leq (\lambda + \beta)^{-1}$. By replacing $(w_1^\star, ..., w_N^\star)$ in the proof of Corollary~\ref{cor:M(P)invbdd} with $(1/N, ..., 1/N)$, we obtain 
$|\widehat{S}(t)|\leq C \cdot N^{1/2}$ for all $t=1$, ..., $T$, 
where $C>0$ is a constant depending on $\lambda$, $\beta$, $\Theta^\star$, $K_x$, $K_y$ and $T$.  
\end{remark}

\begin{lemma} 
\label{lem:lemma_A}
 Let $P$ and $\widehat{P}$ be the solutions of the systems \eqref{eqnP(t)} and \eqref{eqnhatP1/N(t)}, respectively. Then they satisfy that for all $t=1$, $2$, ..., $T$, 
\allowdisplaybreaks
\begin{align} 
  \big| P(t) - \widehat{P}(t) \big| 
 \leq  C \cdot \bar{N} N \max_{1\leq i, k, l \leq N} \Big| \frac{1}{N^3} -  w^\star_k w^\star_l w^\star_i \Big| ,  
\label{estPw-P1/N} \\ 
 \big| (M(P(t)))^{-1} - (M(\widehat{P}(t) ))^{-1} \big| 
 \leq  C \cdot  \bar{N} N \max_{1\leq i, k, l \leq N} \Big| \frac{1}{N^3} -  w^\star_k w^\star_l w^\star_i \Big| ,  
\label{estMPw-MP1/N}  
\end{align}
where $C>0$ is a constant depending on $\lambda$, $\beta$, $K_x$. 
\end{lemma} 

\begin{proof}[{Proof of Lemma~\ref{lem:lemma_A}}] 
By \eqref{eqnhatP1/N(t)}, \eqref{eqnP(t)}, and Lemmas~\ref{lm:|AB|<=|A||B|} and \ref{lm:|A1+An|2<=(|A1|2+|An|2)n}, the difference 
$\widehat{P}(T)-P(T)$ satisfies    
\allowdisplaybreaks
\begin{align} 
|\widehat{P}(T) - P(T)|^2 = & \sum_{k,l=1}^N 
\Big| \frac{1}{N^2} \sum_{k=1}^N \frac{1}{N} 
\sum_{j=1}^{|\mathcal{D}_i|}  u^i_j u^{i\top}_j  - w^\star_k w^\star_l \sum_{i=1}^N w^\star_i  \sum_{j=1}^{|\mathcal{D}_i|} u^i_j u^{i\top}_j  \Big|^2 \notag \\ 
\leq & N^2 \cdot \max_{1\leq k, l, i \leq N} \Big| \frac{1}{N^3} 
  - w^\star_k w^\star_l 
   w^\star_i \Big|^2 \cdot \Big| \sum_{i=1}^N \sum_{j=1}^{|\mathcal{D}_i|}  u^i_j u^{i\top}_j  \Big|^2  \notag \\  
   \leq &   N^2 \cdot \max_{1\leq k, l, i \leq N} \Big| \frac{1}{N^3} 
  -  w^\star_k w^\star_l 
   w^\star_i \Big|^2 \cdot  \bar{N} \sum_{i=1}^N \sum_{j=1}^{|\mathcal{D}_i|}  | u^i_j |^4   \notag 
\end{align} 
and 
\begin{align} 
|\widehat{P}(T) - P(T)| 
 \leq & \bar{N} N K_x^2 \cdot \max_{1\leq k, l, i \leq N} \Big| \frac{1}{N^3} 
  -  w^\star_k w^\star_l 
   w^\star_i \Big| . 
   \label{|P(T)-hatP(T)|est}
\end{align} 
Therefore \eqref{estPw-P1/N} holds for $t=T$. 

Assume by induction that \eqref{estPw-P1/N} holds for $t=u$. We show that \eqref{estPw-P1/N} also holds for $t=u-1$. 
From \eqref{eqnhatP1/N(t)} and \eqref{eqnP(t)}, we have  
\allowdisplaybreaks\begin{align} 
 \widehat{P}(u-1) - P(u-1) 
= & [ M(P(u)) ]^{-1}
 -  [ M( \widehat{P}(u) ) ]^{-1}  \notag  \\ 
 = & [ M( P(u) ) ]^{-1} (\widehat{P}(u) - P(u)) 
[ M( \widehat{P}(u) ) ]^{-1} . 
\label{hatP-Precursion}
\end{align} 
By Lemmas~\ref{lm:|AB|<=|A||B|} and \ref{lm:|A1+An|2<=(|A1|2+|An|2)n}, we further have 
\allowdisplaybreaks\begin{align} 
 | \widehat{P}(u-1) - P(u-1) | 
= & | [ M(P(u)) ]^{-1} 
 -  [ M( \widehat{P}(u) ) ]^{-1} | \notag  \\ 
 \leq & | [ M( P(u) ) ]^{-1} | \cdot  | \widehat{P}(u) - P(u) | \cdot 
 | [ M( \widehat{P}(u) ) ]^{-1} | . 
\label{|hatP-P|recursion}
\end{align} 
By \eqref{|hatP-P|recursion}, Corollary~\ref{cor:M(P)invbdd}, Remark~\ref{rmk:M(hatP)bdd}, and the induction hypothesis, we have 
that \eqref{estPw-P1/N} holds for $t=u-1$. By induction we have shown that \eqref{estPw-P1/N} holds for all $1\leq t\leq T$.

From \eqref{hatP-Precursion}, we have 
\allowdisplaybreaks 
\begin{align} 
& \big| ( M(P(t)) )^{-1}
 -  ( M( \widehat{P}(t) ) )^{-1} \big|   
 \leq  \big| ( M( P(t) ) )^{-1} \big| \cdot  \big| \widehat{P}(t) - P(t) \big| \cdot   
\big| ( M( \widehat{P}(t) ) )^{-1} 
\big| ,  \notag 
\end{align} 
which, together with \eqref{estPw-P1/N}, Corollary~\ref{cor:M(P)invbdd}, and Remark~\ref{rmk:M(hatP)bdd}, establishes \eqref{estMPw-MP1/N}. 
\end{proof}

\begin{lemma}
\label{lem:Lemma_B}
 Let $S$ and $\widehat{S}$ be the solutions of \eqref{eqnS(t)} and \eqref{eqnhatS1/N(t)}, respectively. Then they satisfy that 
for all $t=1$, ..., $T$, 
\allowdisplaybreaks
\begin{align} 
 |S(t) - \widehat{S}(t)| \leq & 
C \cdot \Big\{ \bar{N} N^{3/2} \max_{1\leq i, k, l \leq N} \Big| \frac{1}{N^3} -  w^\star_k w^\star_l w^\star_i \Big| 
  \notag \\ 
& \hspace{1cm} + \bar{N} N^{1/2} \max_{1\leq k, l\leq N} \Big| \frac{1}{N^2} - w^\star_k w^\star_l \Big|
+ |\Theta^\star - \Theta^\star_{(N)}| \Big\} , 
 \label{sup|S-hatS|est}
\end{align}
where $C>0$ is a constant depending on $\lambda$, $\beta$, $\Theta^\star$, $K_x$, $K_y$, and $T$. 
\end{lemma}

\begin{proof}[{Proof of Lemma~\ref{lem:Lemma_B}}]
By \eqref{eqnS(t)}, \eqref{eqnhatS1/N(t)}, and and Lemmas~\ref{lm:|AB|<=|A||B|} and \ref{lm:|A1+An|2<=(|A1|2+|An|2)n}, the terminal condition $S(T)-\widehat{S}(T)$ satisfies 
\allowdisplaybreaks\begin{align} \Big|S(T)-\widehat{S}(T)\Big|^2 
= & \sum_{k=1}^N \Big| 
 \frac{1}{N} \sum_{i=1}^N \frac{1}{N} \sum_{j=1}^{|\mathcal{D}_i|} u^i_j y^i_j - w^\star_k 
 \sum_{i=1}^N w^\star_i \sum_{j=1}^{|\mathcal{D}_i|} u^i_j y^i_j 
\Big|^2 \notag \\ 
\leq & N \cdot \max_{1\leq k, l \leq N} \Big| \frac{1}{N^2} - w^\star_k w^\star_l \Big|^2 \cdot \Big| \sum_{i=1}^N \sum_{j=1}^{|\mathcal{D}_i|} u^i_j y^i_j \Big|^2  \notag \\ 
 \leq & N \cdot \max_{1\leq k, l \leq N} \Big| \frac{1}{N^2} - w^\star_k w^\star_l \Big|^2 \cdot \bar{N} \sum_{i=1}^N \sum_{j=1}^{|\mathcal{D}_i|} | u^i_j|^2 \cdot |y^i_j|^2 ,  
 \notag 
\end{align} 
which implies  
\allowdisplaybreaks\begin{align} 
 \big| S(T) - \widehat{S}(T) \big| 
 \leq & \bar{N} N^{1/2} \max_{1\leq k, l\leq N} \Big| \frac{1}{N^2} - w^\star_k w^\star_l \Big| \cdot   K_x K_y . 
  \label{|S(T)-hatS(T)|est}  
\end{align} 
Therefore \eqref{sup|S-hatS|est} holds for $t=T$. 

Assume by induction that \eqref{sup|S-hatS|est} holds for $t=u$, we show that \eqref{sup|S-hatS|est} holds for $t=u-1$. 
From \eqref{eqnS(t)} and \eqref{eqnhatS1/N(t)}, the difference $S(u-1)-\widehat{S}(u-1)$ can be written as 
\allowdisplaybreaks\begin{align} 
 S(u-1) - \widehat{S}(u-1) = & [ M(P(u)) ]^{-1}[S(u) - \widehat{S}(u) - \lambda (\Theta^\star - \Theta^\star_{(N)}) ] \notag \\ 
 & - \big[ ( M(\widehat{P}(u)))^{-1} - ( M(P(u)) )^{-1} \big]  
 (\widehat{S}(u) - \lambda \Theta^\star_{(N)}) , 
 \label{S(t)-hatS(t)}
\end{align} 
and satisfies by 
Lemma~\ref{lm:|AB|<=|A||B|} that  
\begin{align} 
|S(u-1) - \widehat{S}(u-1)| 
 \leq & | [M(P(u))]^{-1} | \cdot 
\big[  |S(u) - \widehat{S}(u)| + \lambda |\Theta^\star - \Theta^\star_{(N)}|
 \big] \notag \\ 
 & + \big| ( M(\widehat{P}(u)))^{-1} 
  - ( M(P(u)) )^{-1} \big| \cdot  
 \big( |\widehat{S}(u)| + \lambda |\Theta^\star_{(N)}| \big) . \label{|hatS-S|recursion}
\end{align} 
By \eqref{|hatS-S|recursion}, \eqref{estMPw-MP1/N}, Remark~\ref{rmk:M(hatP)bdd}, and the induction hypothesis, we have that 
\eqref{sup|S-hatS|est} holds for $t=u-1$. By induction we have shown that \eqref{sup|S-hatS|est} holds for all $t=1$, ..., $T$.

\end{proof}

We now obtain Theorem~\ref{thrm:Epsilon_Efficiency}; which is a direct consequence of the following Propositions \ref{prop:stabilityaw-a1/N} and \ref{prop:stabilityLw-L1/N}.

\begin{proposition} 
\label{prop:stabilityaw-a1/N}
The controls
\eqref{optimala} and \eqref{hata} satisfy $|\widehat{\boldsymbol{\alpha}}(t)| \leq C \cdot N^{1/2}$ for all $t=0$, $1$, ..., $T-1$, and  
\allowdisplaybreaks\begin{align} 
 \sum_{t=0}^{T-1} |\boldsymbol{\alpha}(t) - \widehat{\boldsymbol{\alpha}}(t)| 
 \leq & C\cdot \Big\{ \bar{N} N^{3/2} \max_{1\leq i, k, l \leq N} \Big| \frac{1}{N^3} -  w^\star_k w^\star_l w^\star_i \Big| \notag \\ 
 &\hspace{2cm} + \bar{N} N^{1/2} \cdot \max_{1\leq k, l\leq N} \Big| \frac{1}{N^2} - w^\star_k w^\star_l \Big| 
 + |\Theta^\star - \Theta^\star_{(N)}| \Big\} , 
\label{eqn:stabilityaw-a1/N}
\end{align} 
where $C>0$ is a constant depending on $\lambda$, $\beta$, $\Theta^\star$, $K_x$, $K_y$, and $T$. 
\end{proposition} 

\begin{proof}[{Proof of Proposition~\ref{prop:stabilityaw-a1/N}}]
We let $C$ be a constant depending on $\lambda$, $\beta$, $\Theta^\star$, $K_x$, $K_y$, and $T$, 
and $C$ is allowed to vary from place to place throughout the proof. 
By \eqref{optimala}, \eqref{hata}, and the representation of $(\Theta(t))_{t=0}^{T-1}$, we can write $\boldsymbol{\alpha}$ and $\widehat{\boldsymbol\alpha}$ as   
\allowdisplaybreaks\begin{align} 
  \boldsymbol{\alpha}(t) 
 = & [ (M(P(t+1)))^{-1} - I_{Np} ]  \Big[ \Theta^\star 
 + \sum_{u=0}^{t-1} \boldsymbol{\alpha}(u) \Big]  
  +  [ M(P(t+1)) ]^{-1} ( \lambda  \Theta^\star - S(t+1) ) ,   \notag \\ 
  \widehat{\boldsymbol{\alpha}}(t) 
 = & [ (M(\widehat{P}(t+1)))^{-1} - I_{Np} ]  \Big[ \Theta^\star 
 + \sum_{u=0}^{t-1} \widehat{\boldsymbol{\alpha}}(u) \Big]  
  +  [ M(\widehat{P}(t+1)) ]^{-1} ( \lambda  \Theta^\star_{(N)} - \widehat{S}(t+1) ) .   \notag 
\end{align} 
For $t=0$, we have 
\begin{align} 
 |\widehat{\boldsymbol{\alpha}}(0)| \leq  & | (M(\widehat{P}(1)))^{-1} - I_{Np} | \cdot |  \Theta^\star |   
  + \big| [ M(\widehat{P}(1)) ]^{-1} \big| ( \lambda  |\Theta^\star_{(N)}| + |\widehat{S}(1)| ) ,  \notag 
\end{align} 
which implies $|\widehat{\boldsymbol{\alpha}}(0)| \leq C N^{1/2}$. 
Assume by induction that $|\widehat{\boldsymbol{\alpha}}(u)| \leq C N^{1/2}$ for $u=0$, ..., $t-1$, then 
\begin{align} 
\widehat{\boldsymbol{\alpha}}(t) 
 \leq & \big| (M(\widehat{P}(t+1)))^{-1} - I_{Np} \big|   \Big[ | \Theta^\star | 
 + \sum_{u=0}^{t-1} | \widehat{\boldsymbol{\alpha}}(u) | \Big] \notag \\  
 & + \big| [ M(\widehat{P}(t+1)) ]^{-1} \big| ( \lambda | \Theta^\star_{(N)} | + | \widehat{S}(t+1) | ) \notag 
\end{align} 
gives that $|\widehat{\boldsymbol{\alpha}}(t)|\leq CN^{1/2}$. By induction we have shown that 
$|\widehat{\boldsymbol{\alpha}}(t)|\leq CN^{1/2}$ for all $t=0$, $1$, ..., $T-1$.

We further write the difference $\boldsymbol{\alpha} - \widehat{\boldsymbol\alpha}$ as    
\allowdisplaybreaks\begin{align}  
 & \boldsymbol{\alpha}(t) - \widehat{\boldsymbol\alpha}(t) \notag \\ 
 = &  \big[(M(P(t+1)))^{-1} - (M(\widehat{P}(t+1)) )^{-1} \big] 
 \Big[ (\lambda + 1)\Theta^\star + \sum_{u=0}^{t-1}  \boldsymbol{\alpha}(u)    - S(t+1) \Big]  \notag \\ 
 & + \big[ ( M(\widehat{P}(t+1)) )^{-1} - I_{Np} \big] 
 \sum_{u=0}^{t-1} (\boldsymbol{\alpha}(u) - \widehat{\boldsymbol{\alpha}}(u)) 
+ [ M(\widehat{P}(t+1)) ]^{-1} (\widehat{S}(t+1) - S(t+1))  \notag \\ 
 & + \lambda [ M(\widehat{P}(t+1)) ]^{-1} ( \Theta^\star 
 - \Theta^\star_{(N)} ) . \label{a-hata} 
\end{align}

For $t=0$, the difference 
\allowdisplaybreaks\begin{align} 
  \boldsymbol{\alpha}(0) - \widehat{\boldsymbol{\alpha}}(0)  
 = & \big[ (M(P(1)))^{-1} - (M(\widehat{P}(1)))^{-1}  \big] 
 \big[ ( \lambda + 1 )   \Theta^\star_{(N)} - \widehat{S}(1)  \big] \notag \\ 
 & + [ M(P(1)) ]^{-1} [ (\lambda+1) (\Theta^\star - \Theta^\star_{(N)} ) + \widehat{S}(1) - S(1) ]   \notag  
\end{align} 
satisfies 
\allowdisplaybreaks\begin{align} 
  \big| \boldsymbol{\alpha}(0) 
   - \widehat{\boldsymbol{\alpha}}(0) \big| 
  \leq & 
  C  \cdot \Big\{ \bar{N} N^{3/2} \max_{1\leq i, k, l \leq N} \Big| \frac{1}{N^3} -  w^\star_k w^\star_l w^\star_i \Big| \notag \\ 
 & \hspace{1cm} + \bar{N} N^{1/2} \max_{1\leq k, l\leq N} \Big| \frac{1}{N^2} - w^\star_k w^\star_l \Big|
 + |\Theta^\star - \Theta^\star_{(N)}| \Big\} , \notag    
\end{align} 
due to \eqref{estMPw-MP1/N}, \eqref{sup|S-hatS|est}, Lemma~\ref{lm:|AB|<=|A||B|}, Corollary~\ref{cor:M(P)invbdd} and Remark~\ref{rmk:M(hatP)bdd}.  
Assume by induction for all $0 \leq u\leq t-1$, 
$\boldsymbol{\alpha}(u-1) - \widehat{\boldsymbol{\alpha}}(u-1)$ satisfies
\allowdisplaybreaks\begin{align} 
  \big| \boldsymbol{\alpha}(u-1) 
  - \widehat{\boldsymbol{\alpha}}(u-1) \big| 
  \leq & C   
 \cdot \Big\{ \bar{N} N^{3/2} \max_{1\leq i, k, l \leq N} \Big| \frac{1}{N^3} -  w^\star_k w^\star_l w^\star_i \Big| \notag \\ 
& \hspace{1cm} + \bar{N} N^{1/2} \max_{1\leq k, l\leq N} \Big| \frac{1}{N^2} - w^\star_k w^\star_l \Big| 
 + |\Theta^\star - \Theta^\star_{(N)}| \Big\}  . \notag   
\end{align} 
By \eqref{a-hata} we have  
\allowdisplaybreaks\begin{align}  
 & | \boldsymbol{\alpha}(t) - \widehat{\boldsymbol{\alpha}}(t) | \notag \\ 
 \leq &  \big| (M(P(t+1)))^{-1} - (M(\widehat{P}(t+1)) )^{-1} \big| \cdot 
 \Big[ \sum_{u=0}^{t-1}  | \boldsymbol{\alpha}(u) |   
 + | ( \lambda +1 ) \Theta^\star | 
 +| S(t+1) | \Big]  \notag \\ 
 & + \big| ( M(\widehat{P}(t+1)) )^{-1} - I_{Np}  \big| 
 \sum_{u=0}^{t-1} | \boldsymbol{\alpha}(u) - \widehat{\boldsymbol{\alpha}}(u) | 
+ | (M(\widehat{P}(t+1)))^{-1} | \cdot | S(t+1) - \widehat{S}(t+1) | ,  \notag   \\ 
 & + \big| [ M(\widehat{P}(t+1)) ]^{-1} \big| \cdot |\lambda | \cdot 
  \big| \Theta^\star - \Theta^\star_{(N)} \big| . \notag  
\end{align} 
By the induction hypothesis, \eqref{estMPw-MP1/N}, 
\eqref{sup|S-hatS|est}, and Remark~\ref{rmk:M(hatP)bdd},  the above inequality implies that  
\allowdisplaybreaks\begin{align}
 |\boldsymbol{\alpha}(t) - \widehat{\boldsymbol{\alpha}}(t)| 
 \leq & C \cdot \Big\{ \bar{N} N^{3/2} \max_{1\leq i, k, l \leq N} \Big| \frac{1}{N^3} -  w^\star_k w^\star_l w^\star_i \Big| 
  \notag \\
& \hspace{1cm} + \bar{N} N^{1/2} \max_{1\leq k, l\leq N} \Big| \frac{1}{N^2} - w^\star_k w^\star_l \Big| 
 + |\Theta^\star - \Theta^\star_{(N)}| \Big\}  . \label{|a(t)-hata(t)|est}  
\end{align}
We have shown by induction that \eqref{|a(t)-hata(t)|est} holds for all $t=0$, $1$, ..., $T-1$, 
and therefore \eqref{eqn:stabilityaw-a1/N} follows.  
\end{proof} 

\begin{proposition} 
\label{prop:stabilityLw-L1/N} 
Under the hypothesis of Proposition \ref{prop:stabilityaw-a1/N}, the costs \eqref{eq:penalized_regret_functional} under the controls \eqref{optimala} 
and \eqref{hata}, respectively, satisfy the estimate 
\allowdisplaybreaks\begin{align} 
 \big| L^\star (  \boldsymbol{\alpha} ; \mathcal{D} ) 
 -  L^\star (  \widehat{\boldsymbol{\alpha}} ;  \mathcal{D} )  \big|  
 \le &  C  \cdot \Big\{ \bar{N} N^3 \max_{1\leq i, k, l \leq N} \Big| \frac{1}{N^3} -  w^\star_k w^\star_l w^\star_i \Big| 
  + \bar{N} N^{3/2} \max_{1\leq k, l\leq N} \Big| \frac{1}{N^2} - w^\star_k w^\star_l \Big| 
  \notag \\
& \hspace{3cm} 
 + \bar{N} \max_{1\leq i \leq N} \Big|\frac{1}{N} - w_i^\star \Big|
 + N |\Theta^\star - \Theta^\star_{(N)}| \Big\}  ,  
\label{stabilityLw-L1/N} 
\end{align} 
where $C>0$ is a constant depending on $\lambda$, $\beta$, $\Theta^\star$, $(w_1^\star, ..., w_N^\star)$,  
$K_x$, $K_y$ and $T$. 
\end{proposition} 

\begin{proof}[Proof of Proposition~\ref{prop:stabilityLw-L1/N} ] 
Let $L^\star(\widehat{\boldsymbol{\alpha}}; \overline{1/N} , \mathcal{D})$, 
$L^\star( \widehat{\boldsymbol{\alpha}}; w^\star, \mathcal{D})$, and
$L^\star(\boldsymbol{\alpha}; w^\star , \mathcal{D})$ be as defined in 
Lemmas~\ref{lm:|L(hata;1/N)-L(hata;w)|est} 
and~\ref{lm:|L(a;w)-L(hata;1/N)|est}.   
The difference $L^\star(\boldsymbol{\alpha}; \mathcal{D}) - L^\star( \widehat{\boldsymbol{\alpha}}; \mathcal{D})$ can be decomposed as 
\begin{align} 
L^\star(\boldsymbol{\alpha}; \mathcal{D}) - L^\star( \widehat{\boldsymbol{\alpha}}; \mathcal{D}) 
= L^\star(\boldsymbol{\alpha}; w^\star, \mathcal{D}) - L^\star(\widehat{\boldsymbol{\alpha}}; \overline{1/N}, \mathcal{D}) 
 + L^\star(\widehat{\boldsymbol{\alpha}}; \overline{1/N} , \mathcal{D}) 
  - L^\star( \widehat{\boldsymbol{\alpha}}; w^\star, \mathcal{D}) ,\notag 
\end{align}  
and satisfies    
\begin{align} 
 | L^\star(\boldsymbol{\alpha}; \mathcal{D}) - L^\star( \widehat{\boldsymbol{\alpha}}; \mathcal{D}) | 
\leq | L^\star(\boldsymbol{\alpha}; w^\star, \mathcal{D}) - L^\star(\widehat{\boldsymbol{\alpha}}; \overline{1/N}, \mathcal{D}) |
 + | L^\star(\widehat{\boldsymbol{\alpha}}; \overline{1/N} , \mathcal{D}) 
  - L^\star( \widehat{\boldsymbol{\alpha}}; w^\star, \mathcal{D}) | . \notag 
\end{align} 
The desired estimate \eqref{stabilityLw-L1/N} then follows from the above inequality and \eqref{|L(hata,1/N)-L(hata,w)|est} 
and \eqref{|L(a,w)-L(hata,1/N)|est}. 
\end{proof}

\begin{lemma}
\label{lm:|L(hata;1/N)-L(hata;w)|est} 
Let $L^\star(\widehat{\boldsymbol{\alpha}}; w^\star, \mathcal{D})$ be \eqref{eq:penalized_regret_functional} with 
information sharing weights $w^\star=(w_1^\star, ..., w_N^\star)$ and $\Theta(t+1)-\Theta(t)=\widehat{\boldsymbol{\alpha}}(t)$, 
and let $L^\star(\widehat{\boldsymbol{\alpha}}; \overline{1/N} , \mathcal{D})$ be \eqref{eq:penalized_regret_functional} with information sharing weights $\overline{1/N}=(1/N, ..., 1/N)$ and $\Theta(t+1)-\Theta(t)=\widehat{\boldsymbol{\alpha}}(t)$. 
Then it satisfies 
\begin{align} 
 | L^\star(\widehat{\boldsymbol{\alpha}}; \overline{1/N} , \mathcal{D}) 
  - L^\star( \widehat{\boldsymbol{\alpha}}; w^\star, \mathcal{D}) | \leq C  \cdot 
  \Big\{ \bar{N} \max_{1\leq i \leq N} \Big|\frac{1}{N} - w_i^\star \Big| + N^{1/2} |\Theta^\star-\Theta^\star_{(N)}| \Big\} 
  \label{|L(hata,1/N)-L(hata,w)|est} , 
\end{align} 
where $C>0$ is a constant depending on $\lambda$, $\beta$, $\Theta^\star$, 
$(w^\star_1, \cdots, w^\star_N)$, $K_x$, $K_y$ and $T$. 
\end{lemma} 
\begin{proof} 
By \eqref{eq:penalized_regret_functional} and \eqref{lwT}, we have 
\begin{align} 
   L^\star(\widehat{\boldsymbol{\alpha}}; \overline{1/N} , \mathcal{D}) 
  - L^\star( \widehat{\boldsymbol{\alpha}}; w^\star, \mathcal{D}) 
 = & \sum_{t=0}^{T-1} \big( |\Theta(t+1) - \Theta^\ast|^2 - |\Theta(t+1) - \Theta_{(N)}^\ast|^2 \big) \notag \\ 
 & + l(\Theta(T), \overline{1/N}; \mathcal{D} ) - l(\Theta(T), w^\star ; \mathcal{D} ) .\notag  
\end{align} 
We have that
\begin{equation}
\label{|L(hata,1/N)-L(hata,w)|rcest} 
\begin{aligned} 
 &  \Big| \sum_{t=0}^{T-1} \big( |\Theta(t+1) - \Theta^\ast|^2 - |\Theta(t+1) - \Theta_{(N)}^\ast|^2 \big) \Big|  
 \\
 =&   \Big| \sum_{t=0}^{T-1} (\Theta^\star_{(N)} - \Theta^\star)^\top (2 \Theta(t+1) - \Theta^\star - \Theta^\star_{(N)}) \Big| \notag \\ 
 \le & \sum_{t=0}^{T-1} | \Theta^\star_{(N)} - \Theta^\star | \cdot 
 | 2 \Theta(t+1) - \Theta^\star - \Theta^\star_{(N)}|  
 \leq C N^{1/2} \cdot  |\Theta^\star_{(N)} - \Theta^\star| ,  
\end{aligned} 
\end{equation}
and 
\begin{align} 
 & l(\Theta(T), \overline{1/N}; \mathcal{D} ) - l(\Theta(T), w^\star ; \mathcal{D} )  \notag \\ 
 = &  \sum_{i=1}^N w^\star_i \sum_{j=1}^{|\mathcal{D}_i|}  \Big[ u^{i\top}_j \sum_{k=1}^N w^\star_k \theta_k - y^i_j  \Big]^2 
 - \sum_{i=1}^N \frac{1}{N} \sum_{j=1}^{|\mathcal{D}_i|}  \Big[ u^{i\top}_j \sum_{k=1}^N \frac{1}{N} \theta_k - y^i_j  \Big]^2 
  \notag \\ 
  =& \sum_{i=1}^N (w^\star_i - 1/N) \sum_{j=1}^{|\mathcal{D}_i|}  \Big[ u^{i\top}_j \sum_{k=1}^N w^\star_k \theta_k - y^i_j  \Big]^2 \notag \\ 
 &  + \sum_{i=1}^N \frac{1}{N} \sum_{j=1}^{|\mathcal{D}_i|}  
 \Big[ u^{i\top}_j \sum_{k=1}^N (w^\star_k -1/N) \theta_k   \Big] 
  \Big[ u^{i\top}_j \sum_{k=1}^N ( w^\star_k + 1/N ) \theta_k - 2 y^i_j  \Big]^2 , \notag 
\end{align} 
which implies that 
\begin{align} 
 | l(\Theta(T), \overline{1/N}; \mathcal{D} ) - l(\Theta(T), w^\star ; \mathcal{D} ) | 
 \leq C \bar{N} \max_{1\leq i \leq N} |1/N - w_i^\star| . 
 \label{|L(hata,1/N)-L(hata,w)|Test}
\end{align} 
The estimate \eqref{|L(hata,1/N)-L(hata,w)|est} then follows from 
\eqref{|L(hata,1/N)-L(hata,w)|rcest}
and \eqref{|L(hata,1/N)-L(hata,w)|Test}. 
\end{proof}

\begin{lemma}
\label{lm:|L(a;w)-L(hata;1/N)|est} 
Let $L^\star(\alpha; w^\star, \mathcal{D})$ be \eqref{eq:penalized_regret_functional} with 
information sharing weights $w^\star=(w_1^\star, ..., w_N^\star)$ and $\Theta(t+1)-\Theta(t)= \boldsymbol{\alpha}(t)$, 
and let $L^\star(\widehat{\boldsymbol{\alpha}}; \overline{1/N} , \mathcal{D})$ be \eqref{eq:penalized_regret_functional} with information sharing weights $\overline{1/N}=(1/N, ..., 1/N)$ and $\Theta(t+1)-\Theta(t)=\widehat{\boldsymbol{\alpha}}(t)$. 
Then it satisfies 
\begin{align} 
 & | L^\star(\boldsymbol{\alpha}; w^\star , \mathcal{D}) 
  - L^\star( \widehat{\boldsymbol{\alpha}}; \overline{1/N}, \mathcal{D}) | 
  \leq   
   C \cdot  \Big\{ \bar{N} N^3 \cdot \max_{1\leq i, k, l \leq N} \Big| \frac{1}{N^3} -  w^\star_k w^\star_l w^\star_i \Big| 
  \notag \\
& \hspace{1.5cm} 
+\bar{N} N^{3/2} \max_{1\leq k, l\leq N} \Big| \frac{1}{N^2} - w^\star_k w^\star_l \Big|  + \bar{N} \cdot \max_{1\leq i \leq N} \Big| \frac{1}{N} - w_i^\star \Big|
+ N |\Theta^\star - \Theta^\star_{(N)}| \Big\} , 
  \label{|L(a,w)-L(hata,1/N)|est}
\end{align} 
where $C$ is a constant depending on $\lambda$, $\beta$, $\Theta^\star$, 
$(w^\star_1, \cdots, w^\star_N)$, $K_x$, $K_y$ and $T$. 
\end{lemma} 

\begin{proof} 
By the similar argument for \eqref{ansatzVw}, we have 
\begin{align} 
 L^\star( \widehat{\boldsymbol{\alpha}}; \overline{1/N}, \mathcal{D}) 
  = \Theta^{\star\top}  \widehat{P}(0)  \Theta^\star 
 + 2\widehat{S}(0)^\top \Theta^\star + \widehat{r}(0) . \label{ansatzV1/N} 
\end{align} 
By \eqref{ansatzVw} and \eqref{ansatzV1/N}, the difference $L^\star(\boldsymbol{\alpha}; w^\star, \mathcal{D}) - L^\star( \widehat{\boldsymbol{\alpha}}; \overline{1/N}, \mathcal{D})$ can be written as 
\allowdisplaybreaks\begin{align} 
  L^\star(\boldsymbol{\alpha}; w^\star, \mathcal{D}) - L^\star( \widehat{\boldsymbol{\alpha}}; \overline{1/N}, \mathcal{D}) 
= \Theta^{\star\top} ( P(0) - \widehat{P}(0) ) \Theta^\star 
 + 2 ( S(0) - \widehat{S}(0) )^\top \Theta^\star + r(0)-\widehat{r}(0) ,  
 \notag 
\end{align} 
and by Lemma~\ref{lm:|AB|<=|A||B|}, the difference satisfies  
\allowdisplaybreaks\begin{align} 
  \big| L^\star (  \boldsymbol{\alpha} ; w^\star, \mathcal{D} ) 
 -  L^\star (  \widehat{\boldsymbol{\alpha}} ; \overline{1/N} , \mathcal{D} ) \big|  
\le & |\Theta^\star|^2 \cdot |  P(0) - \widehat{P}(0) |  
 + 2 | S(0) - \widehat{S}(0) | \cdot | \Theta^\star| 
 \nonumber
 \\ 
 & + | r(0)-\widehat{r}(0) | . \label{|L(a,w)-L(hata,1/N)|leqQudratic}
\end{align}
We claim that $r(0)-\widehat{r}(0)$ has the estimate 
\allowdisplaybreaks\begin{align} 
 | r(0)-\widehat{r}(0) | \le &  C  \cdot  \Big\{ \bar{N} N^3 \cdot \max_{1\leq i, k, l \leq N} \Big| \frac{1}{N^3} -  w^\star_k w^\star_l w^\star_i \Big| 
 +\bar{N} N^{3/2} \max_{1\leq k, l\leq N} \Big| \frac{1}{N^2} - w^\star_k w^\star_l \Big| 
  \notag \\
& \hspace{2cm} 
 + \bar{N} \cdot \max_{1\leq i \leq N} \Big| \frac{1}{N} - w_i^\star \Big|
+ N \cdot |\Theta^\star - \Theta^\star_{(N)}| \Big\}  , 
\label{eqn:estrw-tilder1/N} 
\end{align} 
where $C$ is a constant depending on $\lambda$, $\beta$, $\Theta^\star$, 
$(w^\star_1, \cdots, w^\star_N)$, $K_x$, $K_y$ and $T$. 
Then the desired estimate \eqref{|L(a,w)-L(hata,1/N)|est} is established by \eqref{|L(a,w)-L(hata,1/N)|leqQudratic}, \eqref{estPw-P1/N}, \eqref{sup|S-hatS|est} and \eqref{eqn:estrw-tilder1/N}.

Now we prove the claim \eqref{eqn:estrw-tilder1/N}. 
From \eqref{eqn:Vansatz} and \eqref{lwT}, we obtain the equation that $\widehat{r}$ satisfies on $0\leq t \leq T$: 
\begin{align} 
&\begin{cases} 
 \widehat{r}(t) =  - ( \widehat{S}(t+1) - \lambda \Theta_{(N)}^\star )^\top
 [(\lambda+\beta)I_{Np} + \widehat{P}(t+1)]^{-1} \cdot ( \widehat{S}(t+1) - \lambda \Theta_{(N)}^\star )   \\  
 \hspace{1.1cm}   + \lambda \| \Theta_{(N)}^\star \|_2^2 + \widehat{r}(t+1), 
 \label{eqnhatr(t)} \\ \widehat{r}(T) = \sum_{i=1}^N (1/N) \sum_{j=1}^{|\mathcal{D}_i|} |y^i_j|^2 . 
 \end{cases}  
\end{align} 
Then $r-\widehat{r}$ satisfies the equation on $0\leq t \leq T$
\allowdisplaybreaks\begin{align} 
 & r(t) - \widehat{r}(t) \notag \\ 
= &  r(t+1) - \widehat{r}(t+1) + \lambda (\Theta^\star - \Theta^\star_{(N)}) (\Theta^\star + \Theta^\star_{(N)}) \notag \\  
& - (S(t+1) - \lambda \Theta^\star)^\top \big[ (M(P(t+1)))^{-1} - (M(\widehat{P}(t+1)))^{-1} \big] (S(t+1)-\lambda \Theta^\star) \notag \\ 
& - \big[ S(t+1) - \widehat{S}(t+1) - \lambda (\Theta^\star - \Theta_{(N)}^\star ) \big] 
 (M(\widehat{P}(t+1)))^{-1} 
 \cdot 
\\
\notag
& \qquad
\big[ S(t+1) + \widehat{S}(t+1) - \lambda (\Theta^\star + \Theta_{(N)}^\star ) \big] , \notag \\ 
& r(T) - \widehat{r}(T) = \sum_{i=1}^N (w_i^\star - 1/N) \sum_{j=1}^{|\mathcal{D}_i|} |y^i_j|^2 . \notag  
\end{align} 
Inductively, we have  
\allowdisplaybreaks\begin{align} 
  r(0) - \widehat{r}(0)  
= &   r(T) - \widehat{r}(T) + T \lambda (\Theta^\star - \Theta^\star_{(N)}) (\Theta^\star + \Theta^\star_{(N)})   \notag \\  
 & - \sum_{t=1}^T (S(t) - \lambda \Theta^\star)^\top \big[ (M(P(t)))^{-1} - (M(\widehat{P}(t)))^{-1} \big] (S(t)-\lambda \Theta^\star)   \notag  \\  
&  - \sum_{t=1}^T \big[ S(t) - \widehat{S}(t) - \lambda (\Theta^\star - \Theta_{(N)}^\star ) \big] 
 (M(\widehat{P}(t)))^{-1} \big[ S(t) + \widehat{S}(t) - \lambda (\Theta^\star + \Theta_{(N)}^\star ) \big] . \notag 
\end{align}
By Lemma~\ref{lm:|AB|<=|A||B|}, we have 
\allowdisplaybreaks\begin{align} 
 & | r(0) - \widehat{r}(0) | \leq |r(T) - \widehat{r}(T)| 
  + T \lambda | \Theta^\star - \Theta^\star_{(N)} | \cdot | \Theta^\star + \Theta^\star_{(N)} |   \notag \\ 
&  + \sum_{t=1}^T  \big| (M(P(t)))^{-1} - (M(\widehat{P}(t)))^{-1} \big| \cdot 
 \big| S(t)-\lambda \Theta^\star \big|^2   \notag  \\  
&  + \sum_{t=1}^T \big[ | S(t) - \widehat{S}(t) | + \lambda | \Theta^\star - \Theta_{(N)}^\star | \big] \cdot 
 \big| (M(\widehat{P}(t)))^{-1} \big| \cdot \big| S(t) + \widehat{S}(t) - \lambda (\Theta^\star + \Theta_{(N)}^\star ) \big| .  \label{estrw(0)-tilder1/N(0)}  
\end{align}
Then the estimate \eqref{eqn:estrw-tilder1/N} follows from 
\eqref{estrw(0)-tilder1/N(0)}, \eqref{estMPw-MP1/N}, 
\eqref{sup|S-hatS|est}, and Remark~\ref{rmk:M(hatP)bdd}. 
\end{proof}

    \begin{proof}[{Proof of Theorem~\ref{thrm:Epsilon_Efficiency}}]
        Under Assumption~\ref{assm:bddxy}, the existence of the regret-optimal algorithm $\theta_{\cdot}$ is implied by Theorem~\ref{thm:RegretOptimal_Dynamics}.  The first claim follows from Proposition~\ref{prop:stabilityaw-a1/N} upon noting that the constant ``$C_w$'' therein is bounded as a function of $w$; whence we take the constant $C$ hidden by $\mathcal{O}$ to be $C\eqdef \sup_{w\in [0,1]^N\,\sum_{i=1}^N\,w_i=1}\, C_w$.  Likewise, the second statement follows directly from Corollary~\ref{cor:stabilityLast} upon defining the constant hidden by $\mathcal{O}$ as before.  
    \end{proof}

Further investigation of the ``symmetrized system'' \eqref{eqnhatP1/N(t)} reveals a symmetric form of $\widehat{P}$, which suggests that  $\widehat{P}$ can be decomposed into homogeneous submatrices of lower dimensions. 
Specifically, if we decompose the high dimensional $Np\times Np$ matrix $\widehat{P}$ into $N\times N$ submatrices of dimensions $p\times p$ each,  and either exchange two submatrices symmetrically positioned with respect to the diagonal of $\widehat{P}$ or exchange two submatrices on the diagonal, the system \eqref{eqnhatP1/N(t)} is invariant. 
This suggests that all the submatrices on the diagonal are homogeneous and all the off-diagonal submatrices are homogeneous.
\begin{proposition}[Symmetries in $\widehat{P}$] 
\label{prop:P1/Nsubmat}
The solution of \eqref{eqnhatP1/N(t)} has the following submatrix decomposition  
\begin{align} 
 \widehat{P} = \begin{bmatrix} \pi_1 & \pi_2 & \cdots & \pi_2 \\ 
 \pi_2 & \pi_1 & \cdots &\pi_2 \\ 
\vdots & \vdots & \ddots & \vdots \\ 
 \pi_2 & \pi_2 & \cdots & \pi_1 
 \end{bmatrix} \in \mathbb{R}^{Np\times Np}, 
\quad \pi_1 \in \mathbb{R}^{p\times p}, \quad 
\pi_2 \in \mathbb{R}^{p\times p} , 
\label{hatPpi12}
\end{align} 
and the submatrices $\pi_1$ and $\pi_2$ are uniquely determined by the system   
\allowdisplaybreaks\begin{align}  
\begin{cases}
  \pi_1(t) =  \beta I_p - \beta^2 \gamma_1(t+1)  ,  
\label{eqnpi1pi2}  \\ 
   \pi_2(t) =  - \beta^2 \gamma_2(t+1) ,  \quad 1\leq t \leq T-1 , \\
 \pi_1(T) = \pi_2(T) = (1/N^3)\sum_{i=1}^N \sum_{j=1}^{|\mathcal{D}_i|} u^i_j u^{i\top}_j ,   
\end{cases} 
\end{align}
where $\gamma_1(\cdot)$ and $\gamma_2(\cdot)$ are defined as 
\eqref{gamma1} and \eqref{gamma2}. 
\end{proposition}

We argue as in \cite[Lemmata 1]{huang2021linear} in deriving the following submatrix decomposition of $\widehat{P}$. 
\begin{proof}[{Proof of Proposition~\ref{prop:P1/Nsubmat}}]
We decompose the $Np\times Np$ identity matrix into $N\times N$ submatrices
$I=(I_{ij})_{1\leq i, j \leq N}$ with each 
$I_{ij}\in \mathbb{R}^{p\times p}$.  
For each $1\leq i, j \leq N$, by $E_{ij} $we denote the $Np\times Np$ elementary matrix by exchanging the $i$th and $j$th rows of submatrices of the $Np\times Np$ identity matrix $I=(I_{ij})_{1\leq i, j \leq N}$. Note that $E_{ij}=E_{ij}^{-1}=E_{ij}^T$. 
Multiplying both sides of \eqref{eqnP(t)} from the left by $E_{ij}$ and from the right by $E_{ij}$, we obtain  
\allowdisplaybreaks\begin{align} 
& \begin{cases}
   E_{ij} \widehat{P}(t) E_{ij}  
=  \beta I_{Np} - \beta^2 [(\lambda+\beta ) I_{Np} + E_{ij} \widehat{P}(t+1) E_{ij} ]^{-1} ,  \\ 
  E_{ij} \widehat{P}(T) E_{ij} =  [ (1/N) I_p ,\cdots, (1/N) I_p ]^\top 
 \Big( \sum_{i=1}^N \frac{1}{N} \sum_{j=1}^{|\mathcal{D}_i|}  u^i_j u^{i\top}_j  \Big)  
 [ (1/N) I_p , \cdots, (1/N) I_p ] .   
\end{cases}   \label{eqnEijtildePEij}
\end{align} 
Comparing \eqref{eqnEijtildePEij} and \eqref{eqnhatP1/N(t)} reveals that 
 $\widehat{P}(t)$ and $E_{ij} \widehat{P}(t) E_{ij}$ satisfy the same ODE for all $1\leq i, j \leq N$. This observation implies that 
\allowdisplaybreaks\begin{align} 
 \widehat{P}_{ii} = \widehat{P}_{jj}, \quad \widehat{P}_{ij} = \widehat{P}_{ji} , \quad \widehat{P}_{ik}= \widehat{P}_{jk}, \quad 
\widehat{P}_{ki} = \widehat{P}_{kj}, \quad \forall 1\leq i, j \leq N, \quad \forall k \neq i, j , \notag 
\end{align} 
and the first assertion of the proposition is proved.

Since $\widehat{P}$ takes the form \eqref{hatPpi12}, 
it is straight forward to verify that $[(\lambda+\beta)I_{Np} + \widehat{P}]^{-1}$ takes the form 
\allowdisplaybreaks\begin{align} 
 [(\lambda+\beta) I_{Np} + \widehat{P}]^{-1} = \begin{bmatrix} \gamma_1 & \gamma_2 & \cdots & \gamma_2 \\ 
 \gamma_2 & \gamma_1 & \cdots &\gamma_2 \\ 
\vdots & \vdots & \ddots & \vdots \\ 
 \gamma_2 & \gamma_2 & \cdots & \gamma_1 
 \end{bmatrix} , \label{[(lambda+1)I+P]inv}
\end{align} 
with $\gamma_1$ and $\gamma_2$ given by \eqref{gamma1} and \eqref{gamma2}. 

We substitute \eqref{hatPpi12} and \eqref{[(lambda+1)I+P]inv} into \eqref{eqnhatP1/N(t)} to obtain the equations \eqref{eqnpi1pi2}.  
\end{proof}

Upon substituting $\widehat{P}$  into \eqref{eqnhatS1/N(t)}, we find that the $Np\times 1$ matrix $\widehat{S}$ can be decomposed into homogeneous submatrices of dimensions $p\times 1$ each.  
\begin{proposition} \label{prop:SubmatrixReduction}
The solution of \eqref{eqnhatS1/N(t)} can be decomposed into submatrices 
\begin{align} 
\widehat{S}=(\pi_3^\top, \cdots, \pi_3^\top )^\top \in \mathbb{R}^{Np\times 1}, 
\quad \pi_3 \in \mathbb{R}^{p\times 1} , 
\label{hatSpi3}
\end{align} 
and $\pi_3$ satisfies
\allowdisplaybreaks\begin{align} 
\begin{cases}
  \pi_3(t) = \beta \big[ \gamma_1(t+1) + (N-1) \gamma_2(t+1)  \big] 
  ( \pi_3(t+1) - \lambda \theta^\star_{(N)} )   , \\   
 \pi_3(T) = - (1/N^2) \sum_{i=1}^N \sum_{j=1}^{|\mathcal{D}_i|}  u^i_j y^i_j . 
\end{cases}   \label{eqnpi3}
\end{align} 
\end{proposition}

\begin{proof}[{Proof of Proposition~\ref{prop:SubmatrixReduction}}]
We multiply both sides of \eqref{eqnhatS1/N(t)} from the left by $E_{ij}$ defined in the proof of Proposition \ref{prop:P1/Nsubmat}, 
\allowdisplaybreaks\begin{align} 
 E_{ij} \widehat{S}(t) = &  E_{ij} \beta [(\lambda + \beta )I_{Np} + \widehat{P}(t+1)]^{-1} (\widehat{S}(t+1) - \lambda \Theta^\star_{(N)})
\notag \\ 
= &    E_{ij} \beta [(\lambda + \beta )I_{Np} + \widehat{P}(t+1)]^{-1} E_{ij} E_{ij}(\widehat{S}(t+1) - \lambda \Theta^\star_{(N)}) \notag \\ 
= &    \beta [(\lambda + \beta )I_{Np} + E_{ij} \widehat{P}(t+1)E_{ij}]^{-1}  ( E_{ij}\widehat{S}(t+1) - \lambda \Theta^\star_{(N)}) 
\label{EijtildeS1/N(t)}
\end{align}
and 
\allowdisplaybreaks\begin{align} 
  E_{ij}\widehat{S}(T) = &  - E_{ij}[ (1/N) I_p, \cdots, (1/N) I_p ]^\top \sum_{i=1}^N \frac{1}{N} \sum_{j=1}^{|\mathcal{D}_i|}  u^i_j y^i_j \notag \\ 
= &  - [ (1/N) I_p, \cdots, (1/N) I_p ]^\top \sum_{i=1}^N \frac{1}{N} \sum_{j=1}^{|\mathcal{D}_i|}  u^i_j y^i_j .   
\label{EijtildeS1/N(T)} 
\end{align}
From \eqref{EijtildeS1/N(t)} and \eqref{EijtildeS1/N(t)}, we have that $E_{ij}\widehat{S}$ also satisfies \eqref{eqnhatS1/N(t)} for all 
$1\leq i, j \leq N$. This proves that all submatrices of $\widehat{S}=(\widehat{S}_i)_{1\leq i \leq N}$ are equal on $0\leq t \leq T$. 
We further substitute \eqref{hatPpi12}, \eqref{[(lambda+1)I+P]inv} and 
$\widehat{S} = [\pi_3^\top, \cdots, \pi_3^\top]^\top$ into \eqref{eqnhatS1/N(t)} to obtain \eqref{eqnpi3}. 
\end{proof}

\begin{corollary} 
\label{cor:ComplexityAlgorithm__Heuristic}
Fix a data set $\mathcal{D}$ and assume that \eqref{eq:moredatasets_than_dataperdataset} holds. Then the computational complexity of computing the sequence $\big( \theta(t) \big)_{t=0}^T$ by \eqref{hata} is \newline
$\mathcal{O}(Tp^2 \max\{N, p^{0.373}\})$. 
\end{corollary}

\begin{proof}[{Proof of Corollary~\ref{cor:ComplexityAlgorithm__Heuristic}}]
\hfill\\ 
\noindent \textbf{Complexity of computing $\pi_1(T)$, $\pi_2(T)$, and $\pi_3(T)$ is $O(N^2 p^2)$} \\ 
Computing $u^i_j u^{i\top}_j$ has complexity $O(p^2)$ by Lemma~\ref{lm:matCompComplexity}-(iii) and computing the sum\\
$\sum_{i=1}^N \sum_{j=1}^{|\mathcal{D}_i|} u^i_j u^{i\top}_j$ has a complexity of $O( N^2 p^2)$ by Lemma~\ref{lm:matCompComplexity}-(i). 
Moreover, multiplying the $p\times p$ matrix $\sum_{i=1}^N \sum_{j=1}^{|\mathcal{D}_i|} u^i_j u^{i\top}_j$ by the scalar $1/N^3$ has complexity $O(p^2)$. Hence the computing $\pi_1(T)$ and $\pi_2(T)$ has complexity $\mathcal{O}(N^2 p^2)$. 
Computing the product $u^i_j y^i_j$ has a complexity of $O(p)$ by Lemma~\ref{lm:matCompComplexity}-(iii), and summing $\sum_{i=1}^N \sum_{j=1}^{|\mathcal{D}_i|} u^i_j y^i_j$ has complexity $O(N^2 p)$ by Lemma~\ref{lm:matCompComplexity}-(i). Multiplying the $p\times 1$ vector $\sum_{i=1}^N \sum_{j=1}^{|\mathcal{D}_i|} u^i_j y^i_j$
by the scalar $1/N^2$ has a complexity of $O(p)$. 
Therefore computing $\pi_3(T)$ has complexity $O(N^2 p)$, and the complexity of computing $\pi_1(T)$, $\pi_2(T)$, and $\pi_3(T)$ is $\mathcal{O}(N^2 p)$.  

\noindent\textbf{Complexity of Computing $\gamma_1(T)$ and $\gamma_2(T)$ is $O(p^{2.373})$} 
\\
With $\pi_1(T)$ and $\pi_2(T)$ obtained, we can compute $\gamma_1(T)$ and $\gamma_2(T)$ by \eqref{gamma1} and \eqref{gamma2}.  
From \eqref{gamma1} and \eqref{gamma2}, computing $\gamma_1(\cdot)$ and $\gamma_2(\cdot)$ from $\pi_1(\cdot)$ and $\pi_2(\cdot)$ involves multiplication and addition of $p \times p$ matrices, scalar multiplication of $p\times p$-dimensional matrices, multiplication of a $p\times p$ matrix and a $p\times 1$ vector, and inversion of a $p \times p$ matrix. 
By Lemma~\ref{lm:matCompComplexity}, multiplication of two $p\times p$ matrices has the dominant complexity of $\mathcal{O}(p^{2.373})$ among the above matrix operations. 
Moreover, inverting a $p\times p$ matrix has complexity $\mathcal{O}(p^{2,373})$. 
Hence, computing $\gamma_1(T)$ and $\gamma_2(T)$ has a complexity of $O(p^{2.373})$. 

\noindent \textbf{Complexity of computing $\pi_1(t)$, $\pi_2(t)$, $\pi_3(t)$, $\gamma_1(t)$ and $\gamma_2(t)$ for $t=T-1$, $T-2$, ..., $1$ has complexity $\mathcal{O}(T p^{2.373})$.}
\\ 
Similar to computing $\gamma_1(T)$ and $\gamma_2(T)$, computing $\gamma_1(t)$ and $\gamma_2(t)$ from $\pi_1(t+1)$ and $\pi_2(t+1)$ has complexity $\mathcal{O}(p^{2.373})$. 
According to \eqref{eqnpi1pi2}, computing $\pi_1(t)$ and $\pi_2(t)$ from $\gamma_1(t+1)$ and $\gamma_2(t+1)$ involves scalar multiplication and addition along diagonal of $p\times p$ matrices, and therefore has complexity $O(p^2)$ by Lemma~\ref{lm:matCompComplexity}-(i) and (ii), 
According to \eqref{eqnpi3}, computing $\pi_3(t)$ from $\gamma_1(t+1)$, $\gamma_2(t+1)$, and $\pi_3(t+1)$ involves scalar multiplication and addition of $p\times p$ matrices, scalar multiplication and addition of $p\times 1$ matrices, 
and multiplication of a $p\times p$ matrix with a $p\times 1$ matrix. By Lemma~\ref{lm:matCompComplexity}, the dominant complexity of these operations is $\mathcal{O}(p^2)$.   
Summing up the above argument, computing $\pi_1(t)$, $\pi_2(t)$, $\pi_3(t)$, $\gamma_1(t)$ and $\gamma_2(t)$ for each $t$ has a complexity of $\mathcal{O}(p^{2.373})$, and therefore has a complexity of $\mathcal{O}(T p^{2.373})$ for all $t=T-1$, $T-2$,..., $1$.

\noindent\textbf{The Cost of Computing Each $\Theta(t)$ is $O(Tp^2 \max\{N, p^{0.373}\})$.}
\\
By \eqref{Theta-simple}, the cost of updating from $\Theta(t)$ to $\Theta(t+1)$ is equal to the complexity of computing $\widehat{\boldsymbol{\alpha}}(t)=(\widehat{\alpha}_1^\top(t), \cdots \widehat{\alpha}_N^\top(t))^\top$ for $t=0, 1, ..., T-1$. 

By \eqref{eq:RegretOptimizationHeuristic_Helper__1}, $\widehat\alpha_i(t)$ for all $i=1$, ..., $N$ share the common term 
\allowdisplaybreaks 
\begin{align} 
 & - \big\{   \gamma_2(t+1) [ \lambda I_p + \pi_1(t+1) + (N-2)\pi_2(t+1) ] \notag \\ 
& \hspace{4cm} + \gamma_1(t+1) \pi_2(t+1) \big\}  
  \sum_{k=1}^N \theta_k(t) 
  \notag \\ 
&  - \big[ \gamma_1(t+1) + (N-1) \gamma_2(t+1) \big] (\pi_3(t+1) - \lambda \theta^\star_{(N)}) ,  \notag 
\end{align} 
and computing the common term involves matrix multiplication, addition, and scalar multiplication of $p\times p$ matrices, 
matrix addition and scalar multiplication of $p\times 1$ matrices, 
and multiplication of a $p\times p$ matrix and a $p\times 1$ matrix.  
By Lemma~\ref{lm:matCompComplexity}, these operations have complexity $O(Np + p^{2.373})$, where $Np$ is due to adding the $N$ $p$-dimensional vectors $\sum_{k=1}^N \theta_k(t)$, and $p^{2.373}$ is due to multiplication of two $p\times p$ matrices by a CW-like algorithm.

Computing the individual terms 
\allowdisplaybreaks
\begin{align} 
 - (\gamma_1(t+1) - \gamma_2(t+1) ) (\lambda I_p + \pi_1(t+1) - \pi_2(t+1) ) \theta_i(t) \notag 
\end{align} 
for $\widehat\alpha_i(t)$, $i=1$, ..., $N$ involves addition, scalar multiplication, and matrix multiplication of $p\times p$ matrices, and multiplication of a $p\times p$ matrix and a $p\times 1$ matrix. 
By Lemma~\ref{lm:matCompComplexity}, these operations have complexity of $\mathcal{O}(p^{2.373} + Np)$, where $p^{2.373}$ is due to multiplication of two $p\times p$ matrices by a CW-like algorithm and $Np$ is due to multiplying a $p\times p$ matrix by the $p$-dimensional vector $\theta_i(t)$ for $i=1$, ..., $N$. 
Hence the complexity of computing $\widehat{\boldsymbol{\alpha}}(t)$ for $t=0, 1, ..., T-1$ is $\mathcal{O}(T p^2\max\{N, p^{0.373} \} )$.

\noindent\textbf{The complexity of computing $(\Theta(t))_{t=0}^T$ is $\mathcal{O}(T p^2 \max\{N,p^{0.373}\})$} \\ 
Looking over all computations involved, we deduce that the complexity of computing the sequence 
$(\Theta(t))_{t=1}^T$ starting from $\Theta(0)=\Theta^\star$ has a complexity of $\mathcal{O}(T p^2 \max\{N,p^{0.373}\})$. 
\end{proof}

We may now deduce Theorem~\ref{thrm:accelertion_heuristic_regret_optimal_algorithm} and Theorem~\ref{thrm:regretoptimality_SymmetricCase}.  
\begin{proof}[{Proof of Theorem~\ref{thrm:accelertion_heuristic_regret_optimal_algorithm}}]
    The result directly follows from Theorem~\ref{thm:ComplexityAlgorithm} and Corollary~\ref{cor:ComplexityAlgorithm__Heuristic}.  
\end{proof}

\begin{proof}[{Proof of Theorem~\ref{thrm:regretoptimality_SymmetricCase}}]
Direct consequence of Propositions~\ref{prop:P1/Nsubmat} and~\ref{prop:SubmatrixReduction}.
\end{proof}



\end{document}